\documentclass[twoside]{article}

\usepackage[accepted]{aistats2022}
\usepackage{natbib}

\usepackage{xcolor}

\usepackage{wrapfig}
\usepackage{balance}

\usepackage{amsmath,mathrsfs,amsfonts,amsthm,amscd,amssymb,graphicx,color,cite}
\usepackage{bbm}
\numberwithin{equation}{section}

\usepackage{upgreek}
\usepackage[titletoc,title]{appendix}
\usepackage{enumitem}

\usepackage{mathtools}

\newcommand{\R}{{\mathbb R}}

\newcommand{\G}{{\mathbb G}}

\newcommand{\N}{{\mathbb N}}

\newcommand\norm[1]{\left\| #1\right\|}

\newcommand{\calP}{{\mathcal P}}

\newcommand{\calS}{{\mathcal S}}

\newcommand{\calW}{{\mathcal W}}

\newcommand{\calK}{{\mathcal K}}

\newcommand{\Let}{\coloneqq}

\usepackage{xcolor}

\def\longequals{\mathbin{=\kern-2pt=}}

\newtheorem{theorem}{Theorem}[section]
\newtheorem{assumption}[theorem]{Assumption}
\newtheorem{corollary}[theorem]{Corollary}
\newtheorem{definition}[theorem]{Definition}
\newtheorem{remark}[theorem]{Remark}
\newtheorem{lemma}[theorem]{Lemma}
\newtheorem{proposition}[theorem]{Proposition}

\numberwithin{equation}{section}

\newcommand{\beq}{\begin{equation}}
\newcommand{\eeq}{\end{equation}}

\definecolor{darkred}{rgb}{.70,.12,.20}

\definecolor{darkgreen}{rgb}{.20,.52,.14}

\usepackage{caption}
 \usepackage{subcaption}

\newcommand\blfootnote[1]{%
  \begingroup
  \renewcommand\thefootnote{}\footnote{#1}%
  \addtocounter{footnote}{-1}%
  \endgroup
}

\usepackage{hyperref}

\let\emptyset\varnothing

\begin{document}

%
\runningtitle{Sobolev Transport: A Scalable Metric for Probability Measures with Graph Metrics}

%
\runningauthor{Tam Le, Truyen Nguyen, Dinh Phung, Viet Anh Nguyen}

\twocolumn[

\aistatstitle{Sobolev Transport: A Scalable Metric for Probability Measures with Graph Metrics}

\aistatsauthor{Tam Le$^{*}$ \And Truyen Nguyen$^{*}$ \And  Dinh Phung \And Viet Anh Nguyen }

\aistatsaddress{ RIKEN AIP \And  The University of Akron \And Monash University \And VinAI Research} ]

\begin{abstract}
Optimal transport (OT) is a popular measure to compare probability distributions. However, OT suffers a few drawbacks such as (i) a high complexity for computation, (ii) indefiniteness which limits its applicability to kernel machines. In this work, we consider probability measures supported on a graph metric space and propose a novel Sobolev transport metric. We show that the Sobolev transport metric yields a \emph{closed-form} formula for fast computation and it is negative definite. 
We show that the space of probability measures endowed with this transport distance is isometric to a bounded convex set in a Euclidean space with a weighted $\ell_p$ distance. We further exploit the negative definiteness of the Sobolev transport to design positive-definite kernels, and evaluate their performances against other baselines in document classification with word embeddings and in topological data analysis. 
\end{abstract}

\section{INTRODUCTION} \label{sec:intro}

Optimal transport (OT) is a powerful tool to compare probability measures. OT is widely used in machine learning~\citep{courty2017joint, bunne2019, nadjahi2019asymptotic, peyre2019computational, ref:kuhn2019wasserstein, titouan2019optimal, janati2020entropic, muzellec2020missing,  paty2020regularity, altschuler2021averaging, fatras2021unbalanced, klicpera2021scalable, le2021adversarial, mukherjee2021outlier, nguyen2021optimal,  scetbon2021low, si2021testing}, statistics~\citep{mena2019statistical, pmlr-v99-weed19a, ref:blanchet2021statistical}, computer graphics and vision~\citep{rabin2011wasserstein, solomon2015convolutional, lavenant2018dynamical, nguyen2021point}. However, evaluating the OT incurs a high computational complexity in general~\citep{peyre2019computational} which leads to several proposals in the recent literature to address this drawback of OT, e.g., approximate using entropic regularization~\citep{Cuturi-2013-Sinkhorn}, or exploit geometric structure of supports~\citep{rabin2011wasserstein, ref:le2019tree, pmlr-v130-le21a}. Among them, tree-Wasserstein~\citep{ref:evans2012phylogenetic, ref:le2019tree} (TW) leverages the tree structure over supports to obtain a closed-form for fast computation. However, the requirement about tree structure for supports may be restricted in applications. In this work, we exploit the graph structure, which appears in several applications, and propose a scalable variant of OT to compare probability measures supported on a graph metric space.


Given  any two distributions $\mu$ and $\nu$ supporting on nodes of a tree with nonnegative weights, it is known from  \citet{ref:evans2012phylogenetic, ref:le2019tree} that the $1$-Wasserstein distance $\calW_1$ w.r.t.~the tree distance (i.e., TW) admits a closed-form expression, which allows a fast computation (i.e., its complexity is linear to the number of edges in the tree). The key techniques in deriving this formula are to leverage the dual formulation of $\calW_1$ and exploit the fact that there is a unique path between any two nodes on the tree. Due to a different nature of the dual formulation between $p=1$ and $p>1$, it is, unfortunately, unknown whether the closed-form expression still holds for the $p$-Wasserstein distance with ground tree metric when $p>1$. It is also not known if the closed-form for $\calW_1$ with ground tree metric can be extended to general graphs where there are multiple paths connecting two nodes (i.e., graph metric ground cost). The approaches proposed in~\citet{ref:evans2012phylogenetic, ref:le2019tree, pmlr-v130-le21a} do not resolve these questions, either. 


\paragraph{Related Work.} 
Our proposed Sobolev transport\blfootnote{$^*$: Two authors contributed equally.} is an instance of the integral probability metric~\citep{muller1997integral} and closely related to $\calW_1$ for probability measures supported on a graph metric space. Similar to TW, the Sobolev transport exploits the structure of supports for a fast computation and has similar properties as the TW (e.g., both of them are negative definite which is the key to build positive definite kernels for applications with kernel machines). Moreover, Sobolev transport has more flexibility and degrees of freedom than TW since it requires a graph structure rather than tree structure over supports.

We further note that the Sobolev transport leverages a graph structure for probability measures supported on a \emph{graph metric space}, rather than a \emph{general graph} over supports. For example, the edge weight in the graph, corresponding with a graph metric space, is a cost to move from one node to the other node of that edge (i.e., the distance between two edge nodes), rather than an affinity between these edge nodes of the graph used in diffusion earth mover's distance~\citep{ref:tong2021diffusion}.




\paragraph{Contributions.}
We propose a novel distance, named the Sobolev transport $\calS_p$ of any order $p \geq 1$, to measure the distance between probability measures supported on a graph metric space. Moreover,
\begin{itemize}
    \item we show that $\calS_p$ (i) admits a fast closed-form computation and (ii) is negative definite. Consequently, we can derive positive-definite kernels using our proposed Sobolev transport distance $\calS_p$, which can be applied for many kernel-dependent frameworks in machine learning.
    
    \item when $p=1$ and with a tree structure, we draw a connection of our proposed Sobolev transport $\calS_1$ to the $1$-Wasserstein distance $\calW_1$. 
    \item we also prove that the space of probability measures with Sobolev transport metric $\calS_p$ is isometric to a bounded convex set in a Euclidean space with a weighted $\ell_p$ distance.
\end{itemize}

In Section \ref{sec:preli}, we provide the setup of our problem. The Sobolev transport is formally introduced in Section~\ref{sec:distance}, and we discuss its nice properties in Section \ref{sec:properties}. In Section \ref{sec:numerical}, we illustrate empirically that the kernel machines using our proposed Sobolev transport distance perform favorably compared to other baselines in real-world applications. Proofs are placed in the supplementary (Section \ref{appsec:proofs}). Furthermore, we have released code for our proposals.\footnote{\url{https://github.com/lttam/SobolevTransport}}


\section{PRELIMINARIES} \label{sec:preli}

Let $\G =(V,E)$ be an undirected  and connected graph  with positive edge lengths $\{w_e\}_{e\in E}$. We consider a physical graph in the sense that $V$ is a subset of the vector  space $\R^n$ and each edge $e \in E$ is the standard line segment in $\R^n$ connecting the two end-points of $e$. The most important  case for our applications is when  $w_e$ coincides with the Euclidean length of the edge $e$. 

Henceforth, by mentioning the graph $\G$, we mean the set of  \textit{all}  nodes $V$ together with \textit{all} points forming the edges $E$.\footnote{I.e., the collection of all points in $\R^n$ belongs to one of the edges.} This general consideration allows us to work with a continuous setting to derive a closed-form formula for a newly proposed transport distance. Notice that we can canonically measure the  weighted length for any path in $\G$ whose end-points might not be  nodes in $V$. Indeed, for any two points $x$ and $y$  belonging to the same edge $e= \langle u, v\rangle$ connecting two nodes $u$ and $v$,  we can express $x = (1-s) u + s v$ and $y = (1-t)u + t v$ for some numbers $t,s\in [0,1]$. Then, the length of the path  connecting $x$ and $y$ along edge $e$ (i.e., the line segment $\langle x, y\rangle$)  is defined by $|t-s| w_e$. The length for  an arbitrary path in $\G$ is defined similarly by breaking down into pieces and summing over their corresponding lengths.


We impose on $\G$ the following graph metric $d$: for every $x, y \in \G$, $d(x, y)$ equals to the length of the shortest path on $\G$ between $x$ and $y$. Because the edges are undirected and the lengths $\{w_e\}_{e \in E}$ are positive, it is easy to show that $d$ satisfies the non-negativity, the symmetry and the triangle inequality properties. Thus, $d$, by construction, is a metric.



Further, we assume that $\G$ satisfies the following uniqueness property of the shortest paths.
\begin{assumption}[Unique-path root node] \label{a:unique}
There exists a root node $z_0\in V$ such that for every $x \in \G$,  $d(x,z_0)$ is attained by a \emph{unique} shortest path connecting $x$ and $z_0$. 
\end{assumption}


Recall that a graph is geodetic if for every pair of nodes  the shortest path between them is unique. Thus, geodetic graphs are special examples satisfying  Assumption~\ref{a:unique}. An example of geodetic graph is given in Figure~\ref{fg:GeodeticGraph}.

For $1\leq p\leq \infty$ and for a  nonnegative Borel measure $\lambda$ on $\G$, let   $L^p( \G, \lambda)$ denote the space of all Borel measurable functions  $ f:\G\to \R$ satisfying $\int_\G |f(y)|^p \lambda(\mathrm{d}y) <\infty$. Two functions $f_1, f_2 \in L^p( \G, \lambda)$ are considered to be the same  if $f_1(x) =f_2(x)$ for $\lambda$ almost every $x$ in $\G$. Then, $L^p( \G, \lambda)$ is a normed space with the norm  defined by
\[
\|f\|_{L^p(\G, \lambda)} \coloneqq  \Big(\int_\G |f(y)|^p \lambda(\mathrm{d}y)\Big)^\frac1p.
\]

\begin{figure}
  \begin{center}
    \includegraphics[width=0.2\textwidth]{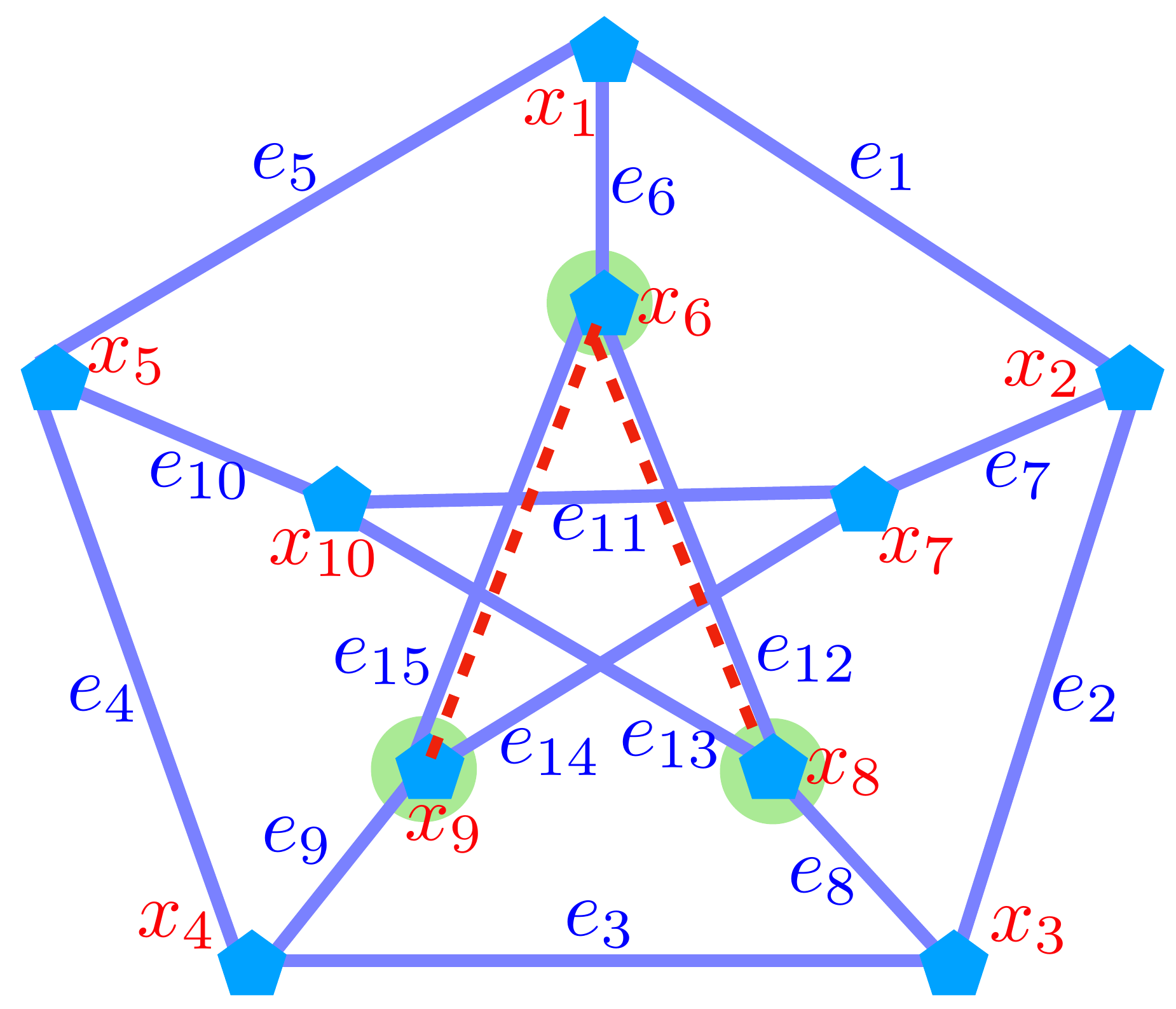}
  \end{center}
  \vspace{-6pt}
  \caption{An illustration for a geodetic graph with 10 nodes $\left\{x_i\right\}_{i=1}^{10}$ and 15 edges $\left\{e_j\right\}_{j=1}^{15}$ where each edge weight equals to one, i.e., $w_{e_j}=1, \forall j$. For any $x_i, x_j$, there is a unique shortest path between them, with a length $2$. Let $x_1$ be the unique-path root node (i.e., $z_0 = x_1$) and $\widetilde{\G}$ be a subgraph containing 3 nodes $\left\{x_6, x_8, x_9\right\}$ and 2 edges $\left\{e_{12}, e_{15}\right\}$, then $\Lambda(x_6) = \gamma(e_6) = \widetilde{\G}$.}
  \label{fg:GeodeticGraph}
 \vspace{-6pt}
\end{figure}

Throughout the paper, we use $\langle x_1, x_2\rangle$ to denote the line segment in $\R^n$ connecting two points $x_1, x_2$, while $( x_1, x_2)$ means the same line segment but without its two end-points. The symbol $[z_0,y]$ denotes the shortest path in $\G$ connecting $z_0$ and $y\in\G$. Under Assumption~\ref{a:unique}, $[z_0, y]$ is a unique path. Also, $\calP(\G)$ represents the  set  of all Borel probability measures on $\G$. The conjugate of  a number $1\leq p\leq \infty$ is denoted by $p'$. This is the  number in $ [1,\infty]$ satisfying $\frac1p +\frac{1}{p'}=1$. In case $p=1$, we have $p' = \infty$. Given $x\in \G$,  define  \begin{equation}\label{sub-graph}
\Lambda(x) \coloneqq \{y\in \G: \, x\in [z_0,y]\}.
\end{equation}
 Notice that $\Lambda(x)$ is always non-empty, and  $x \in \Lambda(x)$.
 On the other hand, let  $\gamma_e$ denote  the collection of all points $y\in\G$ such that the unique shortest path connecting $y$ and $z_0$ contains the edge $e$. That is, 
\begin{equation}\label{sub-graph-over-edge}
\gamma_e \coloneqq \{y\in \G: \, e\subset  [z_0,y]\}.
\end{equation}
In Figure~\ref{fg:GeodeticGraph}, we show a computation for  $\Lambda(x)$ and $\gamma_e$. Furthermore, we write $|E|$ and $|V|$ for the cardinality of sets $E$ and $V$ respectively. For a measure $\mu$, let $\text{supp}(\mu)$ denote the set of supports of $\mu$.

\section{SOBOLEV TRANSPORT DISTANCE}
\label{sec:distance}



In this section, we define an instance of integral probability metrics  between probability distributions on the graph. Our definition is inspired by  the dual form of the $1$-Wasserstein distance $\calW_1$, and by \citet{Mroueh-2018-Sobolev,Xu-2021-Towards}. Instead of using the Lipschitz constraint for the critic as in $\calW_1$, we relax it by considering the constraint in a Sobolev space. We first propose a generalized version of the fundamental theorem of calculus, which defines the derivative of a function at any point $x \in \G$ dependent on the shortest path from the root node $z_0$ to $x$. 


\begin{definition}[Graph-based Sobolev space] \label{def:Sobolev}
Let $\lambda$ be a nonnegative Borel measure on $\G$, and let  $1\leq p\leq \infty$. A continuous function $f: \G \to \R$ is said to belong to the Sobolev space $W^{1,p}(\G, \lambda)$ if there exists a  function $h\in L^p( \G, \lambda) $ satisfying 
\begin{equation}\label{FTC}
f(x) -f(z_0) =\int_{[z_0,x]} h(y) \lambda(\mathrm{d}y)  \quad \forall x\in \G.
\end{equation}
Such function  $h$ is unique in $L^p(\G, \lambda) $ and is called the graph derivative of $f$ w.r.t.~the measure $\lambda$. Hereafter, this graph derivative of $f$ is denoted $f'$.
\end{definition}

The integral in Definition~\ref{def:Sobolev} is a line integral. We now formally define the Sobolev transport distance between two  distributions supported on $\G$.
\begin{definition}[Sobolev transport distance on graphs]
\label{def:distance}
Let $\lambda$ be a nonnegative Borel measure on $\G$.
Let  $1\leq p\leq \infty$  and let  $p'$ be its conjugate, i.e., the number $p'\in [1,\infty]$ satisfying $\frac1p +\frac{1}{p'}=1$.    For $\mu, \nu\in \calP(\G)$, we define 
\begin{equation} \notag \label{eq:distance}
\calS_p(\mu,\nu ) \! \coloneqq \! \left\{
\begin{array}{cl}
\hspace{-0.4em} \sup & \Big[\int_\G f(x) \mu(\mathrm{d}x) - \int_\G f(x) \nu(\mathrm{d}x)\Big] \\
\hspace{-0.4em} \mathrm{s.t.} & f\in W^{1,p'}(\G, \lambda),  \, \|f'\|_{L^{p'}(\G, \lambda)}\leq 1.
\end{array}
\right.
\end{equation}
\end{definition}

By definition, the quantity $\calS_p(\mu,\nu )$ depends on the measure $\lambda$ and 
on the choice of the unique-path root node $z_0$ via the graph derivative $f'$; however, we omit these dependencies 
when no confusion may arise. The role of $\lambda$  will be displayed in Section \ref{sec:properties} when we make a connection between our transport distance and the Wasserstein distance. Specifically, if $p=1$ and $\lambda([z_0,x])= d(z_0,x)$, then the constrain for $f$  in Definition~\ref{def:distance} is the same as $|f(x)-f(z_0)|\leq d(z_0,x)$ for every $x\in \G$.  Thus, the Sobolev transport distance $\calS_1$ coincides with the $1$-Wasserstein distance in this particular case.
The next result asserts that $\calS_p(\mu,\nu )$ is an integral probability metric on the graph $\G$.

\begin{lemma}[Metrization] \label{lem:metrize}
For any $1\leq p\leq \infty$, the Sobolev transport $\calS_p$ is a metric on the space $\calP(\G)$.
\end{lemma}

The next  result gives a  comparison between Sobolev transport distances with different exponent $p$.

\begin{proposition}[Upper bound] \label{prop:upper}
Assume that $\lambda$ is a finite and nonnegative Borel measure on $\G$. Then, for any $1\leq p < q\leq \infty$ with conjugates $1 \le q' < p' \le \infty$,  we have 
\[
\calS_p(\mu,\nu ) 
\leq \lambda(\G)^{\frac{1}{q'} - \frac{1}{p'} } \, \, \calS_q(\mu,\nu ). 
\]
\end{proposition}

Our proposed Sobolev transport distance $\calS_p$ admits  a closed-form formula as follows.
\begin{proposition}[Closed-form formula]\label{prop:closed-form}
Let $\lambda$ be any nonnegative Borel measure on $\G$, and let  $1\leq p\leq \infty$. Then, we have 
\[
\calS_{p}(\mu,\nu )^p
= \int_{\G} | \mu(\Lambda(x)) -  \nu(\Lambda(x))|^p \, \lambda(\mathrm{d}x),
\]
where $\Lambda(x)$ is the subset of $\G$  defined by \eqref{sub-graph}. 
\end{proposition}

\textbf{Sketch of Proof of Proposition~\ref{prop:closed-form}.}~By using representation \eqref{FTC} and employing Fubini's theorem to interchange the order of integration, we have for any function 
$f\in W^{1,p'}(\G, \lambda)$ and any measure $\sigma \in \calP(\G)$ that
\begin{align*}
\int_\G f(x) \sigma(\mathrm{d}x)= f(z_0) \sigma(\G)   + \int_{\G} f'(y)  \sigma(\Lambda(y)) \, \lambda(\mathrm{d}y).
\end{align*}
This together with the definition of distance $\calS_{p}$ and by taking $g = f'$, we deduce that
$\calS_{p}(\mu,\nu )$ is the same as
\begin{align*}
\sup \limits_{\|g\|_{L^{p'}(\G, \lambda)}\leq 1} \int_{\G} g(x) \big[ \mu(\Lambda(x)) -  \nu(\Lambda(x))\big] \, \lambda(\mathrm{d}x).
\end{align*}
This last optimization problem admits a  maximizer  $g^*(x) = \frac{|r(x)|^{p-2} r(x)}{\|r\|_{L^p(\G, \lambda)}^{p-1}}$ with $r(x) \coloneqq \mu(\Lambda(x)) -  \nu(\Lambda(x))$, and the conclusion of the proposition follows.

In the particular case where the probability distributions $\mu$ and $\nu$ are supported only on nodes $V$, the expression in Proposition~\ref{prop:closed-form} can be rewritten more explicitly using the definition of $\gamma_e$ in~\eqref{sub-graph-over-edge}.
\begin{corollary}[Discrete case]\label{cor:discrete}
Assume that the measure $\lambda$ has no atom, i.e., $\lambda(\{x\}) = 0$ for every $x\in \G$. 
Then, if  both measures $\mu$  and $\nu$ in $\calP(\G)$ are supported on $V$,  we have 
\begin{equation}
\label{eq:DiscreteSobolevTransport}
\calS_{p}(\mu,\nu )^p 
= \sum_{e\in E} \lambda(e)\, \big| \mu(\gamma_e) -  \nu(\gamma_e)\big|^p.
\end{equation}
\end{corollary}

\begin{remark}[Two-step computational procedure]\label{rm:computeSobolev} Our calculation of the Sobolev transport distance between $\mu$ and $\nu$ can be  split into two separate steps. The first step is the preprocessing process involving only the graph structure and nothing about the probability distributions, and is done only once regardless how many pairs $(\mu,\nu)$ that we have to measure. In this step by identifying shortest paths (e.g., Dijkstra algorithm), we calculate the set $\gamma_e$ for each edge $e \in E$. In fact, any edge $e$ with $\gamma_e = \emptyset$ does not contribute to the computation of the Sobolev transport. Therefore, we can remove such edge $e$ for the summation over edges in $E$ (in Equation \eqref{eq:DiscreteSobolevTransport}). In the second step, we just simply use the result in Step~1 and Corollary~\ref{cor:discrete} to compute the Sobolev transport distance.
\end{remark}

\paragraph{Complexity.} For preprocessing, the complexity of Dijkstra for shortest paths from the root node $z_0$ to all other supports (or vertices) is $\mathcal{O}(|E| + |V| \log{|V|})$. A key observation is that for any support point $z$ of $\mu$, i.e., $z \in \text{supp}(\mu)$, its mass contributes to $\mu(\gamma_e)$ if and only if the edge $e$ is a subset of the shortest path from the root node $z_0$ to $z$, i.e., $e \subset [z_0, z]$. Let $E_{\mu, \nu}$ be a subset of $E$, defined as:
\[
E_{\mu, \nu} \hspace{-0.2em}:=\hspace{-0.2em} \left\{e \! \in \! E \mid \exists z \! \in \! (\text{supp}(\mu) \cup \text{supp}(\nu)), e \subset [z_0, z] \right\}\!,
\]
then we can rewrite $\calS_{p}(\mu,\nu )^p$ in \eqref{eq:DiscreteSobolevTransport} as
\begin{equation}
\label{eq:DiscreteSobolevTransport_Opt}
\calS_{p}(\mu,\nu )^p 
= \sum_{e\in E_{\mu, \nu}} \lambda(e)\, \big| \mu(\gamma_e) -  \nu(\gamma_e)\big|^p.
\end{equation}
Therefore, the computation of Sobolev transport $\calS_{p}(\mu, \nu)$ is linear to the number of edges in $E_{\mu, \nu}$.


\section{PROPERTIES OF SOBOLEV TRANSPORT}\label{sec:properties}

This section shows a connection between our Sobolev transport distance and the Wasserstein distance when the measure $\lambda$ is chosen as the length measure of the graph. We also demonstrate that the space of distributions $\calP(V)$ is  isometric to a bounded convex set in a Euclidean space. We then prove that for $1 \le p \le 2$, both $\calS_p$ and its $p$-power $\calS_p^p$ are negative definite which allows us to build positive definite kernels upon Sobolev transport. We also propose a slice variant for Sobolev transport.

\subsection{A Connection to Wasserstein Distance}
\label{sub:connection}
We will specifically construct a measure $\lambda^*$ under which the distance $\calS_1$ is the same as  the $1$-Wasserstein distance $\mathcal{W}_1$ w.r.t.~the graph metric $d$.
 

\begin{definition}[Length measure] \label{def:measure} 
Let $ \lambda^*$ be the unique Borel measure on $\G$ such that the restriction of $\lambda^*$ on any edge is the length measure of that edge. That is, $\lambda^*$  satisfies:
\begin{enumerate}
\item[i)] For  any edge $e$ connecting two nodes $u$ and $v$, we have 
 $\lambda^*(\langle x,y\rangle) = (t-s) w_e$ 
 whenever $x = (1-s) u + s v$ and $y = (1-t)u + t v$ for $s,t \in [0,1)$ with $s \leq t$. Here, $\langle x,y\rangle$ is the line segment in $e$ connecting $x$ and $y$.
 \item[ii)] For any Borel set $F \subset \G$, we have
 \[
 \lambda^*(F) = \sum_{e\in E} \lambda^*(F\cap e).
 \]
\end{enumerate}
\end{definition}



The next lemma asserts that $\lambda^*$ is closely connected to the graph metric $d$, and thus justifies the terminology of a length measure.
\begin{lemma}[$\lambda^*$ is the length measure on graph] \label{lem:length-measure}
Suppose that $\G$ has no short cuts, namely, any edge $e$ is a shortest path connecting its two end-points. Then, $\lambda^*$ is a length measure in the sense that
\[
\lambda^*([x,y]) = d(x,y)
\]
for  any  shortest path   $[x,y]$ connecting $x$ and $y$. In particular, $\lambda^*$ has no atom.
\end{lemma}


The measure $\lambda^*$  is special as it is  linked to the metric distance. For trees, $\calS_1$ defined w.r.t. $\lambda^*$ is the same as the Wasserstein distance with cost $d(x,y)$.

\begin{corollary}[Tree topology] \label{cor:tree}
Suppose that the graph $\G$ is a tree and the distance $\calS_1$ is defined  w.r.t.~the measure $\lambda^*$. Then, we have 
\[
\calS_1 \equiv \calW_1,
\]
where $\calW_1$ is the Wasserstein distance\footnote{The definition of  $\calW_1$ is recalled in the supplementary.} with cost $d$.
\end{corollary}

We do not know the exact relationship between $\calS_p$ and the $p$-Wasserstein distance  $ \calW_p
$ when $p>1$. However, the following result shows that $\calS_p$ is always lower bounded by $\calW_1$.

\begin{lemma}[Bounds]\label{w1-vs-sp}
 Suppose the graph $\G$ is a tree and the distance $\calS_p$ is defined  w.r.t.~the measure $\lambda^*$. Then, for any $1\leq p \leq \infty$, we have
\[
\calW_1(\mu,\nu) \leq \lambda^*(\G)^{\frac{1}{p'}} \calS_p(\mu,\nu).
\]
\end{lemma}

\subsection{Isometry Between $\calP(V)$ and a Bounded Convex Set in a Euclidean Space} \label{sec:isometry}
Assume that $V=\{z_0 \equiv x_1, \, x_2,..., \,x_n\}$. For a node $x_i$, let $N(x_i)$ denote the collection of all neighbor nodes of $x_i$, and let
\[
N'(x_i) \coloneqq \Big\{v\in N(x_i): \, d(v, z_0) = d(x_i, z_0) + w_{\langle x_i,v\rangle}\Big\}.
\] 
Also, for $2\leq i \leq n$, let  $\hat x_i$ denote  the unique node $x \in N(x_i)$ such that the shortest path $[z_0, x_i]$ passes through $x$, i.e., $x \in [z_0, x_i]$.


Let us now  take a closer look at the feature map
\[
\rho \in \calP(V) \longmapsto \alpha \coloneqq (\rho(\gamma_e\cap V) )_{e\in E}\in \R^m.
\]


Observe that the representation  vector $\alpha =(\alpha_e)_{e\in E}$ satisfies $\alpha_e \geq 0$, $\alpha_e = 0 $ if $\gamma_e =\emptyset$ and
\begin{align*}
&\sum_{e= \langle x_1,v\rangle: v\in N(x_1)} \alpha_e\leq 1, \\
&\sum_{e=\langle x_i, v\rangle: v\in N'(x_i)}\alpha_e \leq \alpha_{\langle \hat x_i, x_i\rangle} \,\,\, \forall i=2,...,n.
\end{align*}
 Hereafter we use the convention that if $N'(x_i) =\emptyset$, then the corresponding summation is interpreted as zero. We note that $N'(x_i) =\emptyset$ happens precisely when there is no shortest path from other nodes to $z_0$ that passes through $x_i$
 (this, in particular,  occurs  for nodes in the ``last level'').
 
Let $\calK$ denote the set of all vectors $\alpha\in  \R^m$ having the above specified  properties. Clearly, this is a bounded and convex set which is closed w.r.t.~the Euclidean metric in $\R^m$.
In the next proposition, we assume that the distance $\calS_{p}$ is defined w.r.t.~the measure $\lambda^*$ defined in Section~\ref{sub:connection}. 
This result shows that  there is a one-to-one correspondence between $\calP(V)$ and the set $\calK$.
\begin{proposition}[$\calP(V) $ isometric to $\calK$]\label{prop:isometry}
The map
\begin{equation}\label{feature-map}
\rho \in \calP(V) \longmapsto \alpha \coloneqq (\rho(\gamma_e\cap V) )_{e\in E}\in \calK
\end{equation}
is one-to-one and onto. In addition, for any $\alpha =(\alpha_e)_{e\in E}\in \calK$, if we let
\begin{equation}\label{a^i}
\begin{array}{rl}
a^1 &\coloneqq 1- \sum_{e= \langle x_1,v\rangle: v\in N(x_1)} \alpha_e, \\
a^i &\coloneqq \alpha_{\langle \hat x_i, x_i\rangle}  - \sum_{e=\langle x_i, v\rangle: v\in N'(x_i)}\alpha_e   
\end{array}
\end{equation}
for $i=2,...,n,$
then  $\rho \coloneqq \sum_{i=1}^n a^i \delta_{x_i}\in \calP(V)$. Finally,  the  distance   $\calS_{p}$ on $\calP(V)$ is the same as the weighted $\ell_p$ distance on $\calK$, that is,  
\[
\calS_{p}(\rho_1, \rho_2) = \Big(\sum_{e\in E}    w_e \big| \alpha^1_e -  \alpha^2_e\big|^p\Big)^\frac1p,
\]
with $\alpha^i \coloneqq \alpha(\rho_i)$ for $i =1,2$.
\end{proposition}

The isometry is an useful properties of the Sobolev transport since any problem on the space of probability measures with Sobolev transport metric $\calS_p$ can be recasted as a corresponding problem on a bounded convex set of vectors in a Euclidean space with $\ell_p$ metric.

\subsection{Kernels for Sobolev Transport}
    

Our next result about negative definiteness\footnote{We follow the definition of  negative-definiteness in \citep[pp. 66--67]{berg1984harmonic}. A review about kernels is placed in the supplementary (Section \ref{appsec:review}).} is the key to build positive definite kernels upon Sobolev transport for kernel machines.

\begin{proposition}[Negative definiteness]\label{prop:neg_def}
    Suppose that the Sobolev transport distance $\calS_p$ is defined w.r.t. the length measure $\lambda^*$ on graph $\G$ for probability measures in $\calP(V)$, then for $1 \le p \le 2$, $\calS_p$ and $\calS_p^p$ are negative definite. 
\end{proposition}

From Proposition~\ref{prop:neg_def} and following 
\citep[Theorem~3.2.2, pp.74]{berg1984harmonic}, given $t > 0$, $1 \le p \le 2$ and $\mu, \, \nu\in \calP(V)$, the kernels 
\begin{eqnarray*}
&k_{\calS_p}(\mu, \nu) \coloneqq \exp(-t \calS_p(\mu, \nu)), \\
&k_{\calS^p_p}(\mu, \nu) \coloneqq \exp(-t \calS_p^p(\mu, \nu))
\end{eqnarray*}
are positive definite. 


\subsection{Sliced Sobolev Transport Distance} \label{sec:slice}

As remark after Definition~\ref{def:distance}, our Sobolev transport distance depends on the choice of the root node $z_0$ satisfying Assumption~\ref{a:unique}. When there are multiple possible root nodes, each choice of $z_0$ imposes its own geometry on the graph, which characterizes differently the graph derivative $f'$ of the function $f$. To alleviate the dependence in this case, and inspired by the slicing approach in optimal transport~\citep{rabin2011wasserstein, ref:le2019tree, pmlr-v130-le21a} for practical applications, we propose the \textit{sliced} Sobolev transport distance that fuses the Sobolev transport distance in Section~\ref{sec:distance}. Towards this end, let $\mathcal{Z}_0 \subset V$ be a (sub)set of unique-path root nodes:
\[
    \mathcal{Z}_0 \coloneqq \{ z_0 \in V : z_0 \text{ satisfies Assumption~\ref{a:unique}} \}.
\]
The sliced Sobolev transport averages over a sampling distribution $\eta$ on $\mathcal{Z}_0$, and is formally defined as follows.

\begin{definition}[Sliced Sobolev transport]\label{def:SlicedST} 
    Let $\eta$ be a probability measure on $\mathcal{Z}_0$. The sliced Sobolev transport is defined as
   \begin{align*}
        \calS_p^\eta(\mu, \nu) \hspace{-0.2em} \coloneqq \hspace{-0.3em} \int_{\mathcal{Z}_0} \hspace{-0.3em} \calS_p^{z_0}(\mu,\nu) \eta (\mathrm{d} z_0) \hspace{-0.1em} = \hspace{-0.2em} \sum_{ z_0\in \mathcal{Z}_0} \hspace{-0.2em} \eta(\{z_0\})\, \calS_p^{z_0}(\mu,\nu),
    \end{align*} 
    where $\calS_p^{z_0}$ is the Sobolev transport distance in Definition~\ref{def:distance} that is specific to the choice of the unique-path root node $z_0$.
\end{definition}

Because $\calS_p^\eta$ is a convex combination of $\calS_p^{z_0}$, we can readily verify that $\calS_p^\eta$ is also a distance. The proof is relegated to the supplementary.
\begin{proposition}[Metric] \label{prop:slice-metric}
    The sliced Sobolev transport $\calS_p^\eta$ is a distance on $\calP(\G)$.
\end{proposition}

\section{NUMERICAL EXPERIMENTS} \label{sec:numerical}

We evaluate the performance of our proposed Sobolev transport on two applications: (i) document classification with word embedding and (ii) topological data analysis (TDA). 

\paragraph{Probability Measures Representation.} We first describe probability measure representation for documents with word embedding in document classification and persistent diagrams for geometric structured data in TDA.

\begin{itemize}
\item \textbf{Documents with Word Embedding.} We consider each document as a probability measure where each word and its frequency in the document are regarded as a support and a corresponding weight in the probability measure respectively. We then follow the approach in \citet{kusner2015word, ref:le2019tree} to use \textit{word2vec} word embedding \citep{mikolov2013distributed} pretrained on Google News\footnote{https://code.google.com/p/word2vec} for documents. The pretrained \textit{word2vec} contains about $3$ millions words/phrases. Consequently, each word in a document is mapped into a vector in $\R^{300}$. We also remove SMART stop words \citep{salton1988term} or words in documents which are not available in the pretrained \textit{word2vec}.

\item \textbf{Persistence Diagrams.} TDA provides a powerful toolkit to analyze complicated geometric structured data, e.g., object shape, material data, or linked twist maps \citep{adams2017persistence, ref:le2019tree}. TDA leverages algebraic topology methods (e.g., persistence homology) to extract robust topological features (e.g., connected components, rings, cavities) and yield a multiset of points in $\R^2$ which is also known as persistence diagram (PD). The two coordinates of a point in PD are corresponding to the birth and death time of a topological feature respectively. Therefore, each point in PD summarizes a life span of a particular topological feature. We regard each PD as an empirical measure where each 2-dimensional point in PD is considered as a support with a uniform weight in the empirical measure.
\end{itemize}

Note that supports in document classification are in a high-dimensional space (i.e., $\R^{300}$) while supports in TDA are in a low-dimensional space (i.e., $\R^2$). Therefore, these applications allow us to observe how the dimension of supports affects performances. We next describe various graphs with different sizes (i.e., given graph metrics which we assume in applications) considered in our experiments.

\begin{figure*}[ht]
  \begin{center}
    \includegraphics[width=0.75\textwidth]{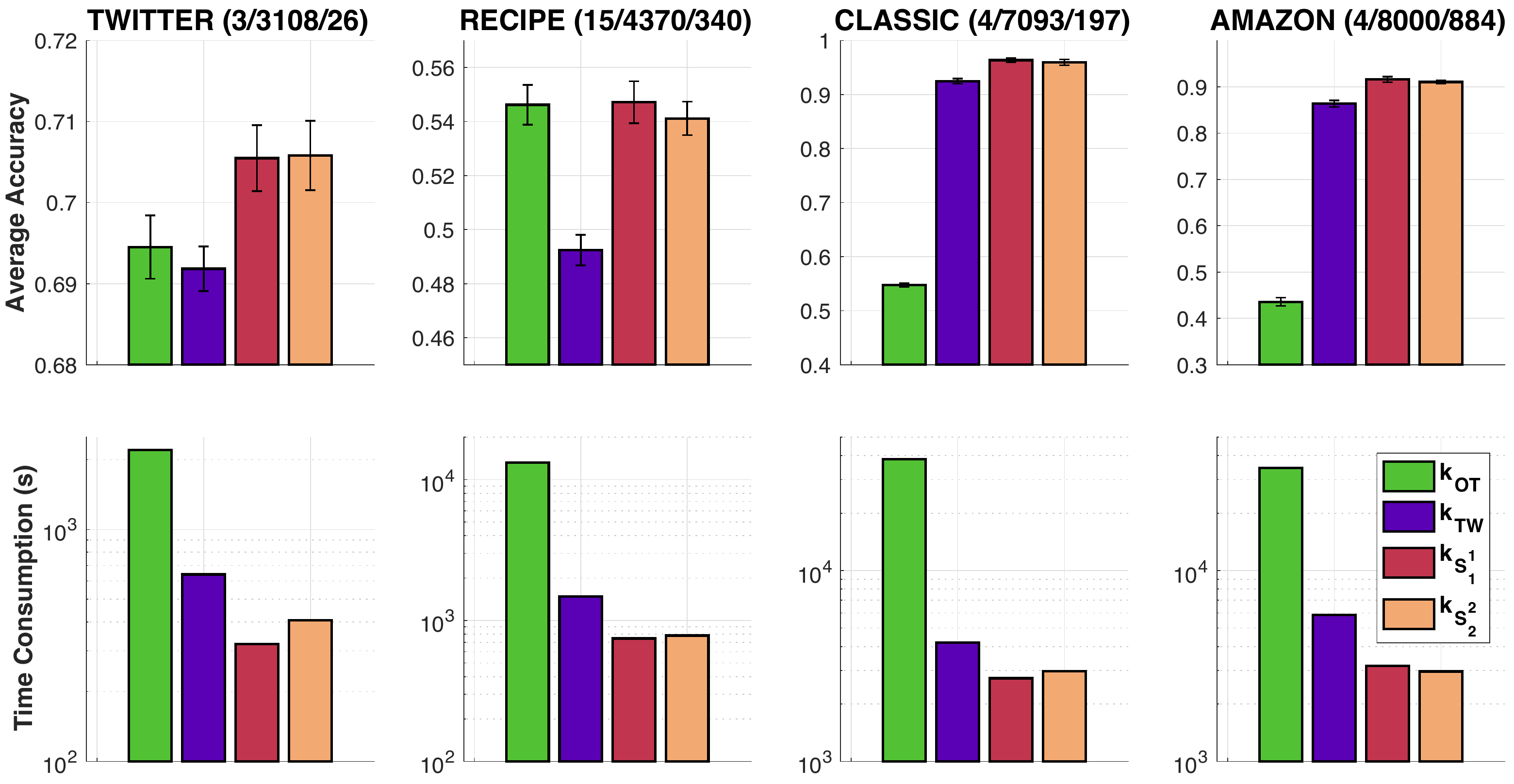}
  \end{center}
  \vspace{-6pt}
  \caption{SVM results and time consumption for kernel matrices in document classification with graph $\G_{\text{Log}}$. For each dataset, the numbers in the parenthesis are the number of classes; the number of documents; and the maximum number of unique words for each document respectively.}
  \label{fg:DOC_10KLog_main}
 \vspace{-6pt}
\end{figure*}

\begin{figure*}[ht]
  \begin{center}
    \includegraphics[width=0.75\textwidth]{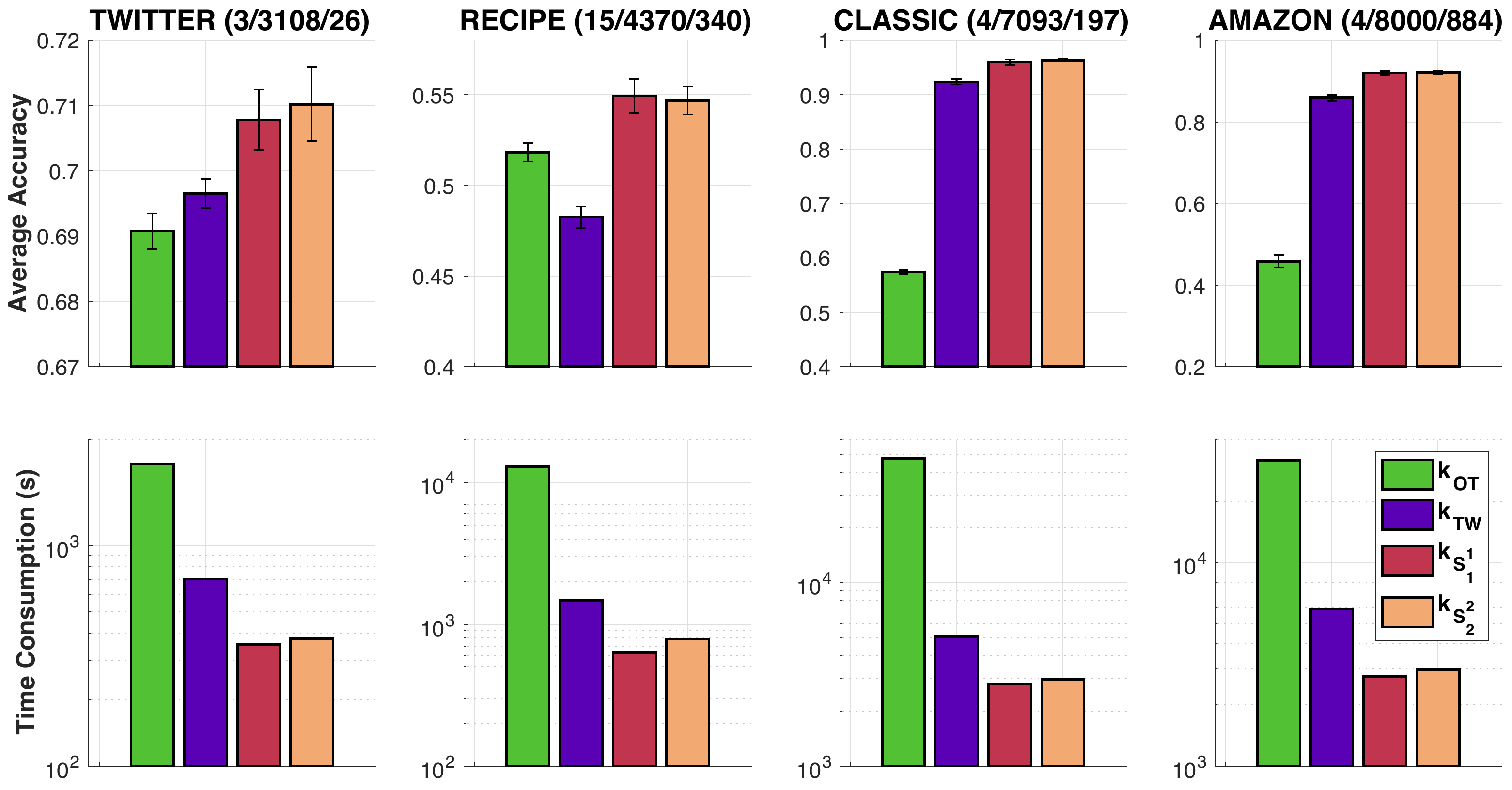}
  \end{center}
  \vspace{-6pt}
  \caption{SVM results and time consumption for kernel matrices in document classification with graph $\G_{\text{Sqrt}}$.}
  \label{fg:DOC_10KPow_main}
 \vspace{-6pt}
\end{figure*}

\paragraph{Graph Metric Construction.} For simplicity, we use a random graph metric for supports of probability measures as follow:

We first apply a clustering method, e.g., the farthest-point clustering, to partition supports of probability measures into at most $M$ clusters.\footnote{We set $M$ for the number of clusters when running the clustering method. Depending on input data, we obtain at most $M$ clusters.} We assign $V$ to be the set of centroids of these clusters. For edges, we consider two options: randomly choose (i) $M\log(M)$ edges or (ii) $M^{3/2}$ edges. For an edge $e$, its corresponding weight $w_e$ is computed by the Euclidean distance between the two nodes of that edge $e$. Let $\tilde{E}$ be the set of those randomly sampled edges and $n_c$ be the number of connected components in the graph $\tilde{\G}(V, \tilde{E})$, we then randomly add $(n_c - 1)$ more edges between these $n_c$ connected components to construct a connected graph $\G$ from $\tilde{\G}$. Denote $E_c$ as the set of these $(n_c - 1)$ added edges and $E = \tilde{E} \cup E_c$, then $\G(V, E)$ is the considered graph. 

We next describe baseline methods and detailed setup for our experiments.

\paragraph{Baselines and Setup.} We consider two typical baseline distances based on OT theory for probability measures supported on a graph metric space: (i) the optimal transport (OT) $d_{\text{OT}}$ with a graph metric cost (i.e., an instance of min-cost flow problem via Beckman formulation~~\citep[Section 6.3]{peyre2019computational}) and (ii) the tree-Wasserstein~\citep{ref:le2019tree} (TW) $d_{\text{TW}}$ where the tree structure is randomly sampled from the graph $\G$. In all experiments, we consider the kernels $k_{\calS_1}$ and $k_{\calS_2^2}$ for the proposed Sobolev transport distances and baseline kernels $k_{\text{OT}}(\cdot, \cdot) \coloneqq \exp(-td_{\text{OT}}(\cdot, \cdot))$ and $k_{\text{TW}}(\cdot, \cdot) \coloneqq \exp(-td_{\text{TW}}(\cdot, \cdot))$ for the corresponding OT distance $d_{\text{OT}}$ and TW distance $d_{\text{TW}}$ respectively. 

Following \citet{ref:le2019tree}, we evaluate those kernels with support vector machine (SVM) for document classification with word embedding and some tasks in TDA, e.g., the orbit recognition and object shape classification. Note that $k_{\calS_1}$, $k_{\calS_2^2}$ and $k_{\text{TW}}$ are positive definite, but $k_{\text{OT}}$ is empirically indefinite.\footnote{Generally, OT spaces are not Hilbertian~\citep[Section 8.3]{peyre2019computational}. Additionally, we also empirically observe that the Gram matrix for $k_{\text{OT}}$ has negative eigenvalues.} Similar as the approach in \citet{ref:le2019tree}, we regularize for the Gram matrix of $k_{\text{OT}}$ by adding a sufficiently large diagonal term. For multi-class classification, we employ 1-vs-1 strategy with Libsvm.\footnote{https://www.csie.ntu.edu.tw/$\sim$cjlin/libsvm/}


For each dataset, we randomly split it into $70\%/30\%$ for training and test with 10 repeats. We typically choose hyper-parameters via cross validation. For kernel hyperparameter, we choose $1/t$ from $\{q_{s}, 2q_{s}, 5q_{s}\}$ with $s = 10, 20, \dotsc, 90$ where $q_s$ is the $s\%$ quantile of a subset of corresponding distances observed on a training set. For SVM regularization hyperparameter, we choose it from $\left\{0.01, 0.1, 1, 10, 100\right\}$. We also consider a various number of nodes $M = 10^2, 10^3, 10^4, 4\times10^4$ for $\G$. Reported time consumption for all methods includes their corresponding preprocessing, e.g., compute shortest paths for Sobolev transport and OT, or sample random tree structure from graph for TW.


\subsection{Document Classification}

We consider 4 document datasets: \texttt{TWITTER}, \texttt{RECIPE}, \texttt{CLASSIC} and \texttt{AMAZON}. The statistical characteristics of these datasets are summarized in Figure~\ref{fg:DOC_10KLog_main}.



\subsection{Topological Data Analysis (TDA)}
For TDA, we consider the orbit recognition and the object shape classification. 



\subsubsection{Orbit Recognition}
We consider the synthesized dataset as in~\citet{adams2017persistence} for link twist map which are discrete dynamical systems to model flows in DNA microarrays~\citep{hertzsch2007dna}. There are $5$ classes of orbits in the dataset. Following \citet{le2018persistence}, for each class, we generated $1000$ orbits where each orbit contains $1000$ points. We consider the 1-dimensional topological features (i.e., connected components) for PD which are extracted with Vietoris-Rips complex filtration \citep{edelsbrunner2008persistent}. The statistical characteristics are summarized in Figure~\ref{fg:TDA_mix10KLog_main}.

\subsubsection{Object Shape Classification}

We consider a subset of \texttt{MPEG7} dataset \citep{latecki2000shape} having 10 classes and each class has 20 samples as in \citet{le2018persistence}. For simplicity, we follow the approach in \citet{le2018persistence} to extract $1$-dimensional topological features (i.e., connected components) for PD with Vietoris-Rips complex filtration~\citep{edelsbrunner2008persistent}. 


\begin{figure}[h]
  \begin{center}
    \includegraphics[width=0.45\textwidth]{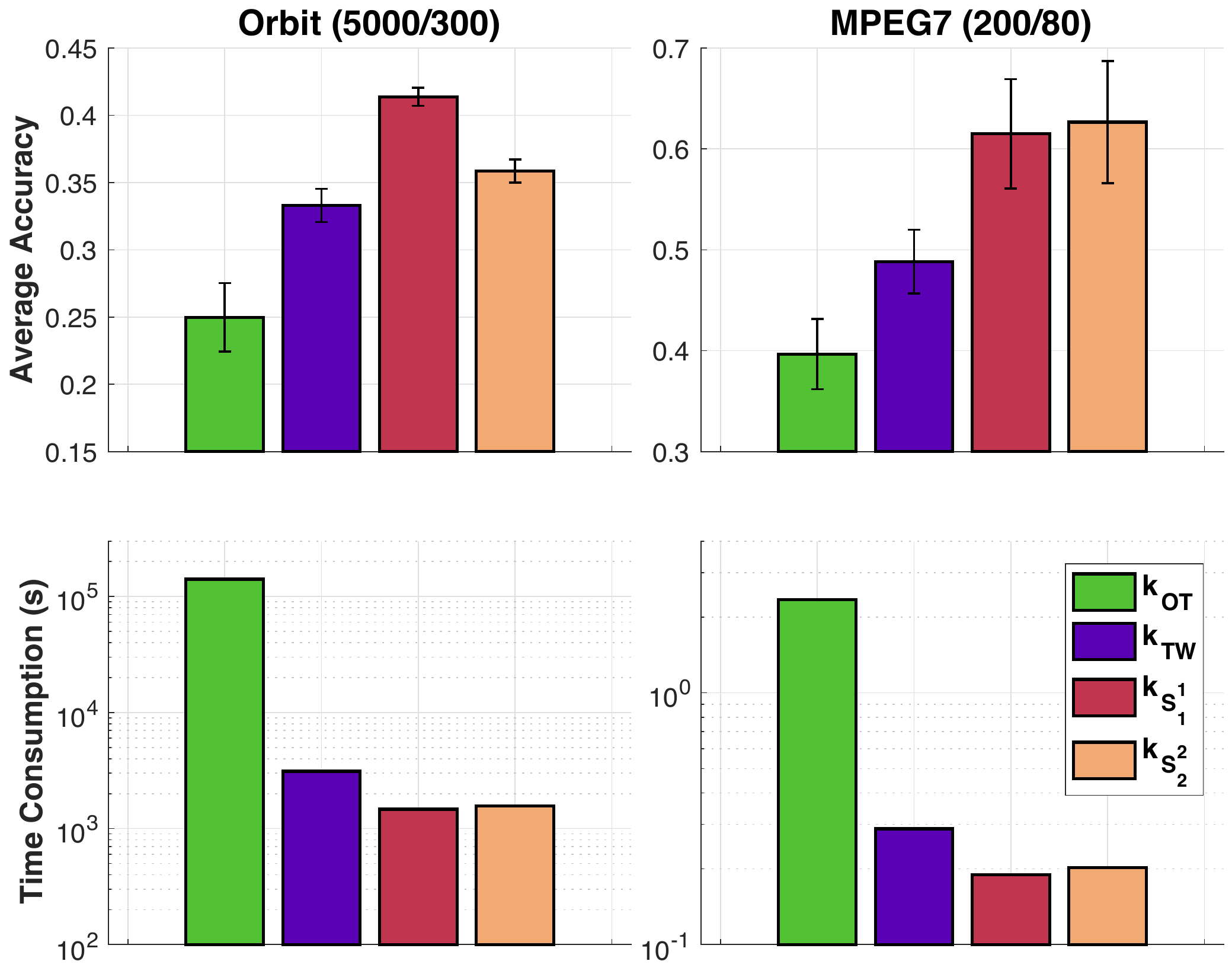}
  \end{center}
  \vspace{-6pt}
  \caption{SVM results and time consumption for kernel matrices in TDA with graph $\G_{\text{Log}}$. For each dataset, the numbers in the parenthesis are respectively the number of PD; and the maximum number of points in PD.}
  \label{fg:TDA_mix10KLog_main}
 \vspace{-6pt}
\end{figure}

\begin{figure}[h]
  \begin{center}
    \includegraphics[width=0.45\textwidth]{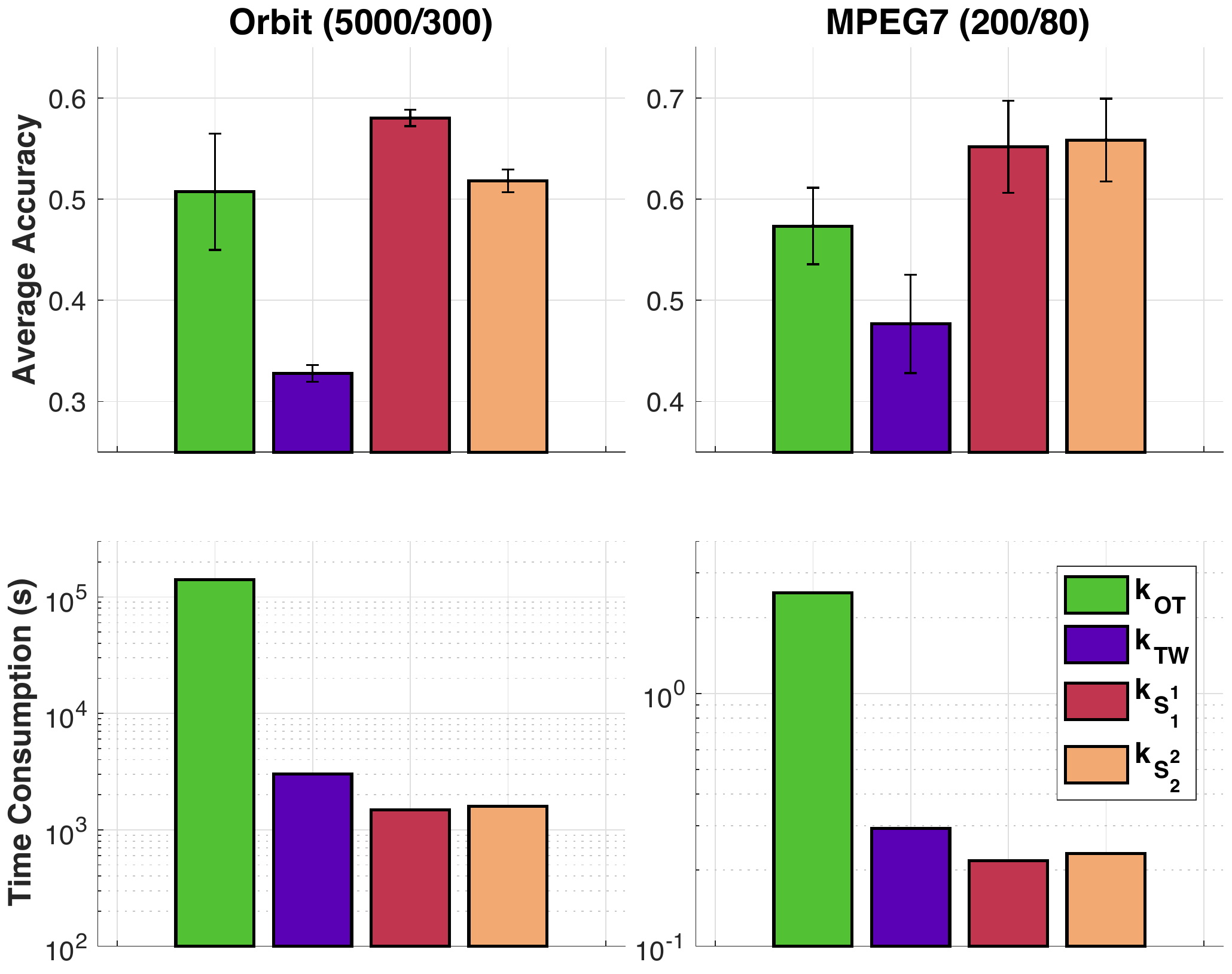}
  \end{center}
  \vspace{-6pt}
  \caption{SVM results and time consumption for kernel matrices in TDA with graph $\G_{\text{Sqrt}}$}
  \label{fg:TDA_mix10KPow_main}
 \vspace{-6pt}
\end{figure}

\begin{figure*}[ht]
  \begin{center}
    \includegraphics[width=0.78\textwidth]{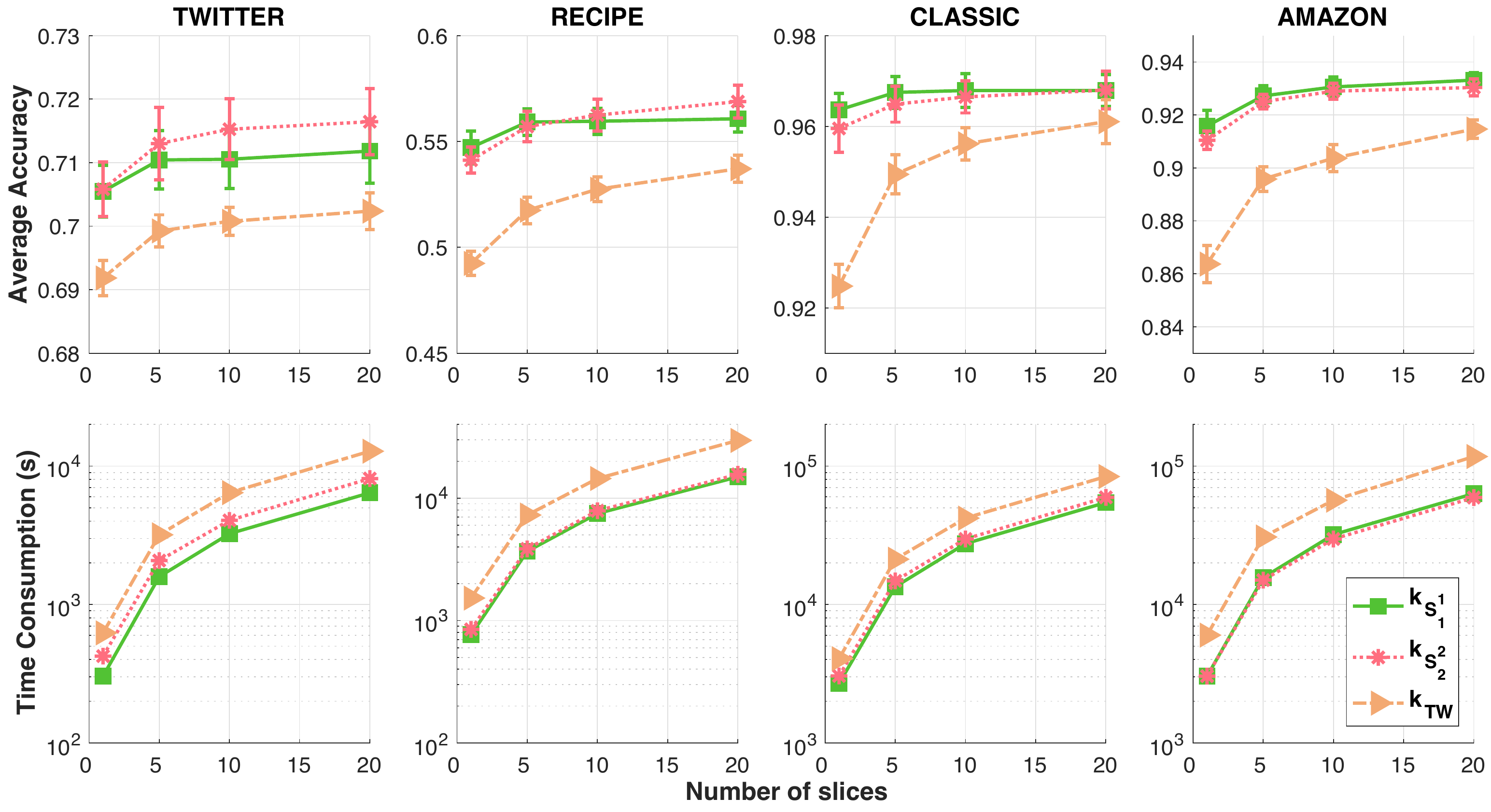}
  \end{center}
  \vspace{-6pt}
  \caption{SVM results and time consumption for kernel matrices of slice variants for Sobolev transport and tree-Wasserstein in document classification with graph $\G_{\text{Log}}$.}
  \label{fg:DOC_10KLog_SLICE_main}
 \vspace{-6pt}
\end{figure*}

\subsection{SVM Results, Time Consumption and Discussions}

We report results for graphs with $M=10^4$ for all datasets except \texttt{MPEG7} where $M=10^3$ due to its small size, and for both cases: (i) with $M\log(M)$ edges and (ii) with $M^{3/2}$ edges, and we denote those graphs as $\G_{\text{Log}}$ and $\G_{\text{Sqrt}}$ respectively.

In Figures~\ref{fg:DOC_10KLog_main} and \ref{fg:DOC_10KPow_main}, we illustrate the SVM results for document classification with word embedding with $\G_{\text{Log}}$ and $\G_{\text{Sqrt}}$, respectively. For TDA, we illustrate the results in Figures~\ref{fg:TDA_mix10KLog_main} and \ref{fg:TDA_mix10KPow_main} for $\G_{\text{Log}}, \G_{\text{Sqrt}}$ respectively. The performances of $k_{\calS_1}$, $k_{\calS_2^2}$ compare favorably with those of $k_{\text{OT}}$ and $k_{\text{TW}}$. Moreover, the time consumption of the Gram matrices for $k_{\calS_1}$, $k_{\calS_2^2}$ is comparative with that of $k_{\text{TW}}$ and is several-order faster than that of $k_{\text{OT}}$. Especially, in \texttt{Orbit} dataset, it took about more than \emph{$39$ hours} to compute the Gram matrix for $k_{\text{OT}}$, but only about \emph{$25$ minutes} for either $k_{\calS_1}$ or $k_{\calS_2^2}$. Recall that $k_{\text{OT}}$ is indefinite, this infiniteness may affect the performances of $k_{\text{OT}}$ in applications (in most of the experiments except the ones in \texttt{RECIPE} dataset with $\G_{\text{Log}}$ and in \texttt{Orbit} dataset with $\G_{\text{Sqrt}}$).

In Figure~\ref{fg:DOC_10KLog_SLICE_main}, we illustrate performances of slice variants for $k_{\calS_1}$, $k_{\calS_2^2}$ and $k_{\text{TW}}$ for document classification with word embedding with $\G_{\text{Log}}$. When we use more slices, the performances are improved. However, its computation is also linearly increased.

Further results are placed in the supplementary (Section  B).

\section{CONCLUSION}

In this paper, we have presented a scalable variant of optimal transport, namely the Sobolev transport, for probability measures supported on a graph (i.e., graph metric ground cost). By exploiting the graph-based Sobolev space structure, the proposed Sobolev transport distance admits a closed form solution for a fast computation. Moreover, the Sobolev transport is negative definite which allows to build positive definite kernels required in many kernel machine frameworks. We believe that exploiting local structures on supports such as tree or graph can improve the scalability for several optimal transport problems.


\subsubsection*{Acknowledgements}
We thank anonymous reviewers and area chairs for their comments. TL acknowledges the support of JSPS KAKENHI Grant number 20K19873. The research of TN is supported in part by a grant from the Simons Foundation ($\#318995$). 

\balance
\bibliographystyle{apalike}
\bibliography{aistats2022_CR_v2.bbl}

\begin{thebibliography}{}

\bibitem[Adams et~al., 2017]{adams2017persistence}
Adams, H., Emerson, T., Kirby, M., Neville, R., Peterson, C., Shipman, P.,
  Chepushtanova, S., Hanson, E., Motta, F., and Ziegelmeier, L. (2017).
\newblock Persistence images: A stable vector representation of persistent
  homology.
\newblock {\em Journal of Machine Learning Research}, 18(1):218--252.

\bibitem[Altschuler et~al., 2021]{altschuler2021averaging}
Altschuler, J.~M., Chewi, S., Gerber, P., and Stromme, A.~J. (2021).
\newblock Averaging on the {B}ures-{W}asserstein manifold: {D}imension-free
  convergence of gradient descent.
\newblock {\em Advances in Neural Information Processing Systems}.

\bibitem[Berg et~al., 1984]{berg1984harmonic}
Berg, C., Christensen, J. P.~R., and Ressel, P. (1984).
\newblock {\em Harmonic Analysis on Semigroups: Theory of Positive Definite and
  Related Functions}.
\newblock Springer.

\bibitem[Blanchet et~al., 2021]{ref:blanchet2021statistical}
Blanchet, J., Murthy, K., and Nguyen, V.~A. (2021).
\newblock Statistical analysis of {W}asserstein distributionally robust
  estimators.
\newblock {\em INFORMS TutORials in Operations Research}.

\bibitem[Bunne et~al., 2019]{bunne2019}
Bunne, C., Alvarez-Melis, D., Krause, A., and Jegelka, S. (2019).
\newblock Learning generative models across incomparable spaces.
\newblock In {\em International Conference on Machine Learning (ICML)},
  volume~97.

\bibitem[Courty et~al., 2017]{courty2017joint}
Courty, N., Flamary, R., Habrard, A., and Rakotomamonjy, A. (2017).
\newblock Joint distribution optimal transportation for domain adaptation.
\newblock In {\em Advances in Neural Information Processing Systems}, pages
  3730--3739.

\bibitem[Cuturi, 2013]{Cuturi-2013-Sinkhorn}
Cuturi, M. (2013).
\newblock Sinkhorn distances: {L}ightspeed computation of optimal transport.
\newblock In {\em Advances in Neural Information Processing Systems}, pages
  2292--2300.

\bibitem[Edelsbrunner and Harer, 2008]{edelsbrunner2008persistent}
Edelsbrunner, H. and Harer, J. (2008).
\newblock Persistent homology-{A} survey.
\newblock {\em Contemporary Mathematics}, 453:257--282.

\bibitem[Evans and Matsen, 2012]{ref:evans2012phylogenetic}
Evans, S. and Matsen, F. (2012).
\newblock The phylogenetic {K}antorovich-{R}ubinstein metric for environmental
  sequence samples.
\newblock {\em Journal of the Royal Statistical Society: Series B (Statistical
  Methodology)}, 74(3):569--592.

\bibitem[Fatras et~al., 2021]{fatras2021unbalanced}
Fatras, K., S{\'e}journ{\'e}, T., Flamary, R., and Courty, N. (2021).
\newblock Unbalanced minibatch optimal transport; applications to domain
  adaptation.
\newblock In {\em International Conference on Machine Learning}, pages
  3186--3197. PMLR.

\bibitem[Hertzsch et~al., 2007]{hertzsch2007dna}
Hertzsch, J.-M., Sturman, R., and Wiggins, S. (2007).
\newblock {DNA} microarrays: {D}esign principles for maximizing ergodic,
  chaotic mixing.
\newblock {\em Small}, 3(2):202--218.

\bibitem[Janati et~al., 2020]{janati2020entropic}
Janati, H., Muzellec, B., Peyr{\'e}, G., and Cuturi, M. (2020).
\newblock Entropic optimal transport between unbalanced {G}aussian measures has
  a closed form.
\newblock {\em Advances in Neural Information Processing Systems}, 33.

\bibitem[Klicpera et~al., 2021]{klicpera2021scalable}
Klicpera, J., Lienen, M., and G{\"u}nnemann, S. (2021).
\newblock Scalable optimal transport in high dimensions for graph distances,
  embedding alignment, and more.
\newblock In {\em International Conference on Machine Learning}, pages
  5616--5627. PMLR.

\bibitem[Kuhn et~al., 2019]{ref:kuhn2019wasserstein}
Kuhn, D., Esfahani, P.~M., Nguyen, V.~A., and Shafieezadeh-Abadeh, S. (2019).
\newblock Wasserstein distributionally robust optimization: Theory and
  applications in machine learning.
\newblock {\em INFORMS TutORials in Operations Research}, pages 130--166.

\bibitem[Kusano et~al., 2017]{kusano2017kernel}
Kusano, G., Fukumizu, K., and Hiraoka, Y. (2017).
\newblock Kernel method for persistence diagrams via kernel embedding and
  weight factor.
\newblock {\em The Journal of Machine Learning Research}, 18(1):6947--6987.

\bibitem[Kusner et~al., 2015]{kusner2015word}
Kusner, M., Sun, Y., Kolkin, N., and Weinberger, K. (2015).
\newblock From word embeddings to document distances.
\newblock In {\em International conference on machine learning}, pages
  957--966.

\bibitem[Latecki et~al., 2000]{latecki2000shape}
Latecki, L.~J., Lakamper, R., and Eckhardt, T. (2000).
\newblock Shape descriptors for non-rigid shapes with a single closed contour.
\newblock In {\em Proceedings of the IEEE Conference on Computer Vision and
  Pattern Recognition (CVPR)}, volume~1, pages 424--429.

\bibitem[Lavenant et~al., 2018]{lavenant2018dynamical}
Lavenant, H., Claici, S., Chien, E., and Solomon, J. (2018).
\newblock Dynamical optimal transport on discrete surfaces.
\newblock In {\em SIGGRAPH Asia 2018 Technical Papers}, page 250. ACM.

\bibitem[Le et~al., 2021a]{le2021flow}
Le, T., Ho, N., and Yamada, M. (2021a).
\newblock Flow-based alignment approaches for probability measures in different
  spaces.
\newblock In {\em International Conference on Artificial Intelligence and
  Statistics}, pages 3934--3942. PMLR.

\bibitem[Le and Nguyen, 2021]{pmlr-v130-le21a}
Le, T. and Nguyen, T. (2021).
\newblock Entropy partial transport with tree metrics: Theory and practice.
\newblock In {\em Proceedings of The 24th International Conference on
  Artificial Intelligence and Statistics}, pages 3835--3843.

\bibitem[Le et~al., 2021b]{le2021adversarial}
Le, T., Nguyen, T., Yamada, M., Blanchet, J., and Nguyen, V.~A. (2021b).
\newblock Adversarial regression with doubly non-negative weighting matrices.
\newblock {\em Advances in Neural Information Processing Systems}, 34.

\bibitem[Le and Yamada, 2018]{le2018persistence}
Le, T. and Yamada, M. (2018).
\newblock Persistence {F}isher kernel: A {R}iemannian manifold kernel for
  persistence diagrams.
\newblock In {\em Advances in Neural Information Processing Systems}, pages
  10007--10018.

\bibitem[Le et~al., 2019]{ref:le2019tree}
Le, T., Yamada, M., Fukumizu, K., and Cuturi, M. (2019).
\newblock Tree-sliced variants of {W}asserstein distances.
\newblock In {\em Advances in Neural Information Processing Systems}.

\bibitem[Mena and Niles-Weed, 2019]{mena2019statistical}
Mena, G. and Niles-Weed, J. (2019).
\newblock Statistical bounds for entropic optimal transport: {S}ample
  complexity and the central limit theorem.
\newblock In {\em Advances in Neural Information Processing Systems}, pages
  4541--4551.

\bibitem[Mikolov et~al., 2013]{mikolov2013distributed}
Mikolov, T., Sutskever, I., Chen, K., Corrado, G.~S., and Dean, J. (2013).
\newblock Distributed representations of words and phrases and their
  compositionality.
\newblock In {\em Advances in Neural Information Processing Systems}, pages
  3111--3119.

\bibitem[Mroueh et~al., 2018]{Mroueh-2018-Sobolev}
Mroueh, Y., Li, C.-L., Sercu, T., Raj, A., , and Cheng, Y. (2018).
\newblock Sobolev {GAN}.
\newblock In {\em ICLR}, pages 5767--5777.

\bibitem[Mukherjee et~al., 2021]{mukherjee2021outlier}
Mukherjee, D., Guha, A., Solomon, J.~M., Sun, Y., and Yurochkin, M. (2021).
\newblock Outlier-robust optimal transport.
\newblock In {\em International Conference on Machine Learning}, pages
  7850--7860. PMLR.

\bibitem[M{\"u}ller, 1997]{muller1997integral}
M{\"u}ller, A. (1997).
\newblock Integral probability metrics and their generating classes of
  functions.
\newblock {\em Advances in Applied Probability}, 29(2):429--443.

\bibitem[Muzellec et~al., 2020]{muzellec2020missing}
Muzellec, B., Josse, J., Boyer, C., and Cuturi, M. (2020).
\newblock Missing data imputation using optimal transport.
\newblock In {\em International Conference on Machine Learning}, pages
  7130--7140. PMLR.

\bibitem[Nadjahi et~al., 2019]{nadjahi2019asymptotic}
Nadjahi, K., Durmus, A., Simsekli, U., and Badeau, R. (2019).
\newblock Asymptotic guarantees for learning generative models with the
  sliced-{W}asserstein distance.
\newblock In {\em Advances in Neural Information Processing Systems}, pages
  250--260.

\bibitem[Nguyen et~al., 2021a]{nguyen2021point}
Nguyen, T., Pham, Q.-H., Le, T., Pham, T., Ho, N., and Hua, B.-S. (2021a).
\newblock Point-set distances for learning representations of 3d point clouds.
\newblock In {\em Proceedings of the IEEE/CVF International Conference on
  Computer Vision (ICCV)}, pages 10478--10487.

\bibitem[Nguyen et~al., 2021b]{nguyen2021optimal}
Nguyen, V., Le, T., Yamada, M., and Osborne, M.~A. (2021b).
\newblock Optimal transport kernels for sequential and parallel neural
  architecture search.
\newblock In {\em International Conference on Machine Learning}, pages
  8084--8095. PMLR.

\bibitem[Paty et~al., 2020]{paty2020regularity}
Paty, F.-P., d’Aspremont, A., and Cuturi, M. (2020).
\newblock Regularity as regularization: Smooth and strongly convex {B}renier
  potentials in optimal transport.
\newblock In {\em International Conference on Artificial Intelligence and
  Statistics}, pages 1222--1232. PMLR.

\bibitem[Peyr{\'e} and Cuturi, 2019]{peyre2019computational}
Peyr{\'e}, G. and Cuturi, M. (2019).
\newblock Computational optimal transport.
\newblock {\em Foundations and Trends{\textregistered} in Machine Learning},
  11(5-6):355--607.

\bibitem[Rabin et~al., 2011]{rabin2011wasserstein}
Rabin, J., Peyr{\'e}, G., Delon, J., and Bernot, M. (2011).
\newblock Wasserstein barycenter and its application to texture mixing.
\newblock In {\em International Conference on Scale Space and Variational
  Methods in Computer Vision}, pages 435--446.

\bibitem[Salton and Buckley, 1988]{salton1988term}
Salton, G. and Buckley, C. (1988).
\newblock Term-weighting approaches in automatic text retrieval.
\newblock {\em Information Processing \& Management}, 24(5):513--523.

\bibitem[Scetbon et~al., 2021]{scetbon2021low}
Scetbon, M., Cuturi, M., and Peyr{\'e}, G. (2021).
\newblock Low-rank {S}inkhorn factorization.
\newblock {\em International Conference on Machine Learning (ICML)}.

\bibitem[Si et~al., 2021]{si2021testing}
Si, N., Murthy, K., Blanchet, J., and Nguyen, V.~A. (2021).
\newblock Testing group fairness via optimal transport projections.
\newblock {\em International Conference on Machine Learning}.

\bibitem[Solomon et~al., 2015]{solomon2015convolutional}
Solomon, J., De~Goes, F., Peyr{\'e}, G., Cuturi, M., Butscher, A., Nguyen, A.,
  Du, T., and Guibas, L. (2015).
\newblock Convolutional {W}asserstein distances: Efficient optimal
  transportation on geometric domains.
\newblock {\em ACM Transactions on Graphics (TOG)}, 34(4):66.

\bibitem[Sriperumbudur et~al., 2009]{sriperumbudur2009integral}
Sriperumbudur, B.~K., Fukumizu, K., Gretton, A., Sch{\"o}lkopf, B., and
  Lanckriet, G.~R. (2009).
\newblock On integral probability metrics, $\phi$-divergences and binary
  classification.
\newblock {\em arXiv preprint arXiv:0901.2698}.

\bibitem[Titouan et~al., 2019]{titouan2019optimal}
Titouan, V., Courty, N., Tavenard, R., and Flamary, R. (2019).
\newblock Optimal transport for structured data with application on graphs.
\newblock In {\em International Conference on Machine Learning}, pages
  6275--6284. PMLR.

\bibitem[Tong et~al., 2021]{ref:tong2021diffusion}
Tong, A.~Y., Huguet, G., Natik, A., Macdonald, K., Kuchroo, M., Coifman, R.,
  Wolf, G., and Krishnaswamy, S. (2021).
\newblock Diffusion earth mover’s distance and distribution embeddings.
\newblock In Meila, M. and Zhang, T., editors, {\em Proceedings of the 38th
  International Conference on Machine Learning}, volume 139 of {\em Proceedings
  of Machine Learning Research}, pages 10336--10346.

\bibitem[Weed and Berthet, 2019]{pmlr-v99-weed19a}
Weed, J. and Berthet, Q. (2019).
\newblock Estimation of smooth densities in {W}asserstein distance.
\newblock In {\em Proceedings of the Thirty-Second Conference on Learning
  Theory}, volume~99, pages 3118--3119.

\bibitem[Xu et~al., 2021]{Xu-2021-Towards}
Xu, M., Zhou, Z., Lu, G., Tang, J., Zhang, W., and Yu, Y. (2021).
\newblock Towards generalized implementation of {W}asserstein distance in
  {GAN}s.
\newblock {\em arXiv preprint arXiv:2012.03420v2}.

\end{thebibliography}


\clearpage
\appendix

\thispagestyle{empty}

\onecolumn \makesupplementtitle

The supplementary is organized into two parts.
\begin{itemize}
    \item In Section \ref{appsec:proofs}, we provide the proofs for the theoretical results in the main manuscript.
    \item In Section \ref{appsec:review}, we briefly review important aspects in our work, provide further experimental results and discussions about our proposed Sobolev transport.
\end{itemize} 

We note that we have released code for our proposals at
\begin{center}
\url{https://github.com/lttam/SobolevTransport}.
\end{center}

\section{PROOFS}\label{appsec:proofs}

\subsection{Proofs and Results for Section~\ref{sec:distance}} \label{appendix:distance}

\paragraph{For Lemma~\ref{lem:metrize}.}
 
 \begin{proof}[Proof of Lemma~\ref{lem:metrize}]
By taking $f=0$ in Definition~\ref{def:distance}, we see that $\calS_p(\mu,\nu )\geq 0$ for any $(\mu, \nu)$, thus $\calS_p$ is non-negative. Assume that $\calS_p(\mu,\nu ) =0$. Then we must have 
\begin{align}\label{int_identity}
\int_\G f(x) \mu(\mathrm{d}x) - \int_\G f(x) \nu(\mathrm{d}x) =0
\end{align}
for all $f\in W^{1,p'}(\G, \lambda)$ satisfying $\|f'\|_{L^{p'}(\G, \lambda)}\leq 1$. Indeed, since otherwise there exists $\tilde f\in W^{1,p'}(\G, \lambda)$ with  $\|\tilde f'\|_{L^{p'}(\G, \lambda)}\leq 1$, and 
\[
\int_\G \tilde f(x) \mu(\mathrm{d}x) - \int_\G \tilde f(x) \nu(\mathrm{d}x) < 0.
\]
Then by  taking $f=-\tilde f $ in Definition~\ref{def:distance}, we see that $\calS_p(\mu,\nu )> 0$ which contradicts the assumption $\calS_p(\mu,\nu ) =0$. Thus \eqref{int_identity} holds. It now follows   from \eqref{int_identity} that 
\[
\int_\G f(x) \mu(\mathrm{d}x) = \int_\G f(x) \nu(\mathrm{d}x) 
\qquad \mbox{for every}\quad f\in W^{1,p'}(\G, \lambda),
\]
giving $\mu =\nu$ as desired. To prove the symmetry of $\calS_p(\mu,\nu )$, observe that if $f\in W^{1,p'}(\G, \lambda)$ with $\|f'\|_{L^{p'}(\G, \lambda)}\leq 1$, then we also have $-f \in W^{1,p'}(\G, \lambda)$ with $\|-f'\|_{L^{p'}(\G, \lambda)} = \|f'\|_{L^{p'}(\G, \lambda)} \leq 1$. As a consequence, $\calS_p(\mu,\nu ) = \calS_p(\nu,\mu)$.
It remains to show that $\calS_p$ satisfies the triangle inequality. For this, let $\mu,\nu,\sigma\in \calP(\G)$. Then for any function $f\in W^{1,p'}(\G, \lambda)$ satisfying $\|f'\|_{L^{p'}(\G, \lambda)}\leq 1$, we have
\begin{align*}
\int_\G f(x) \mu(\mathrm{d}x) - \int_\G f(x) \nu(\mathrm{d}x) &= \Big[ \int_\G f(x) \mu(\mathrm{d}x) - \int_\G f(x) \sigma(\mathrm{d}x)\Big]  + \Big[\int_\G f(x) \sigma(\mathrm{d}x) - \int_\G f(x) \nu(\mathrm{d}x)\Big]\\
&\leq \calS_p(\mu,\sigma ) + \calS_p(\sigma,\nu ).
\end{align*}
This implies that $\calS_p(\mu,\nu )\leq \calS_p(\mu,\sigma ) + \calS_p(\sigma,\nu )$. We therefore conclude that $\calS_p$ is a metric on the space $\calP(\G)$.
\end{proof}

 \paragraph{For Proposition~\ref{prop:upper}.}
 \begin{proof}[Proof of Proposition~\ref{prop:upper}]
Let $ f\in W^{1,p'}(\G, \lambda)$ be such that $\|f'\|_{L^{p'}(\G, \lambda)}\leq 1$. Since $q' < p'$ and $\lambda(\G) < +\infty$, it follows from Jensen's inequality that $f' \in L^{q'}(\G, \lambda)$.
Now define 
\[
g(x) \coloneqq  a f(x) \mbox{ with } a\coloneqq \lambda(\G)^{\frac{1}{p'} - \frac{1}{q'} }.
\]
Then according to Definition~\ref{def:distance} we have $g\in W^{1,q'}(\G, \lambda)$ with $g'(x) = a f'(x)$. Hence by using Jensen's inequality we obtain 
\begin{align*}
\Big(\frac{1}{\lambda (\G)}\int_{\G} |f'(x)|^{q'} \lambda(\mathrm{d}x) \Big)^{\frac{p'}{q'}}\leq  \frac{1}{\lambda(\G)} \int_{\G} |f'(x)|^{p'} \lambda(\mathrm{d}x),
\end{align*} 
which yields
 \begin{align*}
 \|g'\|_{L^{q'}(\G, \lambda)}  = a \|f'\|_{L^{q'}(\G, \lambda)}
&\leq a \,\, \lambda(\G)^{\frac{1}{q'} -\frac{1}{p'}} \|f'\|_{L^{p'}(\G, \lambda)} \\
&= \|f'\|_{L^{p'}(\G, \lambda)} \leq 1.
\end{align*} 
Therefore,
\begin{align*}
a \Big[ \int_\G f(x) \mu(\mathrm{d}x) - \int_\G f(x) \nu(\mathrm{d}x) \Big] = \int_\G g(x) \mu(\mathrm{d}x) - \int_\G g(x) \nu(\mathrm{d}x)  \leq \calS_q(\mu,\nu ). 
\end{align*}
Since this holds for any $ f\in W^{1,p'}(\G, \lambda)$ satisfying$\|f'\|_{L^{p'}(\G, \lambda)}\leq 1$, we conclude that 
\[
\calS_p(\mu,\nu ) \leq a^{-1}\, \calS_q(\mu,\nu ). 
\]
This completes the proof.
\end{proof}

 \paragraph{For Proposition~\ref{prop:closed-form}.}
 
 \begin{proof}[Proof of Proposition~\ref{prop:closed-form}]
For $f\in W^{1,p'}(\G, \lambda)$,  we have by representation \eqref{FTC} that
\[
f(x) = f(z_0) + \int_{[z_0,x]} f'(y) \lambda(\mathrm{d}y).
\]
Let ${\bf{1}}_{[z_0,x]}(y)$ denote the indicator function of the shortest path $[z_0,x]$. That is, ${\bf{1}}_{[z_0,x]}(y)$ equals to $1$ if $y\in [z_0,x]$ and equals to $0$ otherwise.
Then  we obtain \begin{align*}
\int_\G f(x) \mu(\mathrm{d}x) &= f(z_0) \mu(\G) + \int_\G  \int_{[z_0,x]} f'(y) \lambda(\mathrm{d}y) \mu(\mathrm{d}x)\\
&= f(z_0) \mu(\G) + \int_\G  \int_{\G}  {\bf{1}}_{[z_0,x]}(y) \, f'(y) \lambda(\mathrm{d}y) \mu(\mathrm{d}x).
\end{align*}
By using Fubini's theorem to interchange the order of integration in the last expression, we further get \begin{align*}
\int_\G f(x) \mu(\mathrm{d}x) &= f(z_0) \mu(\G) + \int_\G  \int_{\G}  {\bf{1}}_{[z_0,x]}(y) \, f'(y)  \mu(\mathrm{d}x) \lambda(\mathrm{d}y)\\
&= f(z_0) \mu(\G) + \int_\G  \Big[\int_{\G}  {\bf{1}}_{[z_0,x]}(y) \,  \mu(\mathrm{d}x)\Big] f'(y)  \lambda(\mathrm{d}y) \\
& = f(z_0) \mu(\G)   + \int_{\G} f'(y)  \mu(\Lambda(y)) \, \lambda(\mathrm{d}y),
\end{align*}
where we have used the definition of $\Lambda(y)$ in \eqref{sub-graph} to obtain the last identity.

By exactly the same reason, we also have 
\begin{align*}
\int_\G f(x) \nu(\mathrm{d}x) 
 = f(z_0) \nu(\G)   + \int_{\G} f'(y)  \nu(\Lambda(y)) \, \lambda(\mathrm{d}y).
\end{align*}
Therefore, as $\mu(\G)=\nu(\G)$  we infer from Definition~\ref{def:distance} that  
\begin{align*}
\calS_{p}(\mu,\nu )  = 
\sup \limits_{f \in \mathbb{B}}~\int_{\G} f'(x) \big[ \mu(\Lambda(x)) -  \nu(\Lambda(x))\big] \, \lambda(\mathrm{d}x),
\end{align*}
where $\mathbb{B} \coloneqq  \Big\{ f\in W^{1,p'}(\G, \lambda):  \, \|f'\|_{L^{p'}(\G, \lambda)}\leq 1 \Big\}$.

Clearly, $\{ f': \, f\in \mathbb{B}\} \subset \{g\in L^{p'}(\G, \lambda): \, \|g\|_{L^{p'}(\G, \lambda)}\leq 1 \}$. On the other hand, for any $g\in L^{p'}(\G, \lambda)$ we have $g=f'$ with $f(x) \coloneqq  \int_{[z_0,x]} g(y) \lambda(\mathrm{d}y)\in W^{1,p'}(\G, \lambda)$. It follows that $\{ f': \, f\in \mathbb{B}\} = \{g\in L^{p'}(\G, \lambda): \, \|g\|_{L^{p'}(\G, \lambda)}\leq 1 \}$, and hence we can rewrite $\calS_{p}(\mu,\nu )$ as 
\begin{align}\label{Holder}
\calS_{p}(\mu,\nu )  =  
\sup \limits_{\|g\|_{L^{p'}(\G, \lambda)}\leq 1} \int_{\G} g(x) \big[ \mu(\Lambda(x)) -  \nu(\Lambda(x))\big] \, \lambda(\mathrm{d}x) =\Big[\int_{\G} | \mu(\Lambda(x)) -  \nu(\Lambda(x))|^p \, \lambda(\mathrm{d}x) \Big]^\frac1p,
\end{align}
which is the desired conclusion. Let us explain in details how to obtain the last identity in \eqref{Holder}. Firstly, by H\"older's inequality we  have
\begin{align*}
 \int_{\G} g(x) \big[ \mu(\Lambda(x)) -  \nu(\Lambda(x))\big] \, \lambda(\mathrm{d}x) 
 \leq \Big[\int_{\G} |g(x)|^{p'} \, \lambda(\mathrm{d}x) \Big]^{\frac{1}{p'}}\Big[\int_{\G} | \mu(\Lambda(x)) -  \nu(\Lambda(x))|^p \, \lambda(\mathrm{d}x) \Big]^\frac1p,
\end{align*}
and so
\begin{align*}
\sup \limits_{\|g\|_{L^{p'}(\G, \lambda)}\leq 1} \int_{\G} g(x) \big[ \mu(\Lambda(x)) -  \nu(\Lambda(x))\big] \, \lambda(\mathrm{d}x) \leq \Big[\int_{\G} | \mu(\Lambda(x)) -  \nu(\Lambda(x))|^p \, \lambda(\mathrm{d}x) \Big]^\frac1p.
\end{align*}
Secondly, by choosing 
\[g^*(x) = \frac{|r(x)|^{p-2} r(x)}{\|r\|_{L^p(\G, \lambda)}^{p-1}} 
\quad \mbox{with}\quad  r(x) \coloneqq \mu(\Lambda(x)) -  \nu(\Lambda(x))
\]
we see that $\|g^*\|_{L^{p'}(\G, \lambda)}= 1$ and $\int_{\G} g^*(x) \big[ \mu(\Lambda(x)) -  \nu(\Lambda(x))\big] \, \lambda(\mathrm{d}x) =  \Big[\int_{\G} | \mu(\Lambda(x)) -  \nu(\Lambda(x))|^p \, \lambda(\mathrm{d}x) \Big]^\frac1p$.
Thus we infer  that last identity in \eqref{Holder} holds true, and 
the function $g^*$ is  a maximizer for the  optimization problem in \eqref{Holder}. 
\end{proof}

\paragraph{For Corollary~\ref{cor:discrete}.}

 \begin{proof}[Proof of Corollary~\ref{cor:discrete}]
 We first recall that $\langle u,  v\rangle$ denotes the line segment in $\R^n$ connecting two points $u, v$, 
while $( u, v)$ means the same line segment but without its two end-points.
Then from Proposition~\ref{prop:closed-form} and as $\lambda$ has no atom, we get 
\[
\calS_{p}(\mu,\nu )^p 
=\sum_{e=\langle u,v\rangle\in E}   \int_{(u,v)} | \mu(\Lambda(x)) -  \nu(\Lambda(x))|^p \, \lambda(\mathrm{d}x).
\]
Since  $\mu$  and $\nu$ are supported on nodes, we can rewrite the above identity as
\begin{align*}
\calS_{p}(\mu,\nu )^p &=
\hspace{-0.3em}\sum_{e=\langle u,v\rangle\in E}   \int_{(u,v)} \hspace{-1em}| \mu(\Lambda(x)\setminus (u,v)) -  \nu(\Lambda(x)\setminus (u,v))|^p \, \lambda(\mathrm{d}x). 
\end{align*}
For $e=\langle u,v\rangle$ and $x\in (u,v)$, we observe that  $y\in \G\setminus (u,v)$ belongs to $\Lambda(x)$ if and only if $y\in \gamma_e$. It follows that $\Lambda(x)\setminus (u,v) =\gamma_e$, and thus 
\begin{align*}
\calS_{p}(\mu,\nu )^p 
=\hspace{-0.3em}\sum_{e=\langle u,v\rangle\in E}   \int_{(u,v)} \hspace{-1em}| \mu(\gamma_e) -  \nu(\gamma_e)|^p \, \lambda(\mathrm{d}x)
=\sum_{e\in E}    \big| \mu(\gamma_e) -  \nu(\gamma_e)\big|^p \lambda(e),
\end{align*}
which leads to the postulated result.
\end{proof}

\subsection{Proofs and Results for Section~\ref{sec:properties}}
\label{appendix:property}

\paragraph{For Lemma~\ref{lem:length-measure}.}

\begin{proof}[Proof of Lemma~\ref{lem:length-measure}]
Let  $[x,y]$ be a  shortest path connecting $x$ and $y$.  Assume that this path goes  through nodes $v_1, ..., v_k$, then obviously $\langle x,v_1\rangle$, $\langle v_1,v_2\rangle$, ...,$\langle v_k,y\rangle$ are corresponding shortest paths w.r.t.~its end-points. Therefore, it follows from Definition~\ref{def:measure} for $\lambda^*$ that 
\begin{align*}
\lambda^*([x,y]) &= \lambda^*(\langle x,v_1\rangle) + \lambda^*(\langle v_1,v_2\rangle) +\cdots + \lambda^*(\langle v_k,y\rangle)\\
& = d(x,v_1) + d(v_1, v_2) + \cdots + d(v_k,y) \\
& = d(x,y),
\end{align*}
where the last identity is due to the assumption that  $\G$ has no short cuts.
\end{proof}

\paragraph{For Corollary~\ref{cor:tree}.}

\begin{proof}[Proof of Corollary~\ref{cor:tree}]
This is a consequence of our Proposition~\ref{prop:closed-form} for $p=1$ and the results obtained 
in \citep{ref:evans2012phylogenetic, ref:le2019tree}. Indeed, when
$\G$ is a tree with root $z_0$, it is shown in \citep[Equation~(5)]{ref:evans2012phylogenetic} and in the proof of Proposition~1 in \citep{ ref:le2019tree} that the $1$-Wasserstein distance (see \eqref{equ:OTprob} for its definition) is given by 
\[
\vspace{-0.6em}
\calW_1(\mu,\nu)
= \int_{\G} | \mu(\Lambda(x)) -  \nu(\Lambda(x))| \, \lambda(\mathrm{d}x) 
\]
for any $\mu,\nu \in \calP(\G)$. 
By comparing this with our Proposition~\ref{prop:closed-form}, we conclude that $\calS_1(\mu,\nu) = \calW_1(\mu,\nu)$.
\end{proof}

\paragraph{For Lemma~\ref{w1-vs-sp}.}

\begin{proof}[Proof of Lemma~\ref{w1-vs-sp}]


This is a direct consequence of Proposition~\ref{prop:upper} and Corollary~\ref{cor:tree}. Indeed, we obtain from Proposition~\ref{prop:upper} that $\calS_1(\mu,\nu) \leq \lambda^*(\G)^{\frac{1}{p'}} \calS_p(\mu,\nu )$ for any $1\leq p \leq \infty$ (notice that the case $p=1$ is trivial since $\frac{1}{p'} = 0$). Therefore, the conclusion follows as  $\calS_1(\mu,\nu)= \calW_1(\mu,\nu)$ by Corollary~\ref{cor:tree}.
\end{proof}

\paragraph{For Proposition~\ref{prop:isometry}.}

\begin{proof}[Proof of Proposition~\ref{prop:isometry}] The last statement is just a consequence of Corollary~\ref{cor:discrete}.
For the second statement, observe that the condition $\alpha\in  \cal K$ ensures that $a^i\geq 0$ for all $i$. Also  $\sum_{i=1}^n a^i =1$ since by inspection it is easy to see that
\begin{align*}
\sum_{e= \langle x_1,v\rangle: v\in N(x_1)} \alpha_e + \sum_{i=2}^n \sum_{e=\langle x_i, v\rangle: v\in N'(x_i)}\alpha_e 
= \sum_{i=2}^n \alpha_{\langle \hat x_i, x_i\rangle} .
\end{align*}
Therefore,  $ \rho\coloneqq \sum_{i=1}^n a^i \delta_{x_i}$ is a probability distribution on $V$. That is, $\rho \in \calP(V)$.

The second statement also implies that the map 
\eqref{feature-map} is onto.  Indeed, for any given $\alpha =(\alpha_e)_{e\in E}\in \calK$, let $ \rho= \sum_{i=1}^n a^i \delta_{x_i} \in \calP(V)$ be the corresponding measure given by the second statement in Proposition~\ref{prop:isometry}. Let $e\in E$ be arbitrary. Then  either $\gamma_e =\emptyset$ or $\gamma_e \neq \emptyset$. In the first case, we obviously have $\rho(\gamma_e) = 0 =\alpha_e$. On the other hand, for the second case if we let  $x_i$ be the node on the edge $e$ with the smaller distance to $z_0$, then $\gamma_e = \{x_i\} \cup \Big(\cup_{e' = \langle x_i, v\rangle: v\in N'(x_i)} \gamma_{e'}\Big)$ and this is the disjoint union. Thus,
\begin{align*}
\rho(\gamma_e) = \rho(\{x_i\})  + \sum_{e'=\langle x_i, v\rangle: v\in N'(x_i)}\rho(\gamma_{e'}) = a^i + \sum_{e=\langle x_i, v\rangle: v\in N'(x_i)}\alpha_{e'} = \alpha_e, 
\end{align*}
where the second equality is due to the induction process by repeating and tracing back to the base case $N'(x_i) =\emptyset$ to show that $\rho(\gamma_{e'}) = \alpha_{e'}$,  and the last equality is  by  \eqref{a^i}.  Thus the map \eqref{feature-map} is onto. 


To show that the map \eqref{feature-map} is one-to-one, assume that there exist $\rho_1, \, \rho_2 \in \calP(V)$ such that $\rho_1(\gamma_e \cap V)=\rho_2(\gamma_e \cap V)$ for every $e\in E$. Let
$ \alpha \coloneqq (\rho_1(\gamma_e\cap V) )_{e\in E}= (\rho_2(\gamma_e \cap V) )_{e\in E}$, and define $\rho \coloneqq \sum_{i=1}^n a^i \delta_{x_i}\in \calP(V)$
with $a^i$ being given by \eqref{a^i}. Since
\begin{align*}
\rho_1(\{x_1\}) &= 1-  \sum_{e= \langle x_1,v\rangle: v\in N(x_1)}  \rho_1(\gamma_e \cap V ), \\
\rho_1(\{x_i\}) &\coloneqq \rho_1(\gamma_{\langle \hat x_i, x_i\rangle} ) - \sum_{e=\langle x_i, v\rangle: v\in N'(x_i)}\rho_1(\gamma_e \cap V)\quad \mbox{for } i=2,...,n,
\end{align*}
we infer that $\rho_1(\{x_i\}) =a^i$ for all $i$. Due to the above choice of the distribution $\rho$, we therefore conclude that  $\rho_1 =\rho$. By exactly the same reasoning, we also have $\rho_2 =\rho$. Thus $\rho_1 =\rho_2$, and hence the map \eqref{feature-map} is one-to-one.
So the first statement in Proposition~\ref{prop:isometry} holds true, and the proof  is complete.
\end{proof}

\paragraph{For Proposition~\ref{prop:neg_def}}

\begin{proof}[Proof of Proposition~\ref{prop:neg_def}]
Let $\ell_p$ be the distance on $\R^m$ defined by: for  $x, z \in \mathbb{R}^{m}$,  $\ell_p(x,z) = \norm{x-z}_p = \left(\sum_{i=1}^m \left|x_{(i)} - z_{(i)} \right|^p\right)^{1/p}$ where $x_{(i)}$ is the $i^{th}$ coordinate of $x$.  We will first prove that for $1 \le p \le 2$, the $\ell_p$ distance and $\ell_p^p$ are negative definite.

For $a, b \in \mathbb{R}$, it is obvious that the function $(a, b) \mapsto (a - b)^2$ is negative definite. Consider $1 \le p \le 2$ and follow \citep[Corollary~2.10, pp.78]{berg1984harmonic}, the function $(a, b) \mapsto |a - b|^p$ is negative definite. It follows that $\ell_p^p$ is negative definite since it  is a sum of negative definite functions. Using this and  by applying \citep[Corollary~2.10, pp.78]{berg1984harmonic} for the function $\ell^p_p$, we also have that the function $\ell_p$ is  negative definite. 

We are now ready to prove the negative definiteness for $\calS_p$ and $\calS_p^p$. Let $m$ be the number of edges in the graph $\G$. Due to 
Corollary~\ref{cor:discrete},
$\lambda^{*}(e)^\frac1p \mu(\gamma_e) = w_e^\frac1p \mu(\gamma_e)$ with $e \in E$ can be regarded as a feature map for probability measure $\mu$ onto $\R^{m}_{+}$. Therefore, $\calS_p$ is equivalent to the $\ell_p$ distance between these feature maps (see also Proposition~\ref{prop:isometry}). Hence, $\calS_p$ and $\calS_p^p$ are negative definite for $1 \le p \le 2$.
\end{proof}

\paragraph{For Proposition~\ref{prop:slice-metric}.}
\begin{proof}[Proof of Proposition~\ref{prop:slice-metric}] 
By Lemma~\ref{lem:metrize} we know that $\calS_p^{z_0}$ is a metric on $\calP(\G)$ for a given unique-root node $z_0$. On the other hand, according to Definition~\ref{def:SlicedST}, $\calS_p^\eta$ is 
a convex combination of the metric
$\calS_p^{z_0}$ with $z_0 \in \mathcal{Z}_0$\footnote{We assume that $\mathcal{Z}_0 \ne \emptyset$. This assumption is easily satisfied for general graph metric built from data points. See further discussion about the set $\mathcal{Z}_0$ (or the Assumption~\ref{a:unique} in the main text) in \S \ref{appsec:review}.}. Therefore, it follows immediately that $\calS_p^\eta$ is also a metric. Indeed, the nonnegativity and symmetry are obvious. Also, if $\calS_p^\eta(\mu,\nu)=0$ then we have $\calS_p^{z_0}(\mu,\nu) =0$ for every point  $z_0\in \mathcal{Z}_0$ satisfying $\eta(\{z_0\})>0$.  As $\sum_{ z_0\in \mathcal{Z}_0} \eta(\{z_0\}) = 1$, there must exists  a point $\tilde z_0\in \mathcal{Z}_0$ such that $\eta(\{\tilde z_0\})>0$. 
Thus we obtain $\calS_p^{\tilde z_0}(\mu,\nu) =0$, and hence   $\mu = \nu$ by Lemma~\ref{lem:metrize}. To check the triangle inequality, let $\mu,\nu, \sigma \in \calP(\G)$ be arbitrary. We then  use Definition~\ref{def:SlicedST} and  Lemma~\ref{lem:metrize} to get 
\begin{align*}
        \calS_p^\eta(\mu, \nu) 
        = \sum_{ z_0\in \mathcal{Z}_0} \eta(\{z_0\})\, \calS_p^{z_0}(\mu,\nu)
        &\leq \sum_{ z_0\in \mathcal{Z}_0} \eta(\{z_0\})\, \Big[ \calS_p^{z_0}(\mu,\sigma) + \calS_p^{z_0}(\sigma,\nu)\Big] \\
        &= \calS_p^\eta(\mu, \sigma) + \calS_p^\eta(\sigma, \nu).
    \end{align*} 
We thus conclude that $\calS_p^\eta$ is a metric  on $\calP(\G)$.
\end{proof}

\section{FUTHER RESULTS AND DISCUSSIONS}\label{appsec:review}

In this section, we give brief reviews about important aspects in our works, provide further experimental results and further discussions for our proposed Sobolev transport distance.

\subsection{Brief Reviews}

In this section, we briefly review about important aspects in our work and provide further experimental results.

\paragraph{For Kernels.} We review some important definitions (e.g., positive/negative definite kernels \citep{berg1984harmonic}) and theorems (e.g., Theorem 3.2.2 in \citet{berg1984harmonic}) about kernels used in our work.

\begin{itemize}

\item \textbf{Positive Definite Kernels \citep[pp.~66--67]{berg1984harmonic}.} A kernel function $k: \Omega \times \Omega \rightarrow \R$ is  positive definite if $\forall m \in \N^{*}, \forall x_1, x_2, ..., x_m \in \Omega$, we have 
\[
\sum_{i, j} c_i c_j k(x_i, x_j) \ge 0, \qquad \forall c_i \in \R.
\]

\item \textbf{Negative Definite Kernels \citep[pp.~66--67]{berg1984harmonic}.} A kernel function $k: \Omega \times \Omega \rightarrow \R$ is  negative definite if $\forall m \ge 2, \forall x_1, x_2, ..., x_m \in \Omega$, we have 
\[
\sum_{i, j} c_i c_j k(x_i, x_j) \le 0, \qquad \forall c_i \in \R \,\, \text{s.t.} \, \sum_i c_i = 0.
\]

\item \textbf{Theorem 3.2.2 in \citep[pp.~74]{berg1984harmonic} for Kernels.}
If $\kappa$ is a \textit{negative definite} kernel, then $\forall t > 0$, kernel 
\[
k_{t}(x, z) \coloneqq \exp{\left(- t \kappa(x, z)\right)}
\]
is positive definite.
\end{itemize}

\paragraph{For Persistence Diagrams and Definitions in Topological Data Analysis.} We refer the reader to \citet[\S2]{kusano2017kernel} for a review about mathematical framework for persistence diagrams (e.g., persistence diagrams, filtrations, persistent homology).

\paragraph{For the Integral Probability Metric.} Let $\mathfrak{F}$ be a class of real-valued bounded measurable functions on $\Omega$; $\mu, \nu$ be two Borel probability distributions on $\Omega$, then the integral probability metric $\mathcal{I}$ associated with $\mathfrak{F}$ \citep{muller1997integral} is defined as follow:
\[
\mathcal{I}_{\mathfrak{F}} \Let \sup_{f \in \mathfrak{F}} \left| \int_{\Omega} f(x) \mu(dx) - \int_{\Omega} f(z) \nu(dz) \right|.
\]

Some popular instances of the integral probability metrics are: (i) Dudley metric, (ii) Wasserstein metric, (iii) total variation metric, (iv) Kolmogorov metric, (v) maximum mean discrepancies, to name a few \citep{sriperumbudur2009integral, muller1997integral}.

\paragraph{For the 1-Wasserstein Distance.} Let $\mu$, $\nu$ be two Borel probability distributions on $\Omega$, $R(\mu, \nu)$ be the set of probability distributions $\pi$ on $\Omega \times \Omega$ such that $\pi(A \times \Omega) = \mu(A)$ and $\pi(\Omega \times B) = \nu(B)$ for all Borel sets $A$, $B$. The $1$-Wasserstein distance $\calW_1$ with a cost function $c$ is defined as follow:
\begin{equation}\label{equ:OTprob}
\calW_1(\mu, \nu) = \inf\left\{ \int_{\Omega \times \Omega} c(x, z) \pi(dx, dz) \mid \pi \in R(\mu, \nu) \right\}.
\end{equation}
Let $\mathcal{F}_c$ be the set of Lipschitz functions w.r.t. the cost function $c$, i.e. functions $f : \Omega \rightarrow \R$ such that $\left| f(x) - f(z)\right| \le c(x, z), \forall x, z \in \Omega$. Then, the dual of ~\eqref{equ:OTprob} is:
\begin{equation}\label{equ:KantorovichDualityProb}
\calW_1(\mu, \nu) = \sup_{f \in \mathcal{F}_c}\left\{\int_{\Omega} f(x) \mu(dx) - \int_{\Omega} f(z) \nu(dz) \right\}.
\end{equation}

\subsection{Further Discussions}


\paragraph{About the Assumption~\ref{a:unique}.} In our setting, the nodes in the graph are points in $\R^n$, edge weights are the distance (e.g., $\ell_2$ distance) between two corresponding nodes (i.e., points in $\R^n$). Therefore, consider any two nodes in the graph, there may be several paths connecting one node to the other node, and with a high probability, lengths of those paths are different. Hence, it is almost surely that every node in the graph can be regarded as unique-path root node.

In case, we have some special graph, e.g., a grid of nodes. There is no unique-path root node for such graph. However, we can easily \textit{adjust/approximate} such graph into a graph with unique-path root nodes by randomly perturbing each node of such graph in a ball (e.g., $\ell_2$ ball) with a small radius. 

\paragraph{About the Proposed Sobolev Transport Distance.} In our setting, we assume that we know the graph metric space (i.e., the graph structure) which supports of probability measures are living. Giving such graph, we define our Sobolev transport for probability measures supported on that graph metric space.

In our experiments (in Section~\ref{sec:numerical}), we evaluate our proposed Sobolev transport on (i) various graph structures (e.g., $\G_{\text{Log}}$ and $\G_{\text{Sqrt}}$) (ii) with different graph sizes (e.g., the number of nodes in the graphs $M = 10^2, 10^3, 10^4, 4 \times 10^4$. Performances of the Sobolev transport consistently compare favorably with those of the baseline approaches. 

The question about learning the optimal graph metric structure from data for the Sobolev transport is left for future work.

\paragraph{A Further Note on Implementation for the Sobolev Transport.} Following the closed-form solution of Sobolev transport for discrete probability measures supported on a graph metric space in Corollary~\ref{cor:discrete} and Equation~\eqref{eq:DiscreteSobolevTransport_Opt}, we need to compute the mass of $\mu, \nu$ on $\gamma_e$ for each edge $e \in E_{\mu, \nu}$. 

Recall that for any support $z$ of a probability measure, it only contributes to the $\gamma_e$ when $e$ belongs to the shortest path in $\G$ from the unique-path root node $z_0$ to the considered support $z$. Therefore, we only need to run the Dijkstra algorithm for shortest paths one time for the source $z_0$ and the destination $\left(V \setminus \{z_0\}
\right)$.\footnote{One can consider a set of all considered supports (exclude $z_0$) as the destination set for Dijkstra for a faster computation.} Then, we can index for each support $z$ in $\G$ for its contribution to each $\gamma_e$.\footnote{We only need to compute this step one time (i.e., it can be considered as the preprocessing process involving only the graph structure and nothing about the probability distributions, and is done only once regardless how many pairs $(\mu,\nu)$ that we have to measure. In this step by identifying shortest paths we calculate the set $\gamma_e$ for each edge $e\in E$.), see our Remark~\ref{rm:computeSobolev}.} 

Therefore, for a given probability measure $\mu$, we only need to consider each support of $\mu$ one time to compute $\mu(\gamma_e)$ for all edge $e$ in the graph $\G$ instead of a naive implementation where we need to consider all supports of $\mu$ for each $\gamma_e$ in $\G$.

\subsection{Further Experimental Results}

In this section, we provide further experimental results.

\paragraph{Further Results for Document Classification with Word Embedding for Different Values of $M$ (i.e., the Number of Nodes in the Graph).}

\begin{itemize}
    \item \textbf{For Graph $\G_{\text{Log}}$.} Similar to Figure~\ref{fg:DOC_10KLog_main} in the main text, we illustrate the SVM results and time consumption of kernel matrices for document classification with word-embedding for graph $\G_{\text{Log}}$ when $M=10^3$ and $M=10^2$ in Figures~\ref{fg:DOC_1K_Log_AccTime} and \ref{fg:DOC_100_Log_AccTime} respectively.

    \item \textbf{For Graph $\G_{\text{Sqrt}}$.} Similar to Figure~\ref{fg:DOC_10KPow_main} in the main text, we illustrate the SVM results and time consumption of kernel matrices for document classification with word-embedding for graph $\G_{\text{Sqrt}}$ when $M=10^3$ and $M=10^2$ in Figures~\ref{fg:DOC_1K_Pow_AccTime} and \ref{fg:DOC_100_Pow_AccTime} respectively.

\end{itemize}

\paragraph{Further Results for TDA for Different Values of $M$ (i.e., the Number of Nodes in the Graph).}

\begin{itemize}
    \item \textbf{For Graph $\G_{\text{Log}}$.} Similar to Figure~\ref{fg:TDA_mix10KLog_main} in the main text, we illustrate the SVM results and time consumption for TDA for graph $\G_{\text{Log}}$ when $M=10^3$ and $M=10^2$ in Figure~\ref{fg:TDA_Other_Log_AccTime}.

    \item \textbf{For Graph $\G_{\text{Sqrt}}$.} Similar to Figure~\ref{fg:TDA_mix10KPow_main} in the main text, we illustrate the SVM results and time consumption for TDA for graph $\G_{\text{Sqrt}}$ when $M=10^3$ and $M=10^2$ in Figure~\ref{fg:TDA_Other_Pow_AccTime}.
\end{itemize}

\paragraph{Further Results with Large Graph ($M=40000$).} We illustrate the SVM results and time consumption of kernel matrices for large graph with $M=40000$ for both $\G_{\text{Log}}$ and $\G_{\text{Sqrt}}$ in Figure~\ref{fg:LargeGraph_40K_AccTime}.

\paragraph{Further Results for Slice Variants of Sobolev Transport and Tree-Wasserstein.} Similar as Figure~\ref{fg:DOC_10KLog_SLICE_main} in the main text, we illustrate further results for both document classification with word embedding and TDA for slice variants of Sobolev trransport and tree-Wasserstein: (i) for both $\G_{\text{Log}}$ and $\G_{Sqrt}$, (ii) for different values of $M$ (e.g., $10^2, 10^3, 10^4$).

\begin{itemize}
    \item \textbf{For Document Classification with Word Embedding.}
    \begin{itemize}
        \item \textbf{For Graph $\G_{\text{Log}}$.} We illustrate SVM results and time consumption of kernel matrices for sliced variants of Sobolev transport and tree-Wasserstein for document classification with word-embedding for graph $\G_{\text{Log}}$ when $M=10^3$ and $M=10^2$ in Figures \ref{fg:DOC_1K_Log_AccTime_SLICE} and \ref{fg:DOC_100_Log_AccTime_SLICE} respectively.
        
        \item \textbf{For Graph $\G_{\text{Sqrt}}$.} We illustrate SVM results and time consumption of kernel matrices for sliced variants of Sobolev transport and tree-Wasserstein for document classification with word-embedding for graph $\G_{\text{Log}}$ when $M=10^4$, $M=10^3$ and $M=10^2$ in Figures \ref{fg:DOC_10K_Pow_AccTime_SLICE}, \ref{fg:DOC_1K_Pow_AccTime_SLICE}, and \ref{fg:DOC_100_Pow_AccTime_SLICE} respectively.
    \end{itemize}
    
    \item \textbf{For TDA.}
    \begin{itemize}
        \item \textbf{For Graph $\G_{\textbf{Log}}$.} We illustrate SVM results and time consumption of kernel matrices for sliced variants of Sobolev transport and tree-Wasserstein for TDA for graph $\G_{\text{Log}}$ when $M=10^4$ for \texttt{Orbit} and $M=10^3$ for \texttt{MPEG7} in Figure~\ref{fg:TDA_mix10K1K_Log_AccTime_SLICE} (due to a small size of the dataset \texttt{MPEG7}); and with $M=10^3$ and $M=10^2$ for both datasets in Figure \ref{fg:Other_TDA_Log_AccTime_SLICE}.
        
        \item \textbf{For Graph $\G_{\textbf{Sqrt}}$.} We illustrate SVM results and time consumption of kernel matrices for sliced variants of Sobolev transport and tree-Wasserstein for TDA for graph $\G_{\text{Sqrt}}$ when $M=10^4$ for \texttt{Orbit} and $M=10^3$ for \texttt{MPEG7} in Figure~\ref{fg:TDA_mix10K1K_Pow_AccTime_SLICE} (due to a small size of the dataset \texttt{MPEG7}); and with $M=10^3$ and $M=10^2$ for both datasets in Figure \ref{fg:Other_TDA_Pow_AccTime_SLICE}.
    \end{itemize}
\end{itemize}

\paragraph{Further Results with Large Graph ($M=40000$) for sliced variants.} We illustrate the SVM results and time consumption of kernel matrices for sliced variants of Sobolev transport and tree-Wasserstein for large graphs where the number of nodes is $M=40000$ for both $\G_{\text{Log}}$ and $\G_{\text{Sqrt}}$ in Figure~\ref{fg:Other_TDA_Large_AccTime_SLICE}.

\paragraph{Further Results for Tree-Wasserstein Kernel.} We illustrate the SVM results for tree-Wasserstein kernel with the minimum spanning tree of the given graph, denote as $k_{\text{TW}^{\text{MST}}}$ for both graphs $\G_{\text{Log}}$ and $\G_{\text{Sqrt}}$ where the number of nodes is $M=10000$ on document classification in Figure~\ref{fg:Other_DOC_MST_AccTime}. The performances of $k_{\text{TW}^{\text{MST}}}$ improves those of $k_{\text{TW}}$ (with random trees from a given graph).

\paragraph{Discussions.} Through various tasks (e.g., document classification with work embedding and TDA), with various graph structure (e.g., $\G_{\text{Log}}$ and $\G_{\text{Sqrt}}$) with different graph sizes (e.g., the number of nodes in the graphs $M = 10^2, 10^3, 10^4, 4 \times 10^4$), the performances of the proposed Sobolev transport consistently compare favorably with those of other baselines. The Sobolev transport is several-order faster than the optimal transport with graph metric. Additionally, the Sobolev transport can leverage information from the graph which is more flexible and has more degree of freedom in applications than tree-Wasserstein (for tree structure). The question about learning the optimal graph structure from data is left for future work. We also think that local structures on supports such as graph structure in our work or tree structure in \citet{ref:le2019tree, pmlr-v130-le21a, le2021flow} play an important role to scale up problems in optimal transport, especially for large-scale applications.


\begin{figure}
  \begin{center}
    \includegraphics[width=0.7\textwidth]{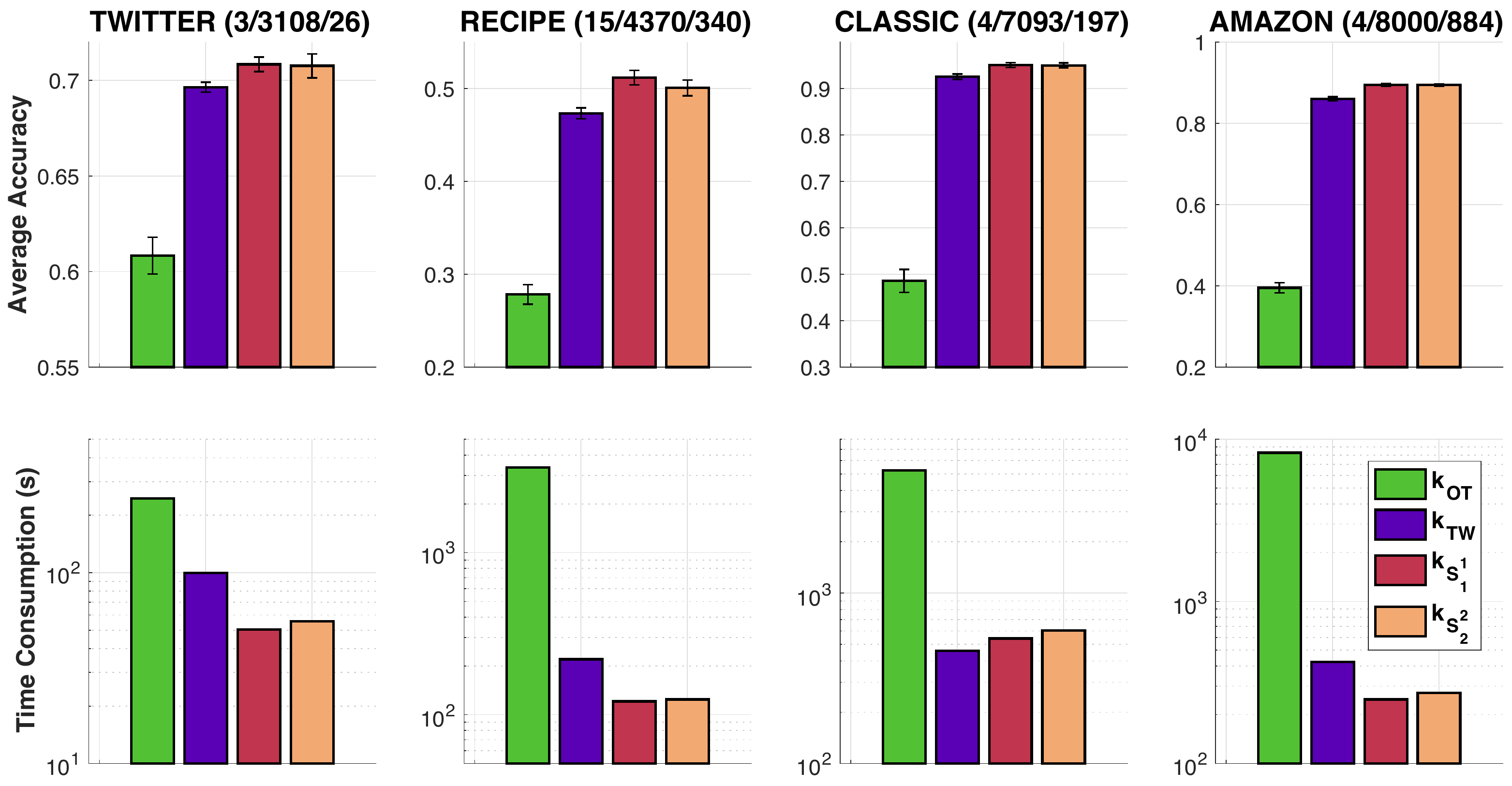}
  \end{center}
  \vspace{-6pt}
  \caption{SVM results and time consumption for kernel matrices with $\G_{\text{Log}}$ where $M = 10^3$.}
  \label{fg:DOC_1K_Log_AccTime}
 \vspace{-10pt}
\end{figure}

\begin{figure}
  \begin{center}
    \includegraphics[width=0.7\textwidth]{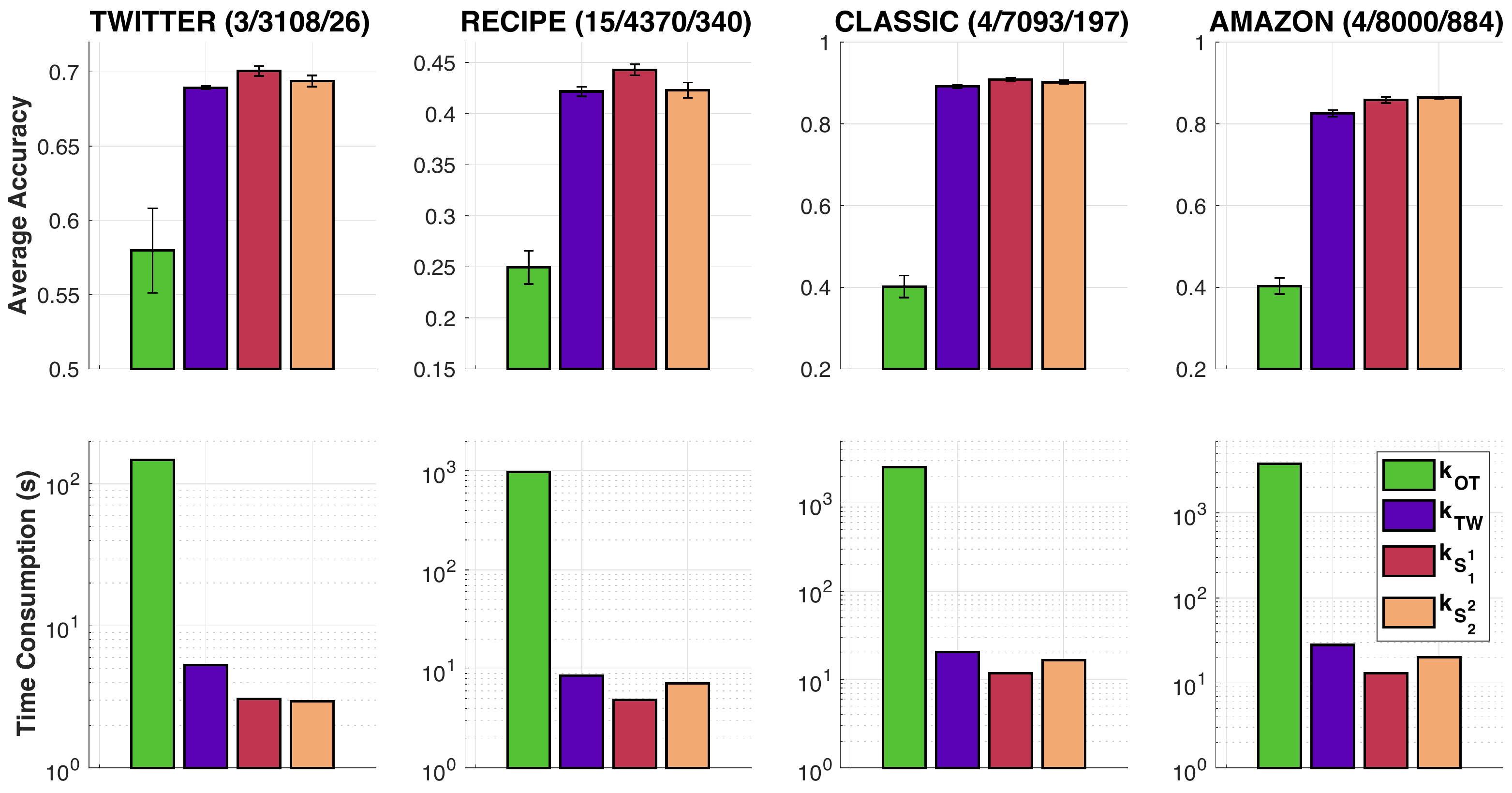}
  \end{center}
  \vspace{-6pt}
  \caption{SVM results and time consumption for kernel matrices with $\G_{\text{Log}}$ where $M = 10^2$.}
  \label{fg:DOC_100_Log_AccTime}
 \vspace{-10pt}
\end{figure}


\begin{figure}
  \begin{center}
    \includegraphics[width=0.7\textwidth]{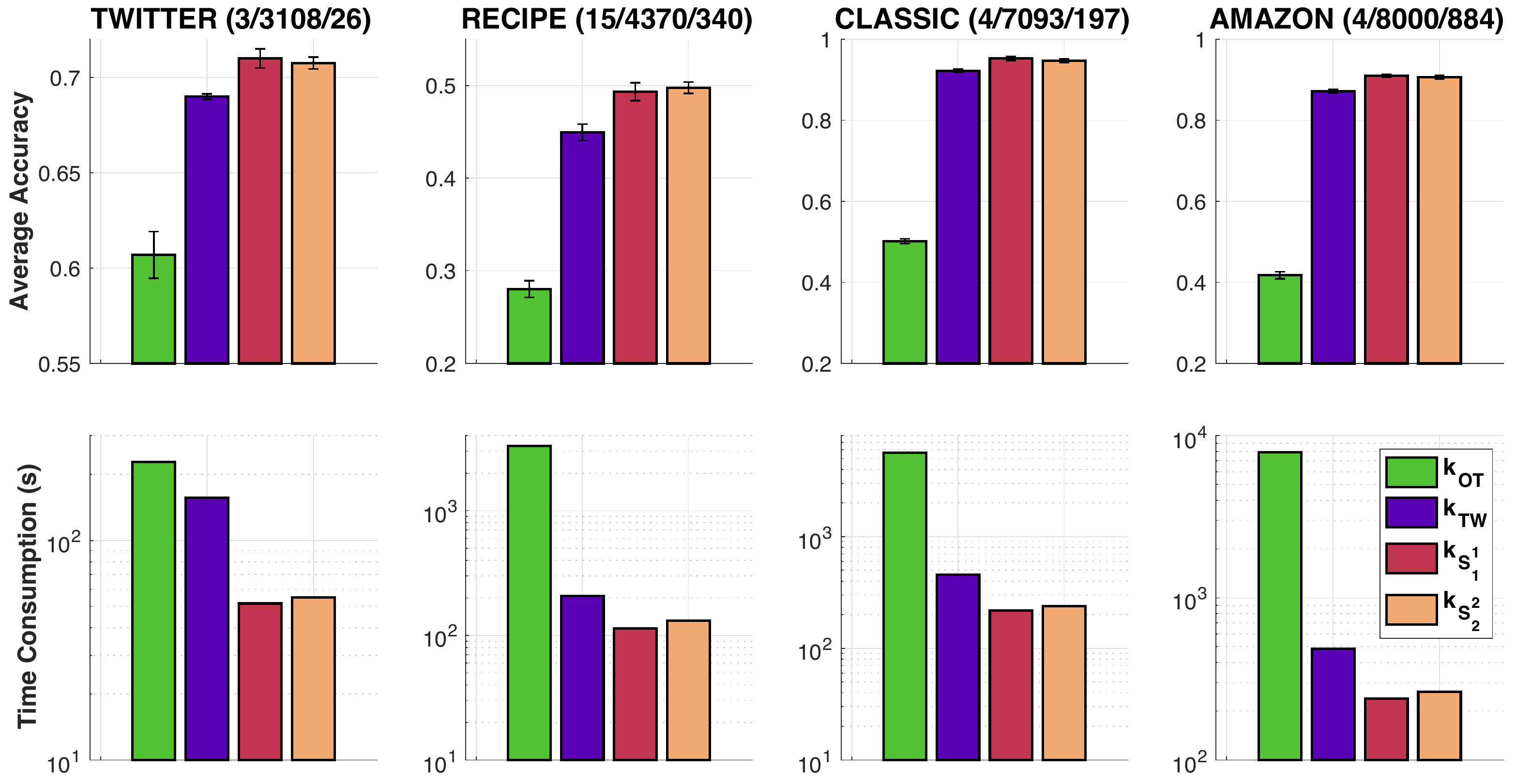}
  \end{center}
  \vspace{-6pt}
  \caption{SVM results and time consumption for kernel matrices with $\G_{\text{Sqrt}}$ where $M = 10^3$.}
  \label{fg:DOC_1K_Pow_AccTime}
 \vspace{-10pt}
\end{figure}

\begin{figure}
  \begin{center}
    \includegraphics[width=0.7\textwidth]{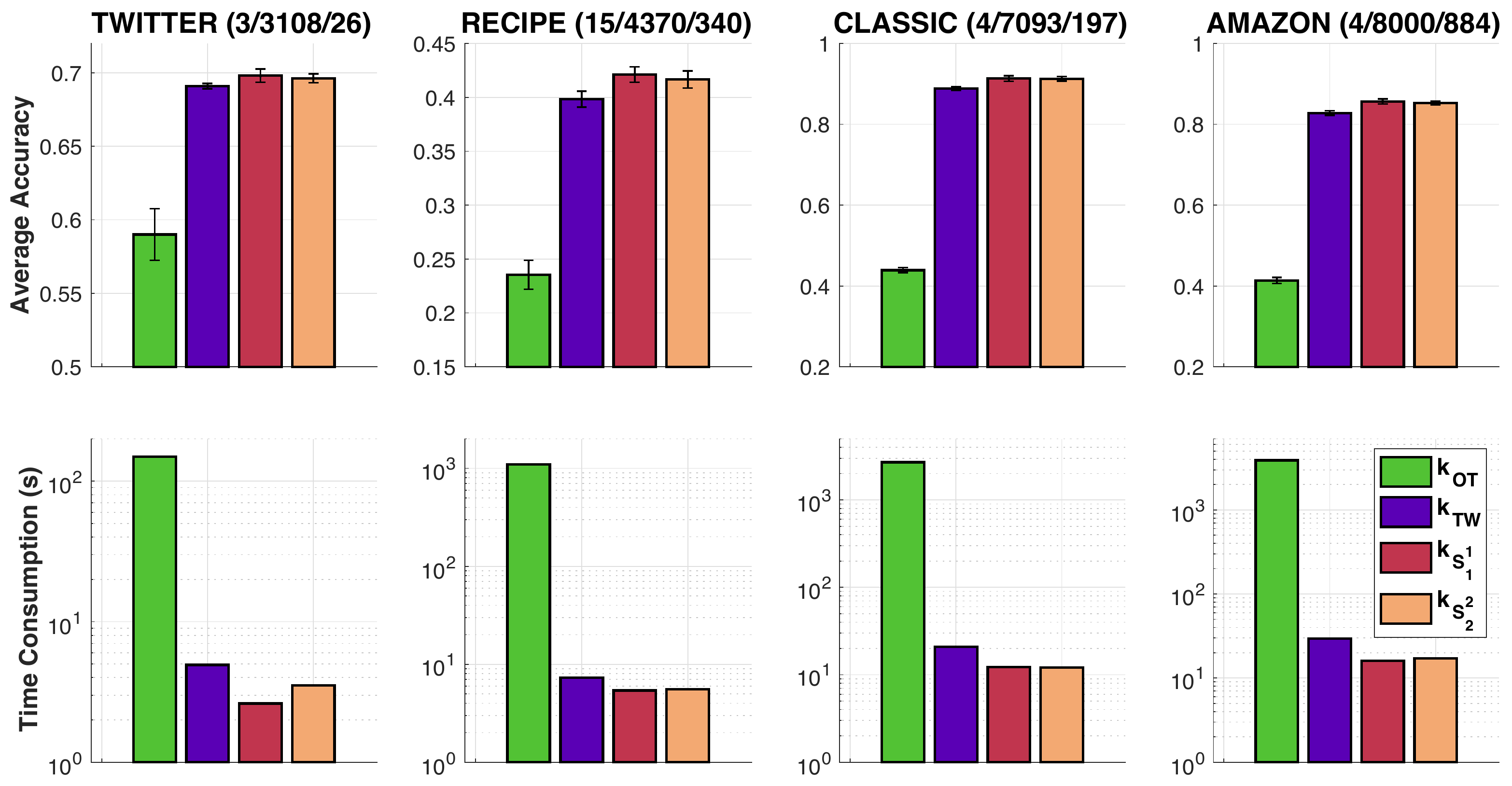}
  \end{center}
  \vspace{-6pt}
  \caption{SVM results and time consumption for kernel matrices with $\G_{\text{Sqrt}}$ where $M = 10^2$.}
  \label{fg:DOC_100_Pow_AccTime}
 \vspace{-10pt}
\end{figure}


\begin{figure}
    \centering
\begin{subfigure}[b]{0.48\textwidth}
    \includegraphics[width=0.9\textwidth]{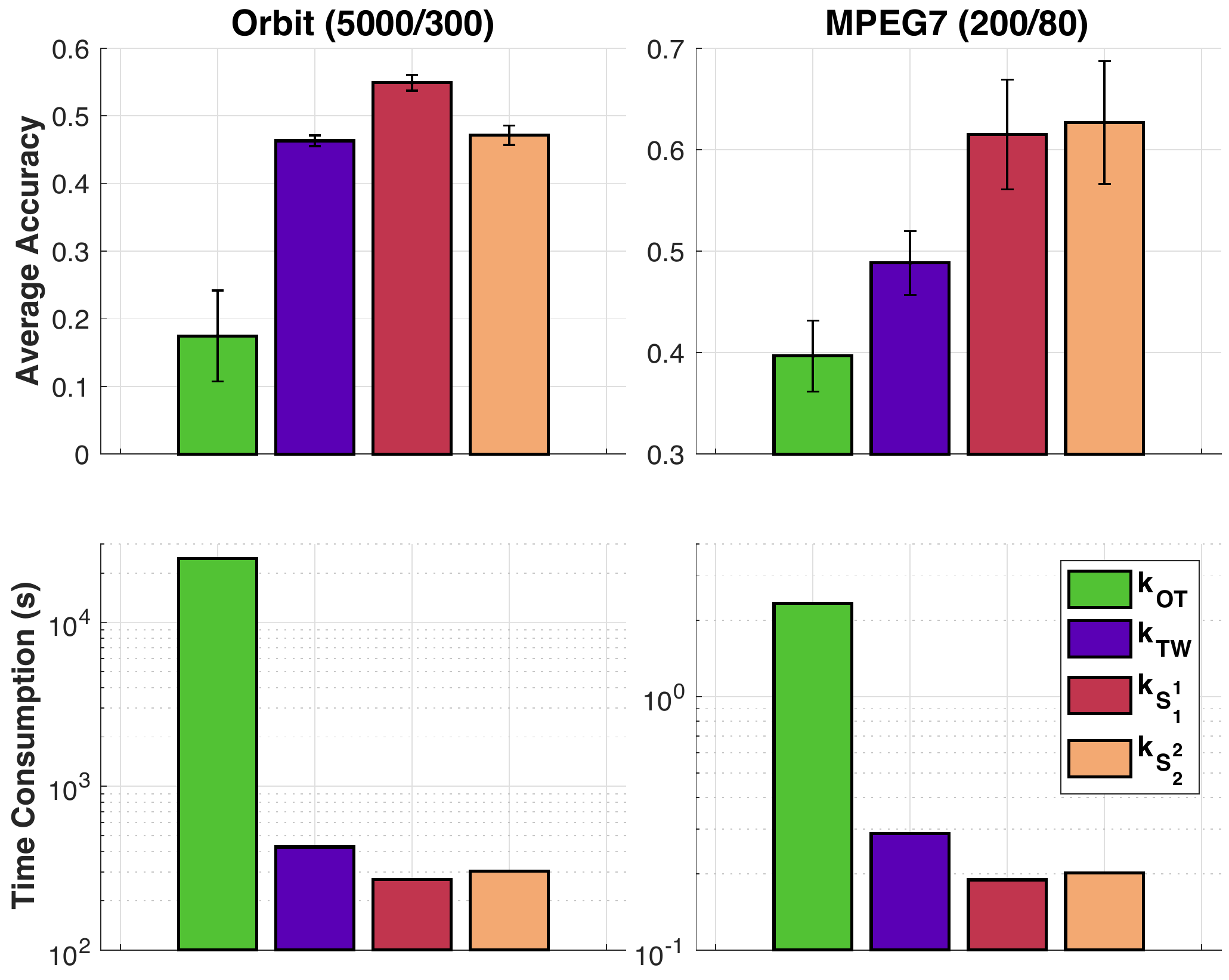}
  \vspace{-6pt}
  \caption{With $M=10^3$}
  \label{fg:TDA_1K_Log_AccTime}
 \vspace{-4pt}
\end{subfigure}
\hfill
\begin{subfigure}[b]{0.48\textwidth}
    \includegraphics[width=0.9\textwidth]{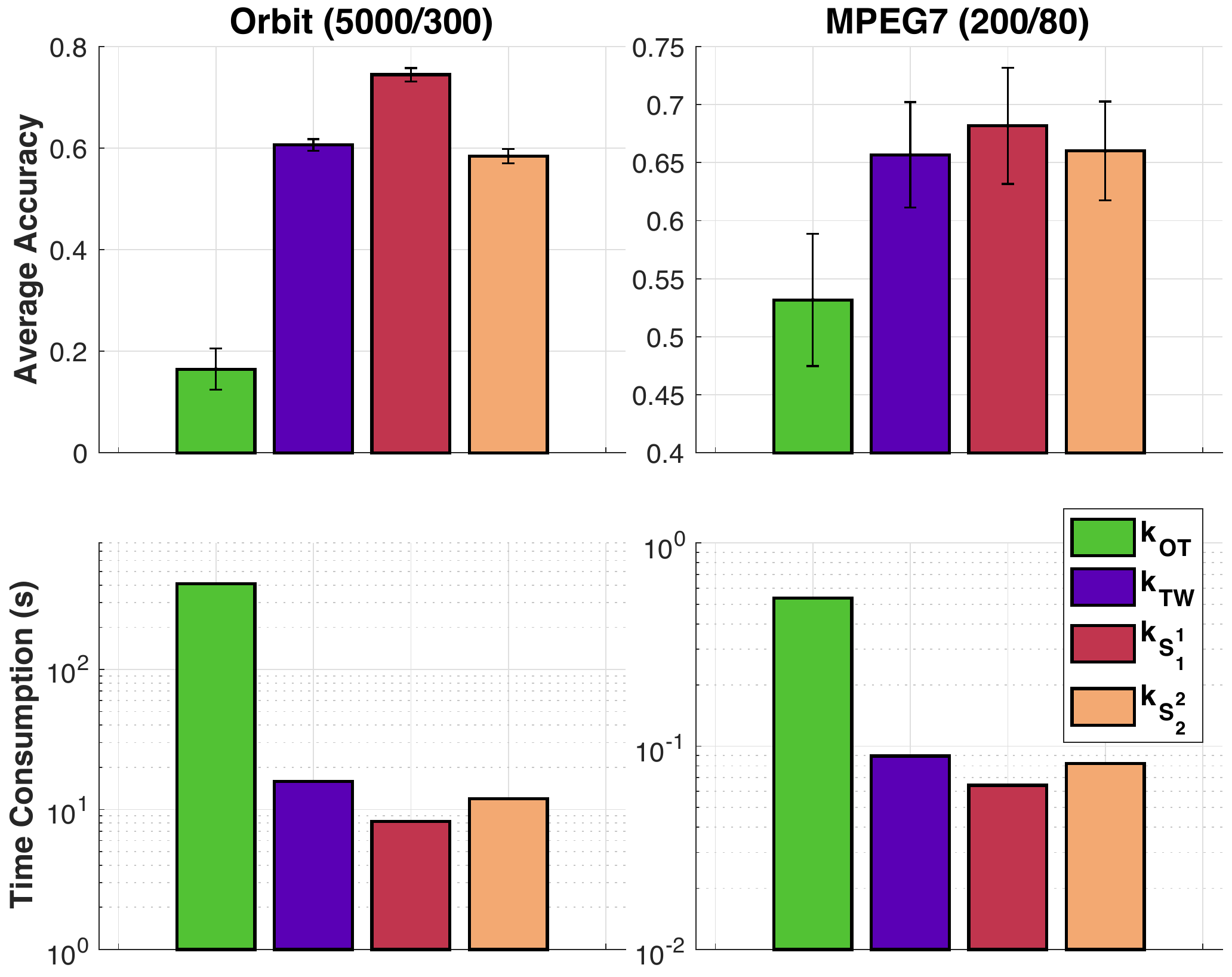}
  \vspace{-6pt}
  \caption{With $M=10^2$.}
  \label{fg:TDA_100_Log_AccTime}
 \vspace{-4pt}
\end{subfigure}
\caption{SVM results and time consumption for kernel matrices with $\G_{\text{Log}}$.}
\label{fg:TDA_Other_Log_AccTime}
\end{figure}


\begin{figure}
    \centering
\begin{subfigure}[b]{0.48\textwidth}
    \includegraphics[width=0.9\textwidth]{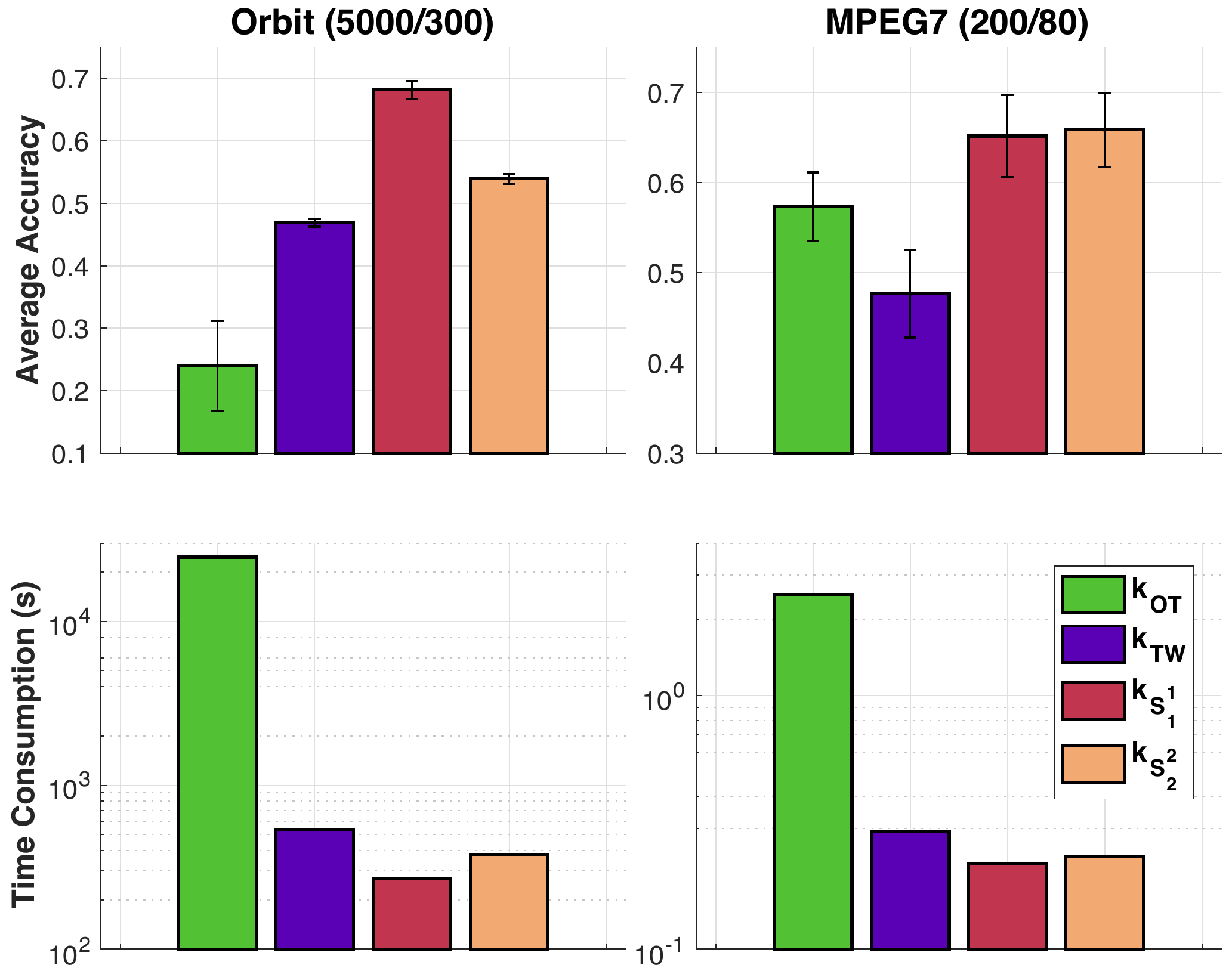}
  \vspace{-6pt}
  \caption{With $M=10^3$}
  \label{fg:TDA_1K_Log_AccTime}
 \vspace{-4pt}
\end{subfigure}
\hfill
\begin{subfigure}[b]{0.48\textwidth}
    \includegraphics[width=0.9\textwidth]{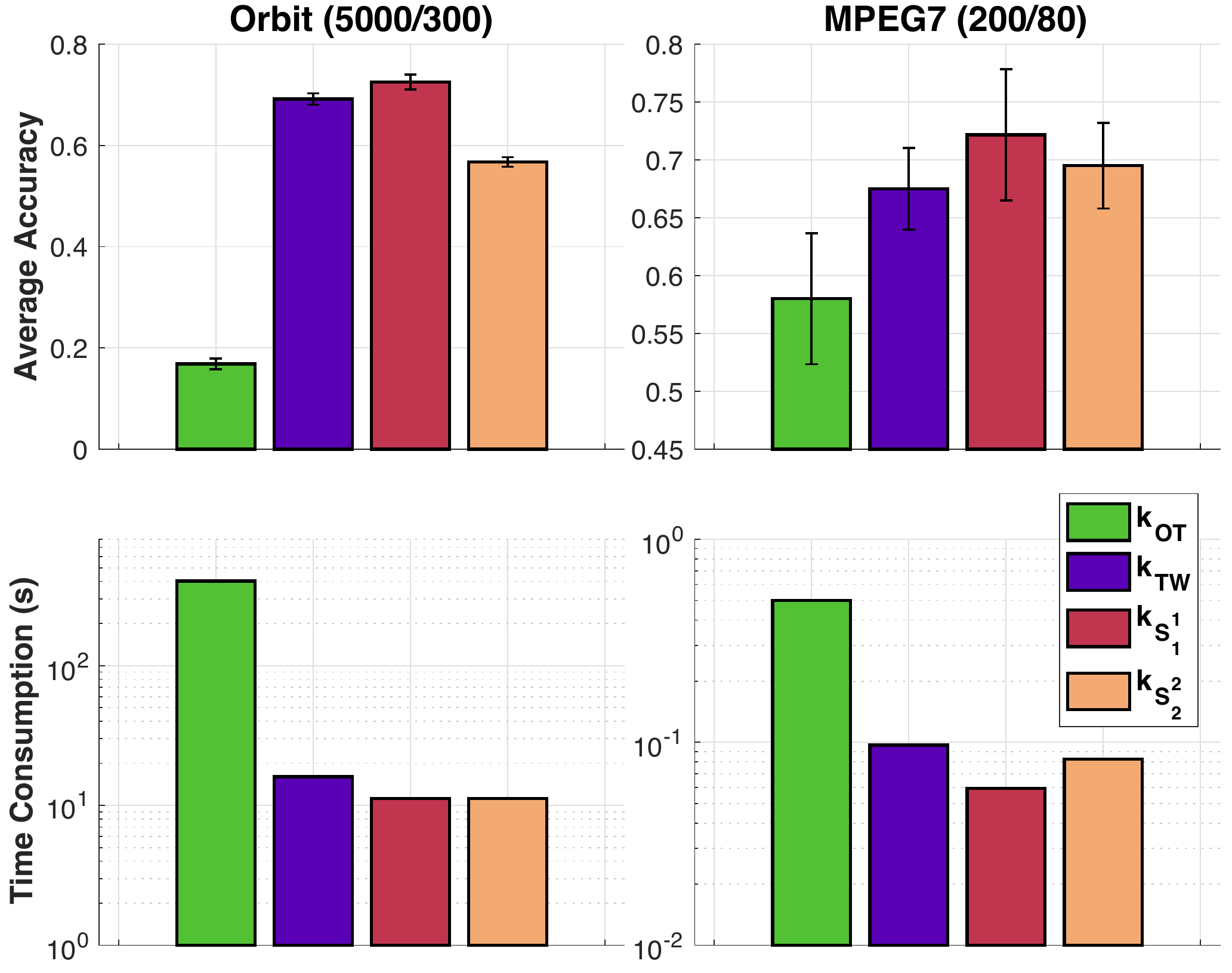}
  \vspace{-6pt}
  \caption{With $M=10^2$.}
  \label{fg:TDA_100_Log_AccTime}
 \vspace{-4pt}
\end{subfigure}
\caption{SVM results and time consumption for kernel matrices with $\G_{\text{Sqrt}}$.}
\label{fg:TDA_Other_Pow_AccTime}
\end{figure}

\begin{figure}
    \centering
\begin{subfigure}[b]{0.48\textwidth}
    \includegraphics[width=0.9\textwidth]{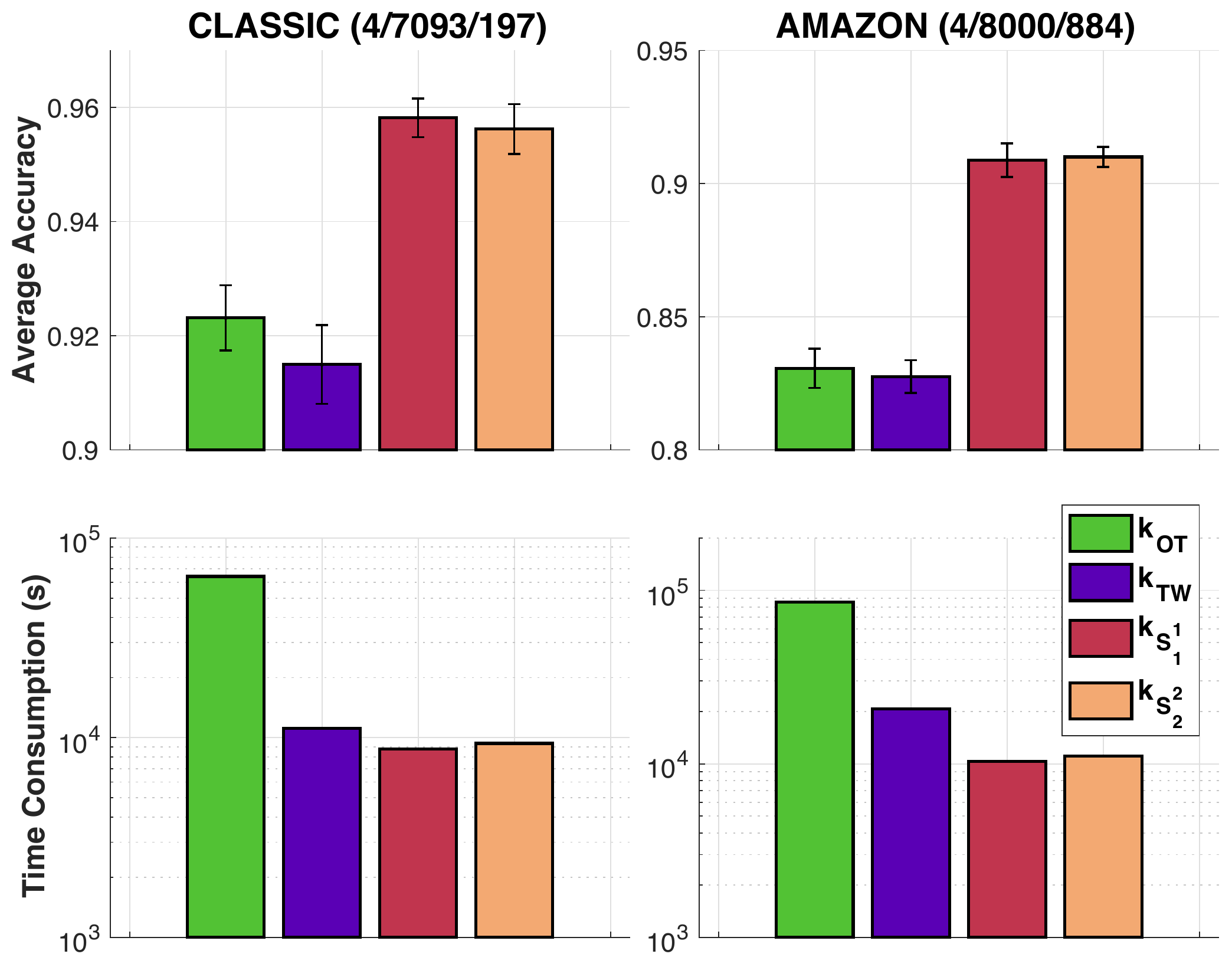}
  \vspace{-6pt}
  \caption{For graph $\G_{\text{Log}}$.}
  \label{fg:LargeGraph_40K_AccTime_Log}
 \vspace{-4pt}
\end{subfigure}
\hfill
\begin{subfigure}[b]{0.48\textwidth}
    \includegraphics[width=0.9\textwidth]{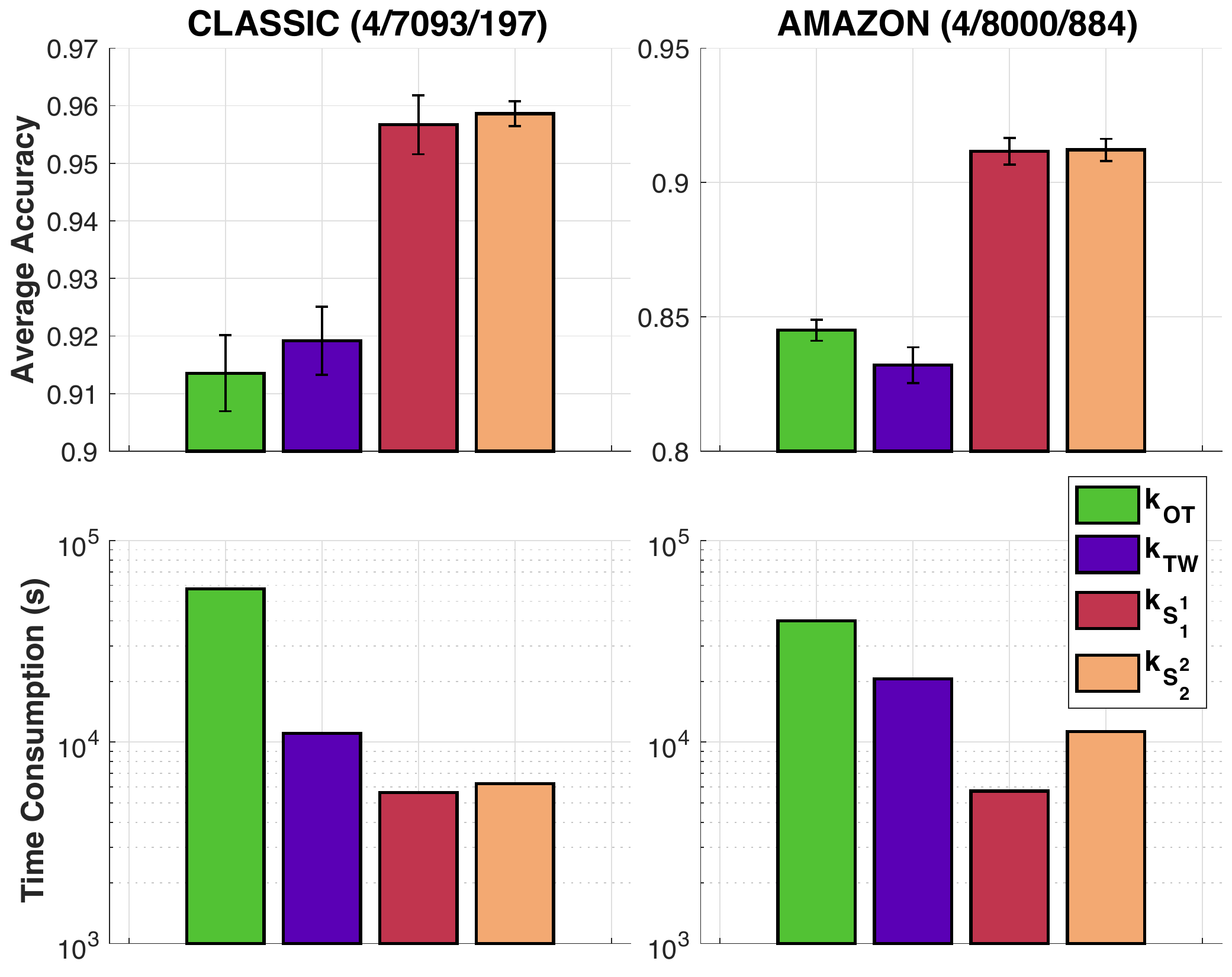}
  \vspace{-6pt}
  \caption{For graph $\G_{\text{Sqrt}}$.}
  \label{fg:LargeGraph_40K_AccTime_Pow}
 \vspace{-4pt}
\end{subfigure}
\caption{SVM results and time consumption for kernel matrices with large graphs where $M = 40000$.}
\label{fg:LargeGraph_40K_AccTime}
\end{figure}


\begin{figure}
  \begin{center}
    \includegraphics[width=0.65\textwidth]{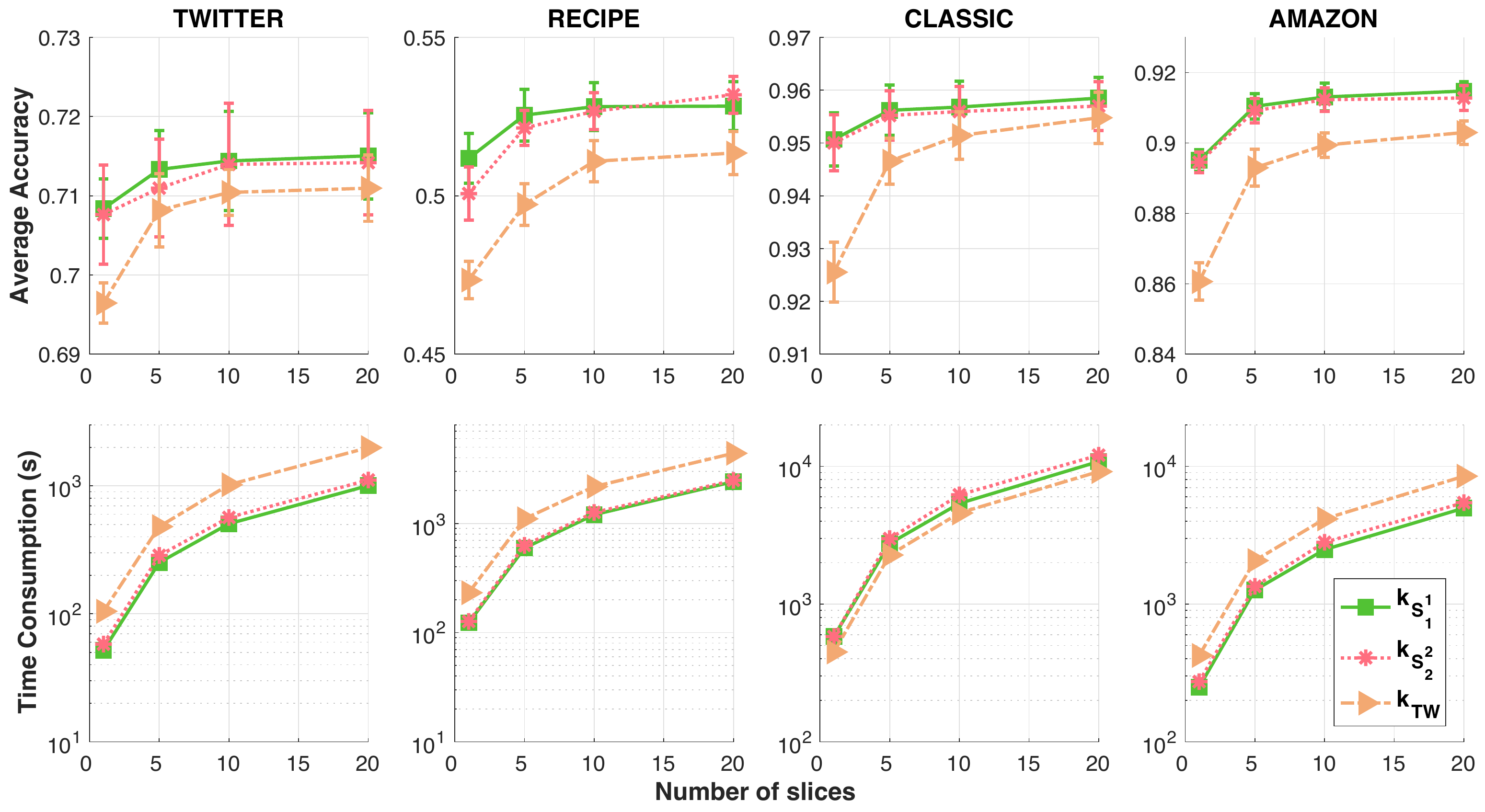}
  \end{center}
  \vspace{-6pt}
  \caption{SVM results and time consumption for kernel matrices of slice variants with $\G_{\text{Log}}$ ($M=10^3$).}
  \label{fg:DOC_1K_Log_AccTime_SLICE}
 \vspace{-10pt}
\end{figure}

\begin{figure}
  \begin{center}
    \includegraphics[width=0.65\textwidth]{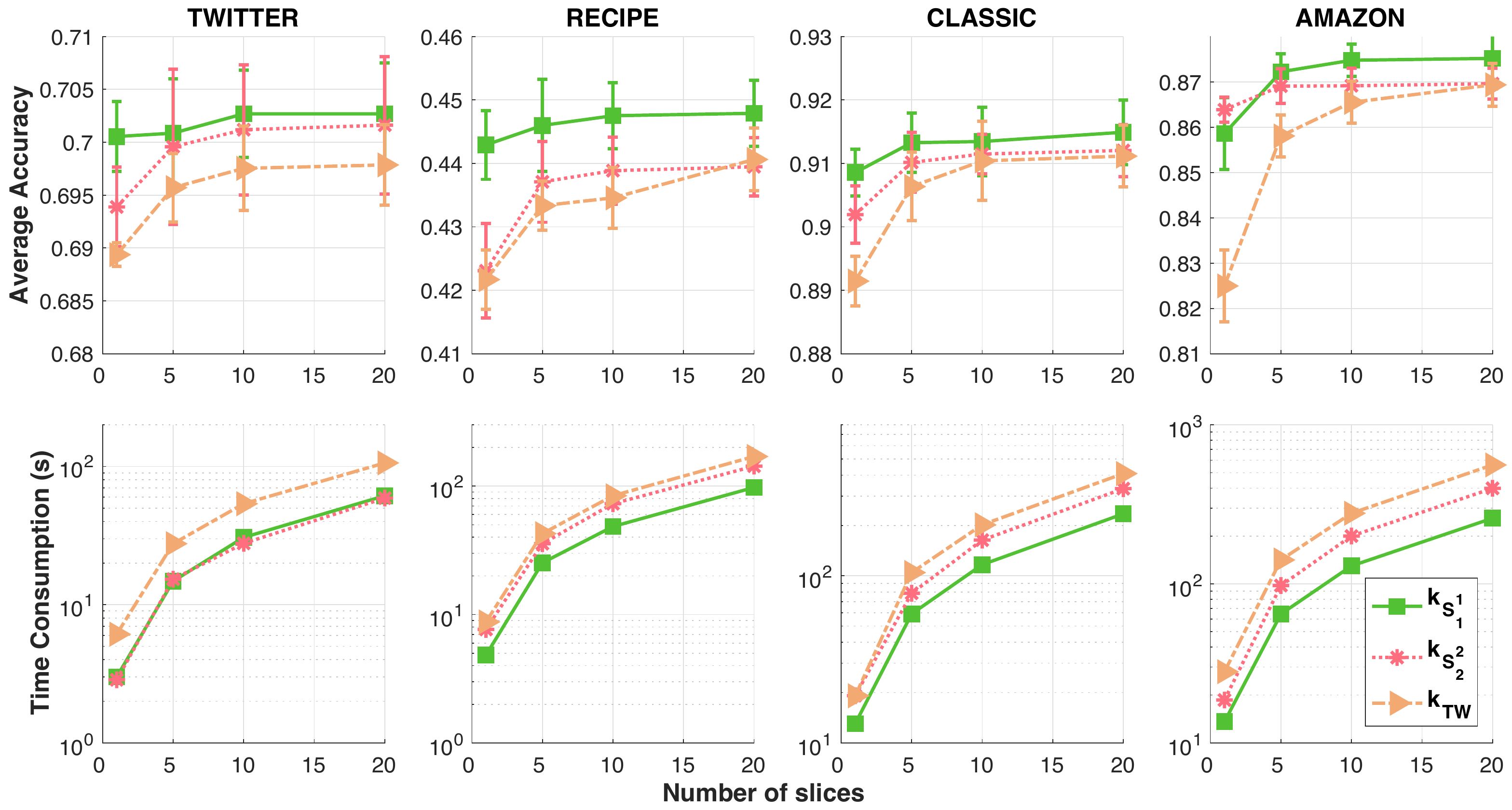}
  \end{center}
  \vspace{-6pt}
  \caption{SVM results and time consumption for kernel matrices of slice variants with $\G_{\text{Log}}$ ($M=10^2$).}
  \label{fg:DOC_100_Log_AccTime_SLICE}
 \vspace{-10pt}
\end{figure}

\begin{figure}
  \begin{center}
    \includegraphics[width=0.65\textwidth]{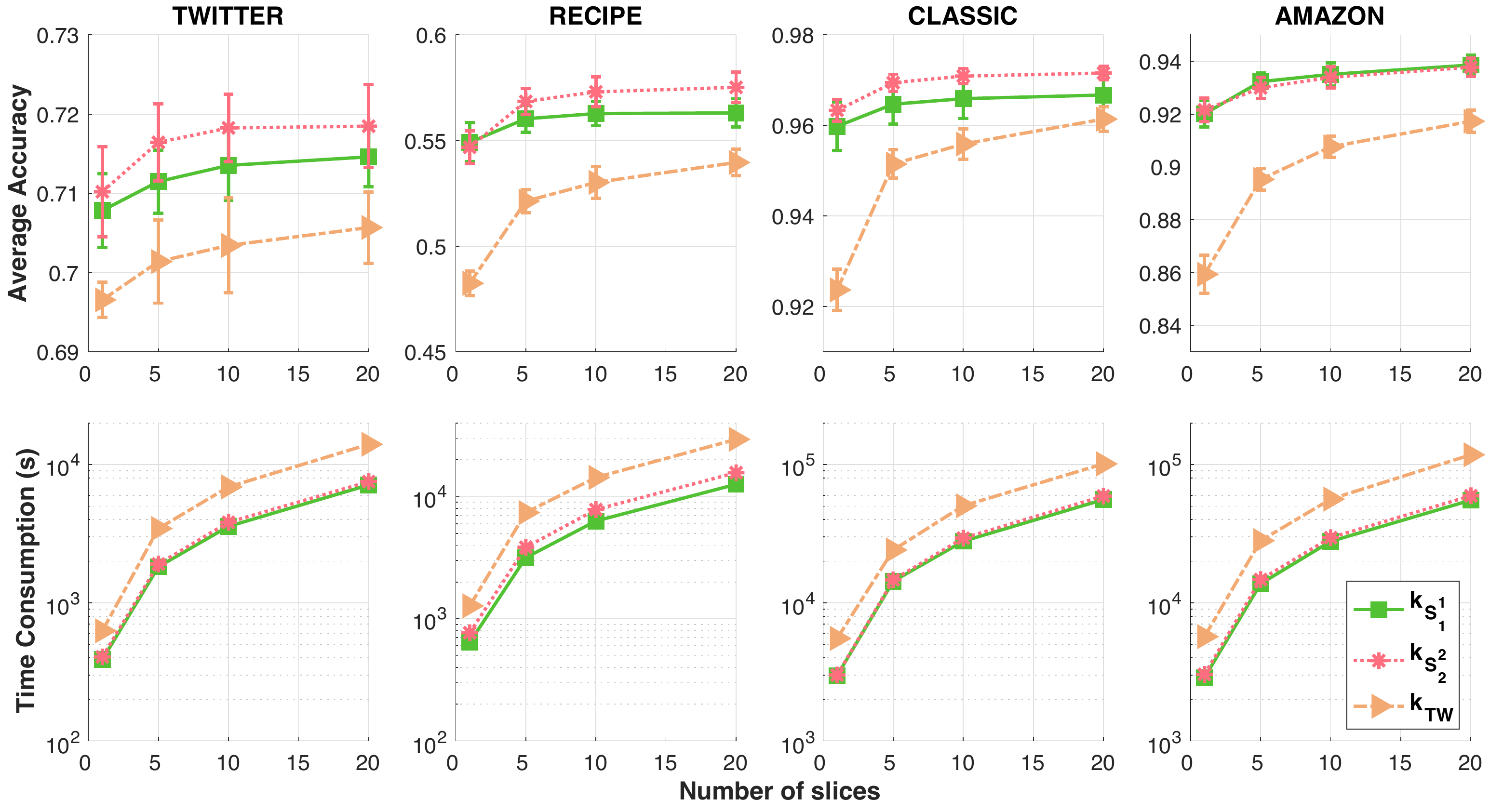}
  \end{center}
  \vspace{-6pt}
  \caption{SVM results and time consumption for kernel matrices of slice variants with $\G_{\text{Sqrt}}$ ($M=10^4$).}
  \label{fg:DOC_10K_Pow_AccTime_SLICE}
 \vspace{-10pt}
\end{figure}

\begin{figure}
  \begin{center}
    \includegraphics[width=0.65\textwidth]{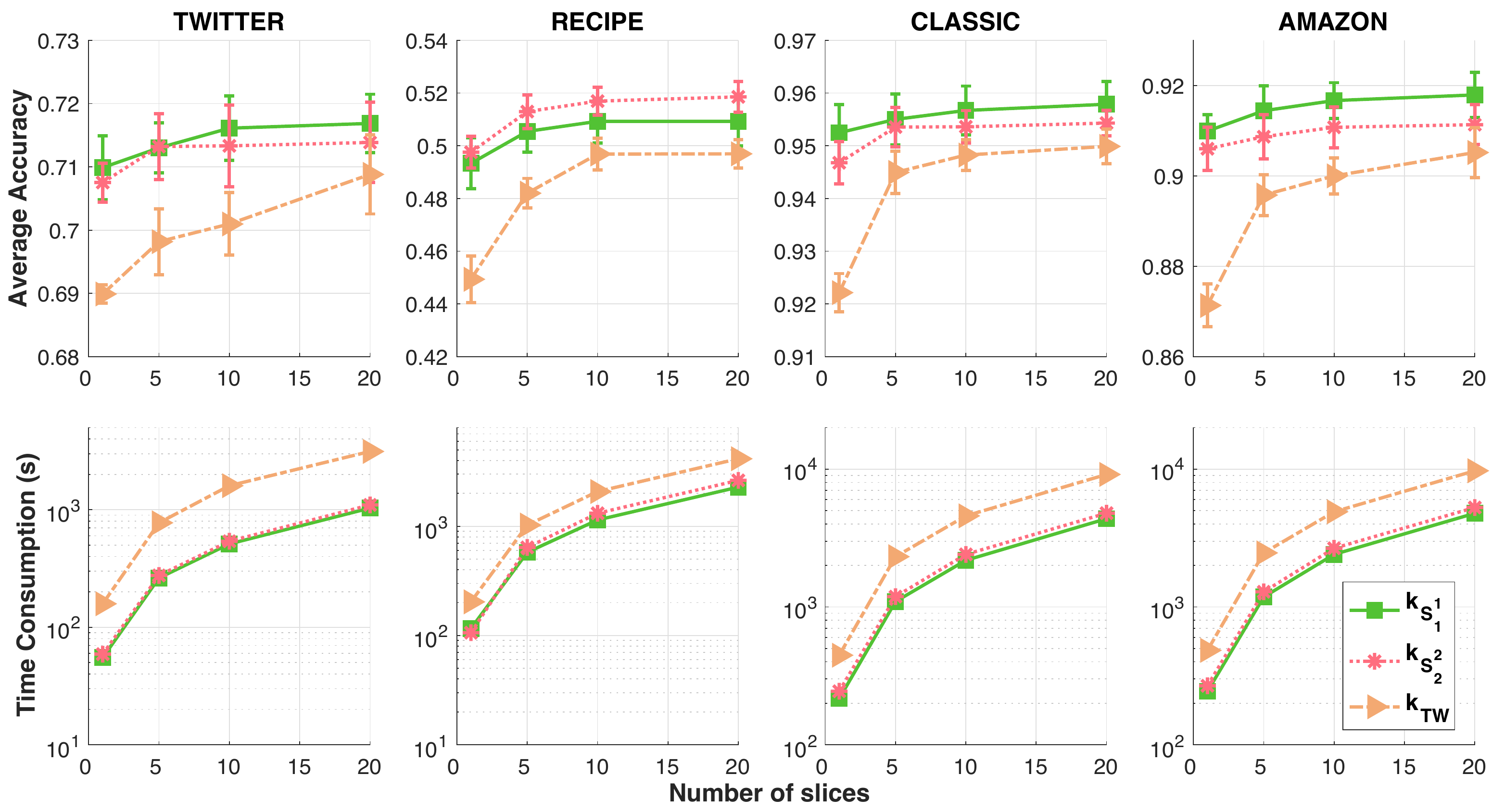}
  \end{center}
  \vspace{-6pt}
  \caption{SVM results and time consumption for kernel matrices of slice variants with $\G_{\text{Sqrt}}$ ($M=10^3$).}
  \label{fg:DOC_1K_Pow_AccTime_SLICE}
 \vspace{-10pt}
\end{figure}

\begin{figure}
  \begin{center}
    \includegraphics[width=0.65\textwidth]{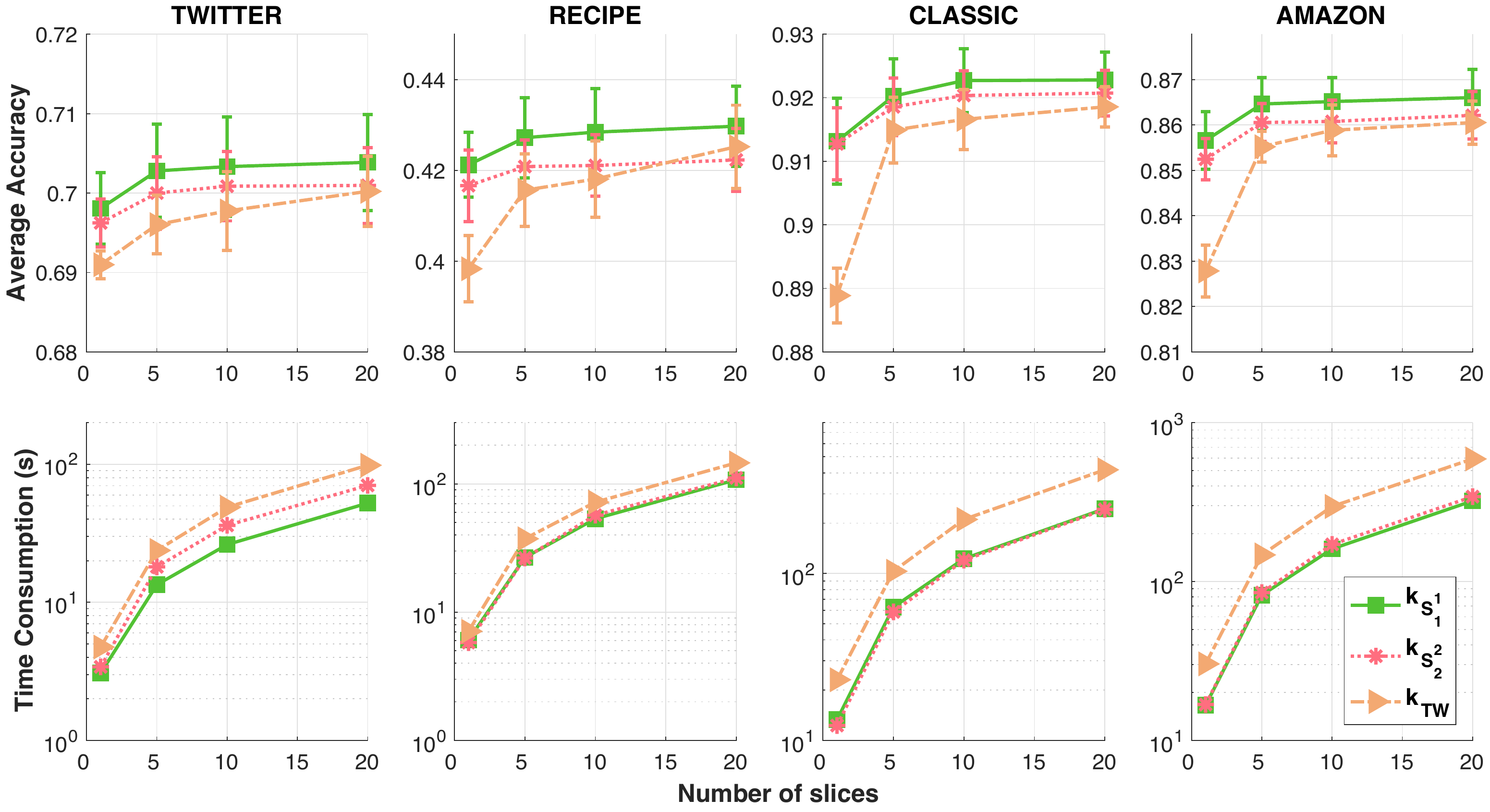}
  \end{center}
  \vspace{-6pt}
  \caption{SVM results and time consumption for kernel matrices of slice variants with $\G_{\text{Sqrt}}$ ($M=10^2$).}
  \label{fg:DOC_100_Pow_AccTime_SLICE}
 \vspace{-10pt}
\end{figure}

\begin{figure}
  \begin{center}
    \includegraphics[width=0.45\textwidth]{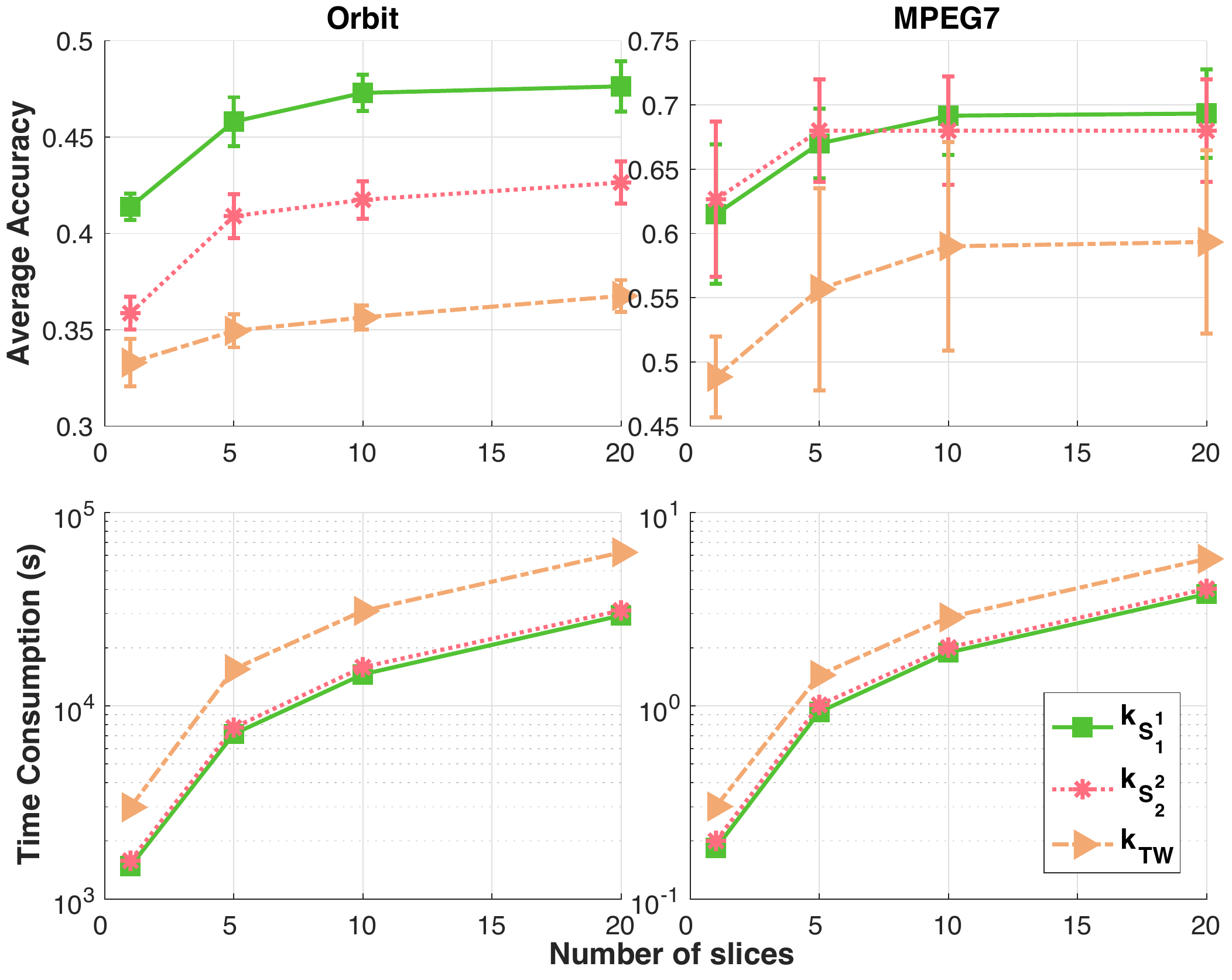}
  \end{center}
  \vspace{-6pt}
  \caption{SVM results and time consumption for kernel matrices of slice variants with $\G_{\text{Log}}$ where $M=10^4$ for \texttt{Orbit}, and $M=10^3$ for \texttt{MPEG7}.}
  \label{fg:TDA_mix10K1K_Log_AccTime_SLICE}
 \vspace{-10pt}
\end{figure}

\begin{figure}
    \centering
\begin{subfigure}[b]{0.48\textwidth}
    \includegraphics[width=0.85\textwidth]{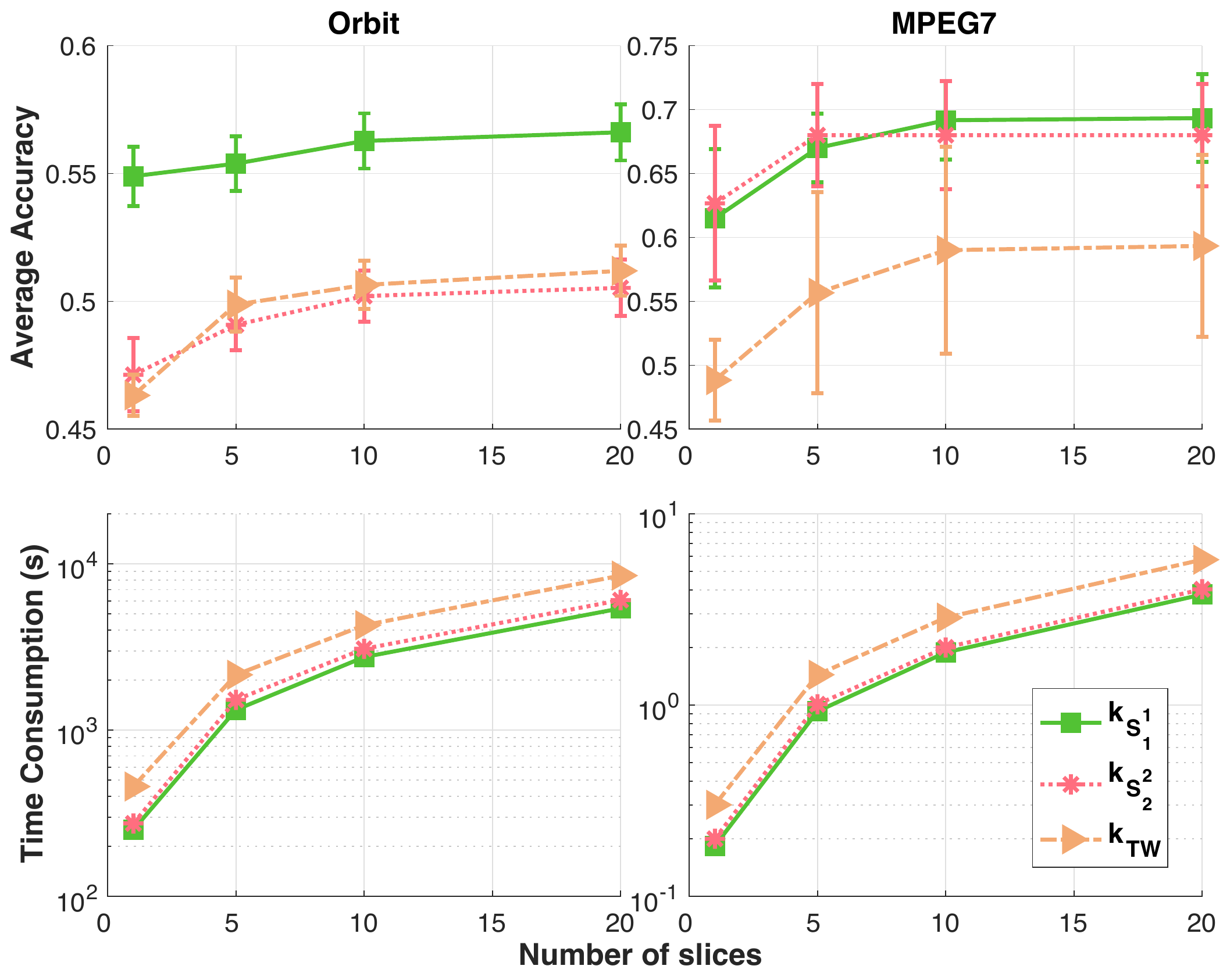}
  \vspace{-6pt}
  \caption{With $M=10^3$.}
  \label{fg:TDA_1K_Log_AccTime_SLICE}
 \vspace{-4pt}
\end{subfigure}
\hfill
\begin{subfigure}[b]{0.48\textwidth}
    \includegraphics[width=0.85\textwidth]{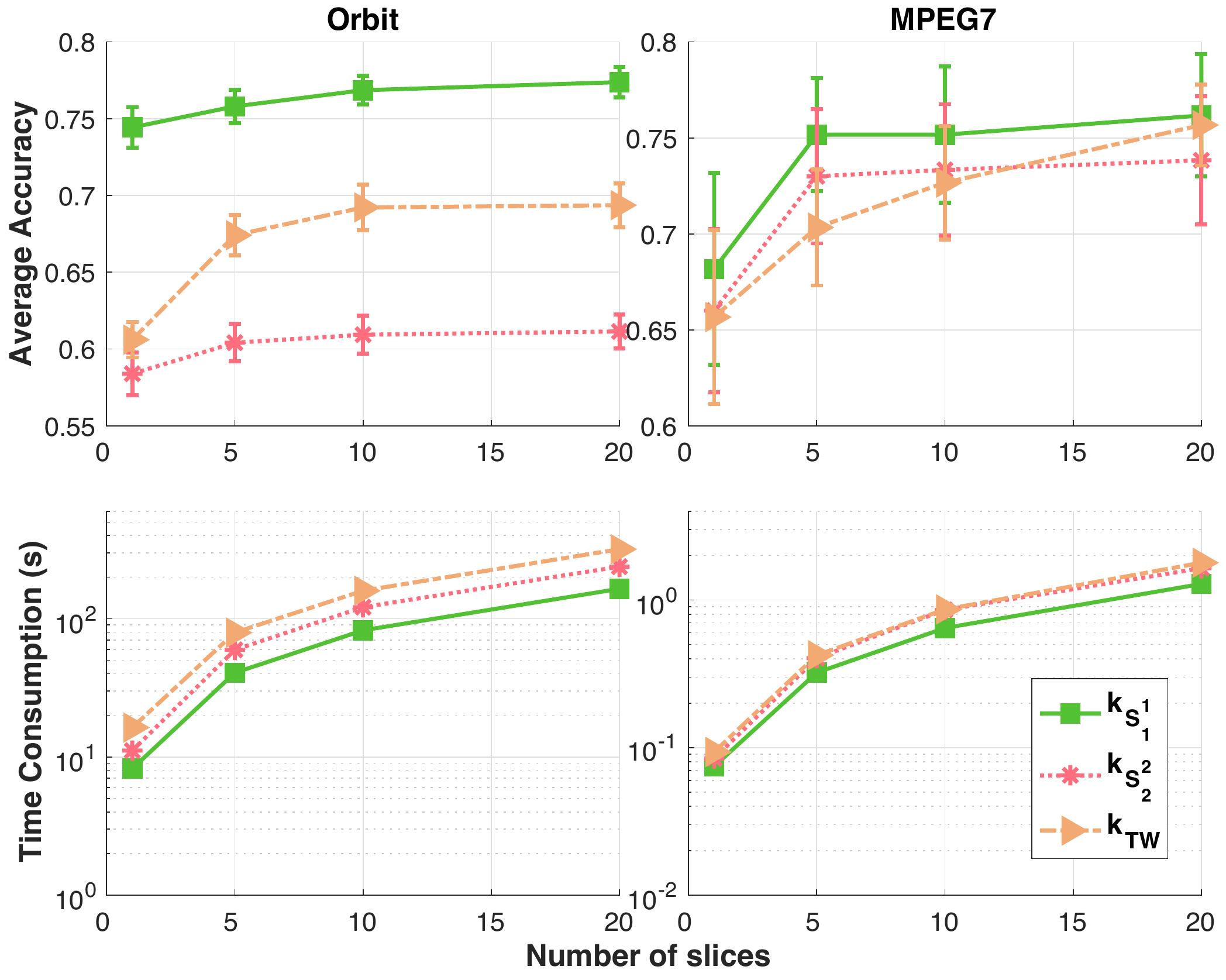}
  \vspace{-6pt}
  \caption{With $M=10^2$.}
  \label{fg:TDA_100_Log_AccTime_SLICE}
 \vspace{-4pt}
\end{subfigure}

    \caption{SVM results and time consumption for kernel matrices of slice variants with $\G_{\text{Log}}$ for TDA.}
    \label{fg:Other_TDA_Log_AccTime_SLICE}
\end{figure}


\begin{figure}
  \begin{center}
    \includegraphics[width=0.45\textwidth]{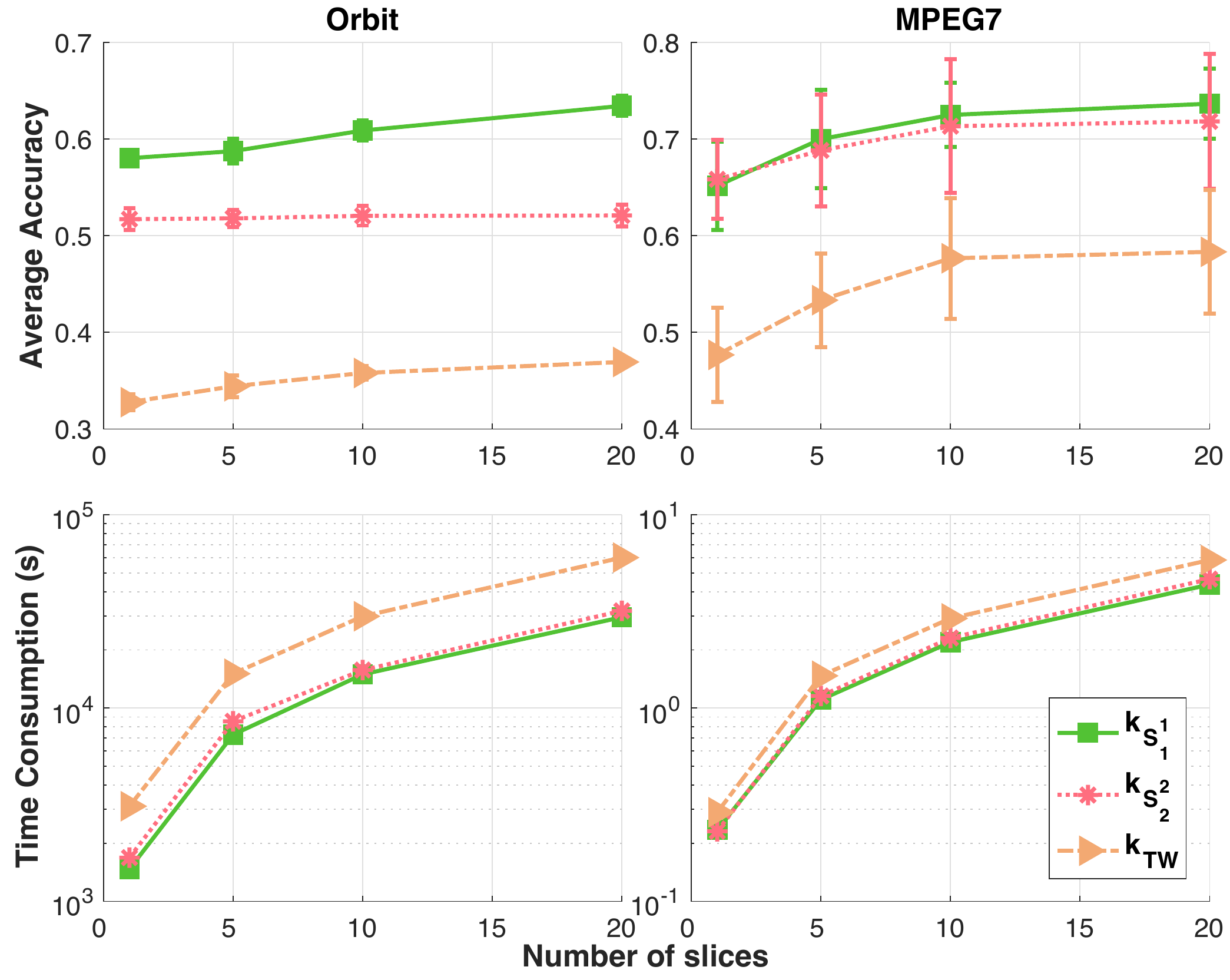}
  \end{center}
  \vspace{-6pt}
  \caption{SVM results and time consumption for kernel matrices of slice variants with $\G_{\text{Sqrt}}$ where $M=10^4$ for \texttt{Orbit}, and $M=10^3$ for \texttt{MPEG7}.}
  \label{fg:TDA_mix10K1K_Pow_AccTime_SLICE}
 \vspace{-10pt}
\end{figure}

\begin{figure}
    \centering
\begin{subfigure}[b]{0.48\textwidth}
    \includegraphics[width=0.85\textwidth]{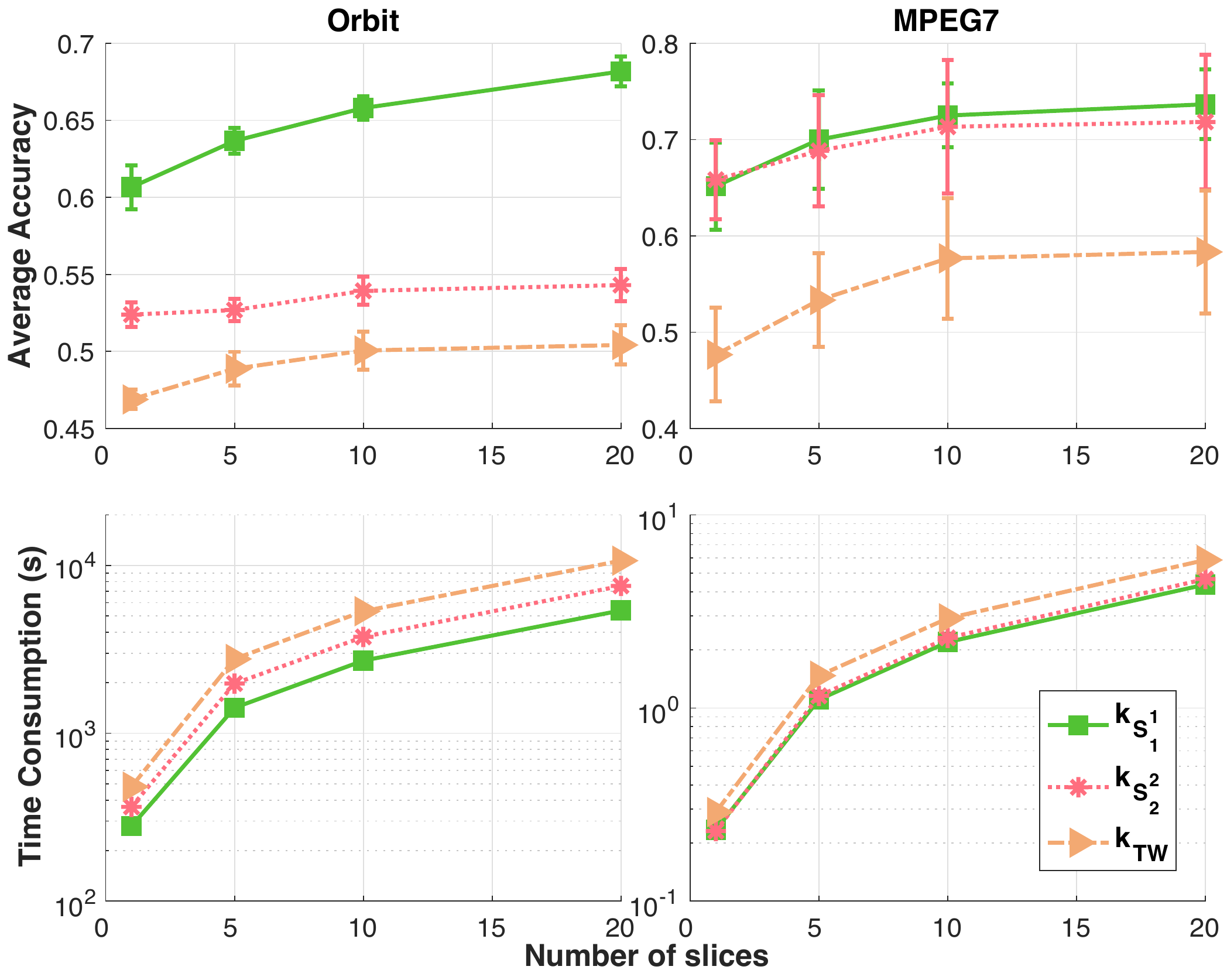}
  \vspace{-6pt}
  \caption{With $M=10^3$.}
  \label{fg:TDA_1K_Pow_AccTime_SLICE}
 \vspace{-4pt}
\end{subfigure}
\hfill
\begin{subfigure}[b]{0.48\textwidth}
    \includegraphics[width=0.85\textwidth]{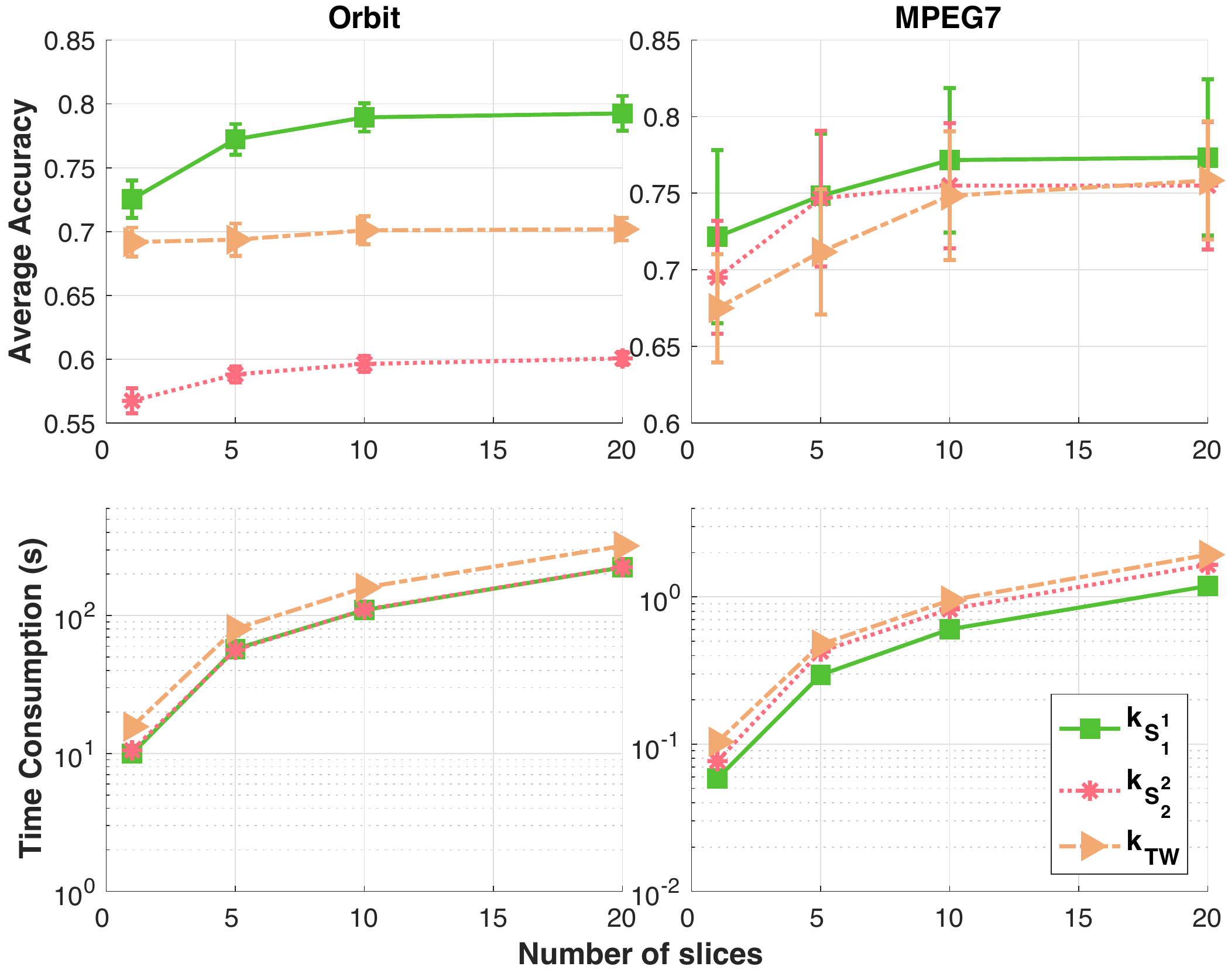}
  \vspace{-6pt}
  \caption{With $M=10^2$.}
  \label{fg:TDA_100_Pow_AccTime_SLICE}
 \vspace{-4pt}
\end{subfigure}

   \caption{SVM results and time consumption for kernel matrices of slice variants with $\G_{\text{Sqrt}}$ for TDA.}
    \label{fg:Other_TDA_Pow_AccTime_SLICE}
\end{figure}


\begin{figure}
    \centering
\begin{subfigure}[b]{0.48\textwidth}
    \includegraphics[width=0.85\textwidth]{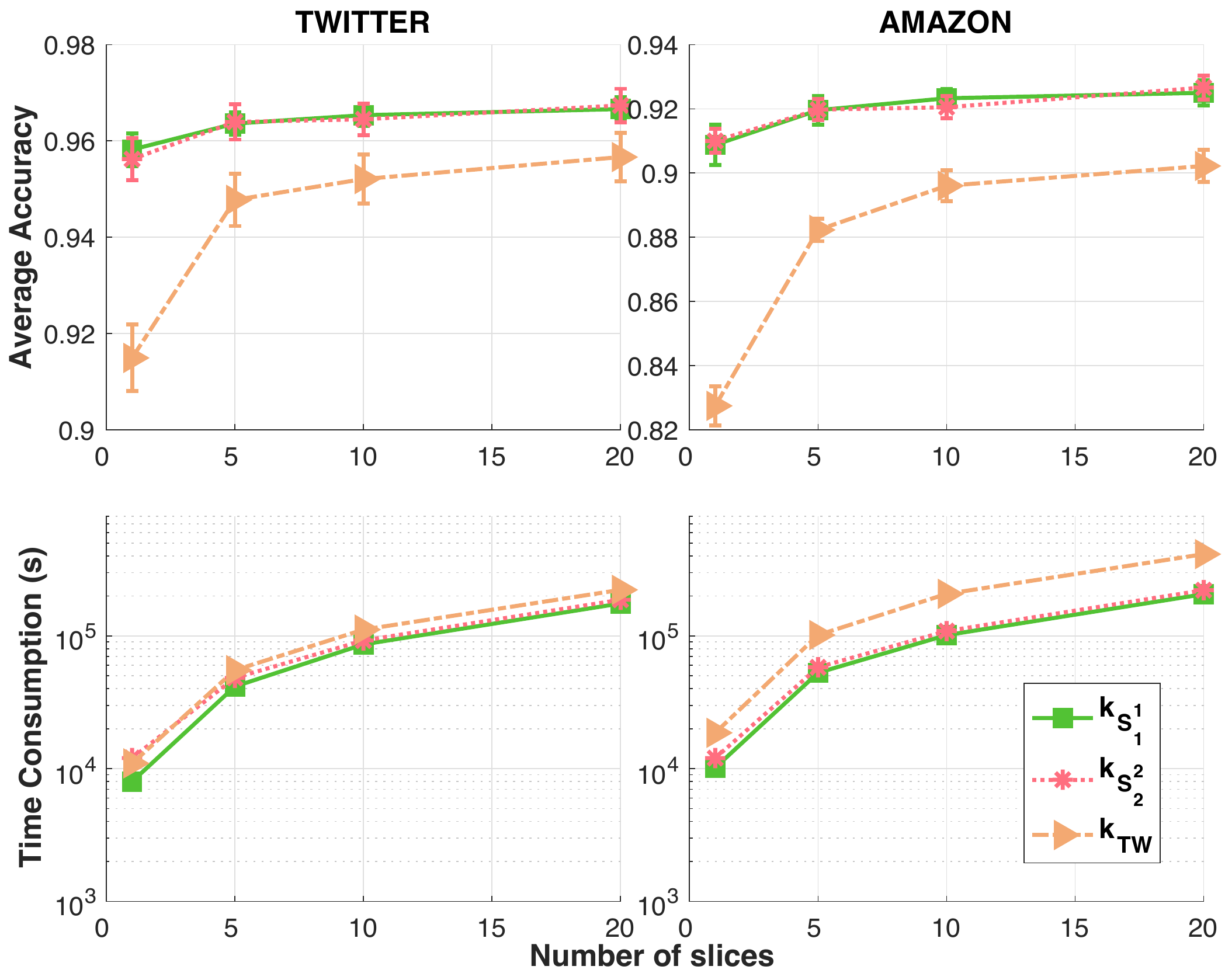}
  \vspace{-6pt}
  \caption{For graph $\G_{\text{Log}}$.}
  \label{fg:DOC_40K_Log_AccTime_SLICE}
 \vspace{-4pt}
\end{subfigure}
\hfill
\begin{subfigure}[b]{0.48\textwidth}
    \includegraphics[width=0.85\textwidth]{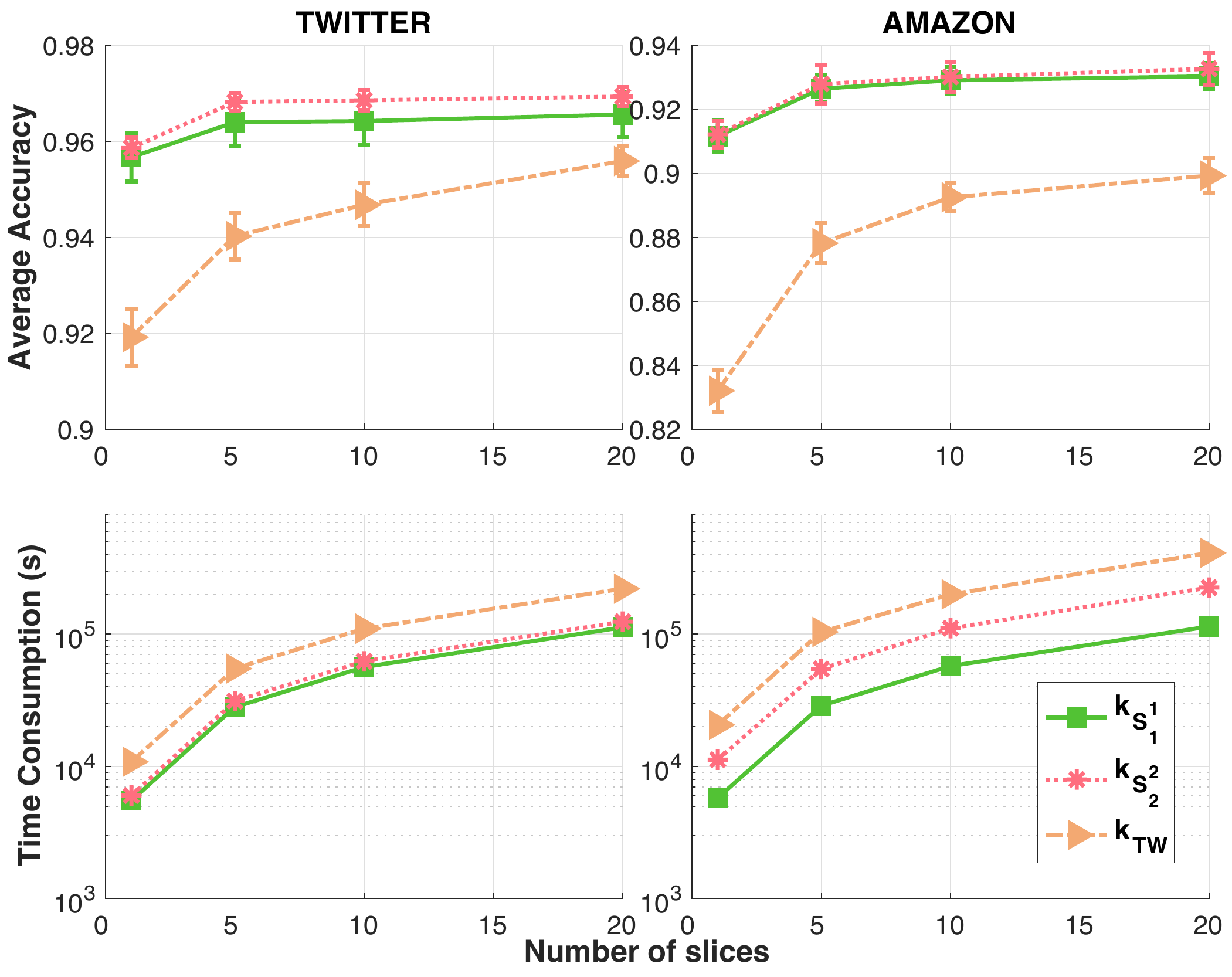}
  \vspace{-6pt}
  \caption{For graph $\G_{\text{Pow}}$.}
  \label{fg:DOC_40K_Pow_AccTime_SLICE}
 \vspace{-4pt}
\end{subfigure}
    \caption{SVM results and time consumption for kernel matrices of slice variants with $\G_{\text{Sqrt}}$ for TDA with a large graph where the number of nodes is $40000$.}
    \label{fg:Other_TDA_Large_AccTime_SLICE}
\end{figure}

 \begin{figure}
     \centering
 \begin{subfigure}[b]{0.48\textwidth}
     \includegraphics[width=0.93\textwidth]{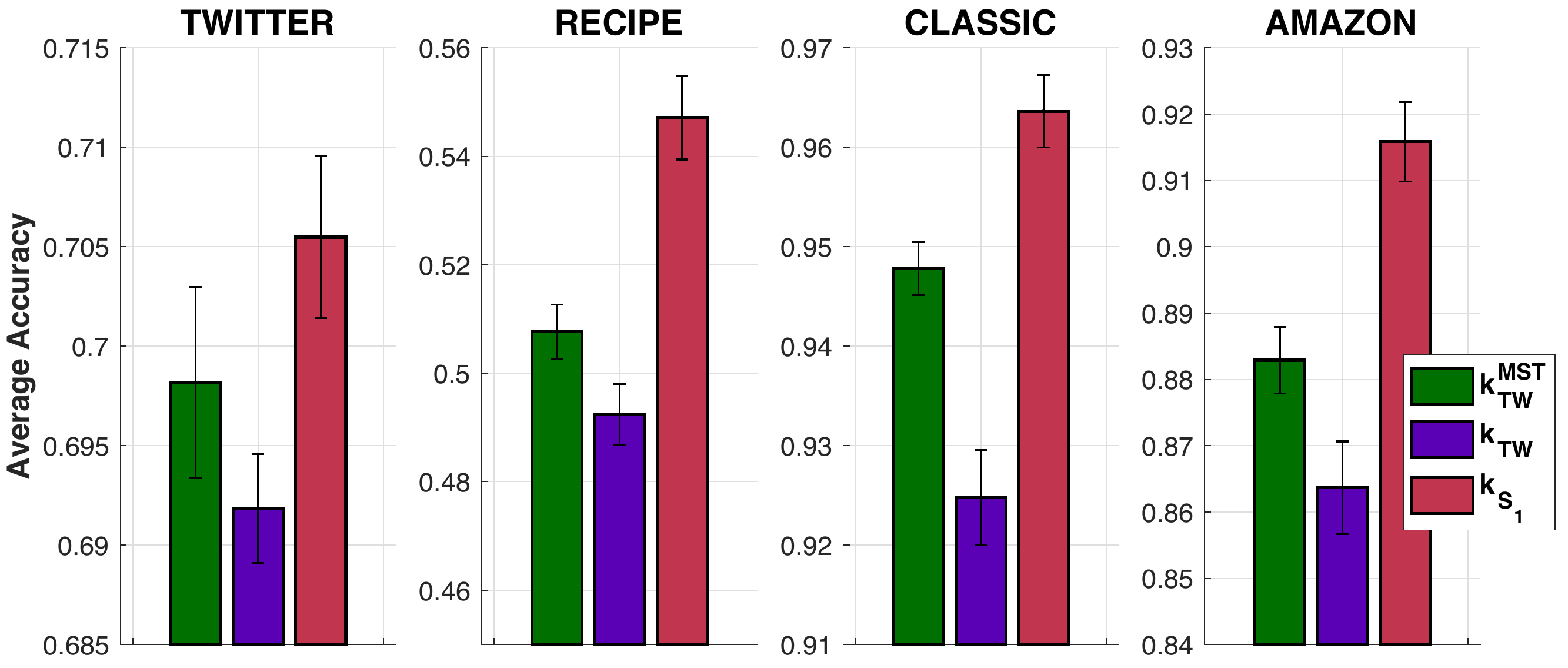}
   \vspace{-6pt}
   \caption{For graph $\G_{\text{Log}}$.}
   \label{fg:DOC_MST_40K_Log_AccTime}
  \vspace{-4pt}
 \end{subfigure}
 \hfill
 \begin{subfigure}[b]{0.48\textwidth}
     \includegraphics[width=0.99\textwidth]{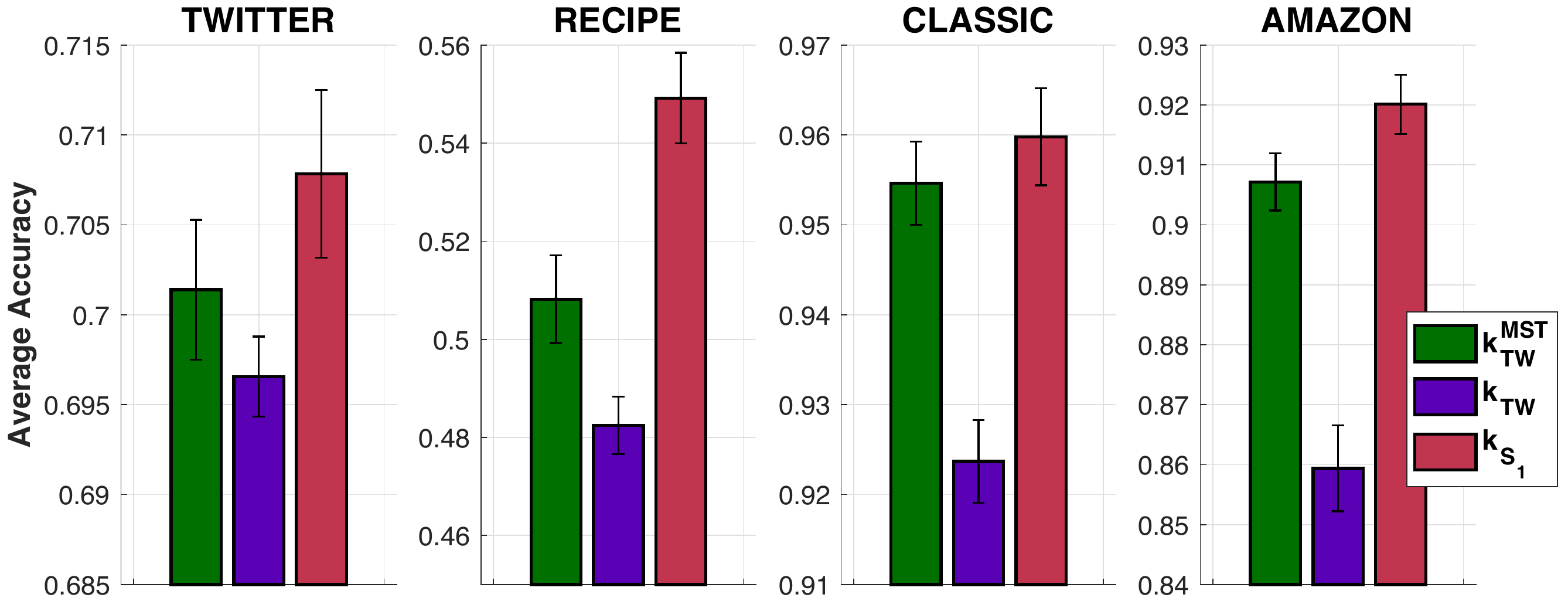}
   \vspace{-6pt}
   \caption{For graph $\G_{\text{Pow}}$.}
   \label{fg:DOC_MST_40K_Pow_AccTime}
  \vspace{-4pt}
 \end{subfigure}
     \caption{SVM results for document classification with $M=10000$ graph nodes.}
     \label{fg:Other_DOC_MST_AccTime}
 \end{figure}

\end{document}